\documentclass[twoside]{article}
\usepackage{ecj,palatino,epsfig,latexsym,natbib,soul,hyperref,xcolor,amsmath,amssymb,pifont,mathabx,dsfont}
\usepackage{float,subfigure,lscape,threeparttable,booktabs,multirow,afterpage,adjustbox,multibib}
\usepackage{rotate}
\usepackage{amsthm}
\usepackage{amsmath}
\usepackage{tikz,pgfplots}
\usetikzlibrary{quotes,angles}
\usetikzlibrary{calc}
\usetikzlibrary{arrows.meta}
\theoremstyle{definition}

\newtheorem{prop}{Proposition}

\begin{document}

\title{\bf The importance of being constrained:\\ dealing with infeasible solutions in\\ Differential Evolution and beyond}

\author{\name{\bf Anna V. Kononova} \hfill \addr{a.kononova@liacs.leidenuniv.nl}\\ 
        \addr{LIACS, Leiden University, The Netherlands}
\AND
       \name{\bf Diederick Vermetten} \hfill \addr{d.l.vermetten@liacs.leidenuniv.nl}\\
        \addr{LIACS, Leiden University, The Netherlands}
\AND
       \name{\bf Fabio Caraffini} \hfill \addr{fabio.caraffini@dmu.ac.uk}\\
        \addr{Institute of Artificial Intelligence, De Montfort University, Leicester, UK} 
\AND
       \name{\bf Madalina-A. Mitran} \hfill \addr{madalina.mitran96@e-uvt.ro}\\
        \addr{Department of Computer Science, West University of Timi\c soara, Romania }
\AND
       \name{\bf Daniela Zaharie} \hfill \addr{daniela.zaharie@e-uvt.ro}\\
        \addr{Department of Computer Science, West University of Timi\c soara, Romania}
}
\maketitle

\begin{abstract}
We argue that results produced by a heuristic optimisation algorithm \textit{cannot} be considered reproducible unless the algorithm fully specifies what should be done with solutions generated outside the domain, \textit{even} in the case of simple box constraints. Currently, in the field of heuristic optimisation, such specification is rarely mentioned or investigated due to the assumed triviality or insignificance of this question. Here, we demonstrate that, at least in algorithms based on Differential Evolution, this choice induces notably different behaviours -- in terms of performance, disruptiveness and population diversity. This is shown theoretically (where possible) for standard Differential Evolution in the absence of selection pressure and experimentally for the standard and state-of-the-art Differential Evolution variants on special test function $f_0$ and BBOB benchmarking suite, respectively. Moreover, we demonstrate that the importance of this choice quickly grows with problem's dimensionality. Different Evolution is \textit{not} at all special in this regard -- there is no reason to presume that other heuristic optimisers are not equally affected by the aforementioned algorithmic choice. Thus, we urge the field of heuristic optimisation to formalise and adopt the idea of a \textit{new algorithmic component} in heuristic optimisers, which we call here a strategy of dealing with infeasible solutions. This component needs to be consistently (a) specified in algorithmic descriptions to guarantee reproducibility of results, (b) stu\-died to better understand its impact on algorithm's performance in a wider sense and (c) included in the (automatic) algorithmic design. All of these should be done \textit{even} for problems with box constraints. 
\end{abstract} 


\section{Introduction}\label{sect:Introduction}
The overwhelming majority of practical optimisation problems are constrained at least in some sense: from simply limiting the ranges of input variables to complex nonlinear or black-box functional constraints. There is a clear benefit of tackling such problems according to their constrained nature by considering the search space as constrained rather than at first approximating the problem as an unconstrained variant.

While current level of numerical heuristic optimisation allows tackling these problems directly as constrained, it is not done consistently, especially for the simplest types of constraints, on the input variables. For example, very few papers in the field mention what should be done with infeasible solutions (that violate some constraint) that might, or even are very likely to, be generated during an optimisation run. Many options are possible for handling such solutions, e.g. penalty functions, repair methods and feasibility preserving generating operators. However, in practice, the choice made for a particular algorithm is often omitted from the description due to either its assumed insignificance or triviality in the eyes of algorithm's designer. And yet this choice has been recently shown to strongly influences algorithm's performance \citep{Boks2021,deNobel2021,Kononova2020_outside}. 

It is true that both engineering and pure mathematical approaches dictate that infeasible solutions should not be evaluated during optimisation: the former - due to physical limitations of the underlying processes/devices and the latter - since the objective function is not required to be formally defined for such infeasible solutions. However, heuristic optimisation approaches, not being exact, sometimes use some kind of `information' on the values of objective function of infeasible solutions -- see exterior penalty approaches in \citep{CoelloCoello2002_survey}. On one hand, this narrows down the applicability of such methods; on the other hand, they get an advantage through  additional domain information. Is it fair to compare such algorithms with those not using this information? In our view, such distinction should at least be highlighted. Additionally, when benchmarking heuristic optimisation algorithms, it is often unclear whether the boundaries should be dealt with explicitly by the algorithm, or if the problem itself should handle this using a penalty function~\citep{hansen2021coco}. 

Moreover, the aforementioned ambiguity in algorithms' specifications regarding dealing with infeasible solutions naturally leads to reproducibility issues. It is this aspect specifically that is being discussed in this paper.

If algorithm's source code is not made available, such an ambiguity has to be resolved via wasteful trying (`and-erroring') of infinite possibilities -- further burdened by the random nature of the algorithms which does not allow guesses regarding missing algorithm specification exactly~\footnote{Unless, in addition to the available source code, the \textit{exact} specifications of the operating system where the reported experiments have been executed, versions of programming languages and random seed values of pseudorandom number generator are also known \citep{LEcuyer2007TestU01,vdHonert2021}.}.

If the source code of an algorithm is available, but the accompanying specification is still ambiguous regarding the strategy for dealing with infeasible solutions, we can assume that either multiple options have been tried and the selected one performs best, or this component was considered as non-essential to the algorithms behaviour and thus not investigated. In either case, information about the algorithm in relation to the strategy for dealing with infeasible solutions is never reported~\footnote{At the same time, it is not guaranteed that missing specification can be easily extracted from the available code.}, and when only one method was considered this could mean that potential performance gains have not been realised. This same argument can hold in the case of automated algorithm configurations, where modules dealing with the boun\-da\-ry constraints are rarely part of the search-space. Even when they are considered in the configuration, they are often grouped together with another operator such as mutation, which makes it more challenging to accurately see its impact~\citep{Stuetzle2019,Duarte2020}.

While the wider community of evolutionary algorithms has recently become more aware of the challenges and benefits of reproducibility \citep{Ibanez2021}, the standards for availability of code, data and other artefacts still differ widely between the venues. 
For the field of heuristic optimisation algorithms to move forward, a stronger reliance on reproducible experimental results is needed, especially in the absence of overarching theoretical frameworks. Thus, in order to ensure that any reported findings can be reproduced, we should aim to be aware of even the seemingly small design decisions within our algorithms that are often overlooked, since even minor changes in algorithm behaviour can lead to irreproducible results if not properly documented or otherwise made available.

Unfortunately, many of the papers which propose new (meta)heuristics or algorithmic improvements on existing ones do not contain explicit information on how the out of bounds components are treated. In the particular case of Differential Evolution (DE) the amount of such papers is significant -- see Section~\ref{sect:SDIS_literature} for a review.
We believe such problem manifests itself for \textit{all heuristic optimisation methods} discussed in both specialised theoretical and applied literature. 
Thus, with this paper, we conclude that there is a \textit{major reproducibility issue} with the state-of-the-art heuristic optimisation methods and call for proper formalisation of a new operator/algorithmic component that deals with infeasible solutions inside heuristic optimisers. Such component needs to be \textit{consistently}:
\begin{itemize}
    \item[(a)] specified in algorithmic descriptions to guarantee full reproducibility of results, 
    \item[(b)] studied to understand its impact on algorithms' performance in a wider sense,
    \item[(c)] included in the (automatic) design of algorithms. 
\end{itemize}
All of the above should be done \textit{even} for problems with box constraints only.

To emphasise the importance of such an algorithmic component, we propose an integrated approach, at both experimental and theoretical levels, to analyse the impact of the strategies used to deal with solutions violating box-constraints on the behaviour of the optimisation algorithm. One of the main contributions of this paper is the usa\-ge of the cosine similarity measure to quantify the influence of the aforementioned strategies on the search direction induced by the optimisation algorithm. 

The remainder of this paper is organised as follows: In Section~\ref{sect:boxcons}, we discuss the general problem of using heuristic optimization methods to solve box-constrained problems, with a focus on Differential Evolution and the Strategy for Dealing with Infeasible Solutions (SDIS) used in this context. In Section~\ref{sect:SearchDirection}, we consider the notion of disruption of the search behaviour and propose the cosine similarity to measure this phenomenon. Then, in Section~\ref{sect:theoretic}, we present a theoretical analysis on the amount of infeasible solutions within DE, and analyse the impact of several popular SDIS on search directions and diversity of the population. To further analyse the impact of SDIS on these aspects, we make use of the function $f_0$, which assigns uniformly distributed random values to the elements and hence `removes' the selective pressure \textit{without} any modification to the algorithm under investigation, and uses this to study the relation between parameters of DE, SDIS, cosine similarity between search directions and population diversity. We also consider the overall amount of infeasible solutions generated, and relate the analysis on $f_0$ to the concept of structural bias. Finally, we perform a benchmark study on several versions of DE and investigate the empirical impact of SDIS on their performance, while comparing the algorithmic behaviour observed to the results obtained theoretically and on $f_0$. In Section~\ref{sec:conclusion}, we conclude that SDIS does indeed have an impact on performance, and should be taken into consideration more closely to improve the state of reproducibility in our field. We look ahead at potential solutions and future research directions in Section~\ref{sec:future}.

\section{Heuristic optimisation with box constraints}\label{sect:boxcons}

Optimisation problems faced by practitioners from different application fields are necessarily defined within a domain \textbf{D}, commonly referred to as the search space in the heuristic optimisation community. Indeed, in the real-world context, the presence of feasibility constraints is almost inevitable, and even when the nature of the problem seems to be unconstrained one may argue that when using heuristic approaches the need for sampling solutions and generating random numbers imposes boundaries for drawing such values. In this light, even when equality and/or inequality constraints are not present, it is generally assumed that each design variable of the problem at hand must be bounded between some lower and upper bound, thus defining a search space \textbf{D} shaped as a hyperparallelepiped (or as a hypercube when each design  variable is constrained within the same range). This is commonly referred to as `box-constrained' problem in computer science jargon. In this study, we focus on \textit{real-valued single objective box-constrained optimisation problems}, as defined and discussed in the next Section.

\subsection{Problem formulation, related constrained optimisation problems}\label{sect:boxformulation}
A \textit{real-valued box-constrained problem} is defined as finding a minimum of function  
\begin{equation}\label{eq:def_f}
    f: \mathbf{D} = \bigtimes_{i=1}^n [a_i,b_i] \rightarrow \mathbb{R}
\end{equation}
where $-\infty < a_i < b_i < \infty$ and $\mathbf{D}\subset \mathbb{R}^n$~\footnote{Mathematically, if $f$ can be extended to a larger domain, this problem can be rewritten as an \textit{unconstrained optimisation problem with in\-equa\-li\-ty
constraint}: \begin{equation} \arg\min_{x\in\mathbb{R}^n}{f(x)} \textnormal{ subject to } g(x)\leq0, \nonumber \end{equation}
where $g(x)=\mathbf{A}\bf{x}-b$ with $\mathbf{A}=[-\bf{I}_n,\bf{I}_n]^{\bf{T}}$ and $b=[-a_1,\dots,-a_n,b_1,\dots,b_n]^{\bf{T}}$, $g:\mathbb{R}^n\to\mathbb{R}^{2n}$, where $\bf{I}_n$ stands a unity matrix of size $n$. In other words, trivially, box constraints represent a special case of a set of linear constraints. However, \textit{in practice}, application of unconstrained optimisation methods might lead to poor results, e.g. depending on the way the function at hand is extended. \label{ftnt}}. This represents the lowest complexity of inequality constraint condition on the variables that a problem can have. For this reason, some confusion arises in the literature with several authors often referring to this class of problems as `unconstrained' to stress the fact that design variables are not subject to more complex linear or nonlinear constraints. However, we argue that this is incorrect as ignoring box-constrains is an oversimplification leading to confusion and reproducibility issues. It is indeed common to find articles in the literature where information on the employed Strategy for Dealing with Infeasible Solutions (SDIS, see definition in Section~\ref{sect:SDIS}) is omitted (see Section~\ref{sect:SDIS_literature}), even though recent studies indicate that different SDIS operators differently influence (at least) the structural bias \citep{Kononova2015} of a heuristic approach \citep{vStein2021_emergence,techrxiv_bias,bib:DEAnalysis2022}, and thus playing a role on the algorithmic behaviour of an optimisation algorithm. Hence, box-constraints \textit{should not} be ignored and solutions violating them are to be dealt with an appropriate SDIS.

In this light, a well-designed algorithm for real-valued unconstrained optimisation, i.e. where each design variable can be anywhere in the real axis ($x_i\in\mathbb{R}$), might not be as suitable for box-constrained optimisation. It should indeed be observed how some algorithms, see e.g. \citep{Kononova2020_outside}, are prone to produce high numbers of solutions outside the search domain under certain parameters configurations. This is quite likely to occur when an algorithm for unconstrained optimisation is used over a box-constrained domain. In this scenario, the algorithm has to be equipped with a SDIS which would have to be activated for the vast majority of objective function evaluations, thus leading the search and taking over the actual working logic of the algorithmic itself. To prevent this phenomenon, being \textit{constrained} would be a key feature of an optimisation algorithm for box-constrained problems.

\subsection{Classic and state-of-the-art versions of Differential Evolution}\label{sect:DEs}
Despite the numerous advances in the field of DE, its solid general algorithmic framework has remained quite unchanged since the first studies \citep{storn1996usage,storn1997differential}, with many of the most important variants being proposed by mainly acting on the mutation operator, where individuals are linearly combined, see  \citep{bib:Lampinen2000,price2006differential}, and on adding self-adaptation rules for its $3$ parameters \citep{DAS20161}. These are the population size $N$ and the two control parameters $F\in (0,2]$, acting as a \textit{scale factor} for the mutation operator, and the \textit{crossover rate} $C_r\in[0,1]$. The working mechanism of DE is quite known and established, and for general information one can see \citep{Caraffini2019,Kononova2020_outside,bib:DEAnalysis2022}, where description, pseudocode and analyses of its algorithmic behaviour are provided. However, for the sake of clarity, we briefly report relevant DE terminology which is used in the remainder of this paper.

In DE, the $N$ individuals in the population are processed one at a time. When the commonly called `current' individual to be perturbed is selected to undergo recombination, it gets referred to as the \texttt{target}. Through the crossover operator, which requires the availability of an `intermediate' solution referred to as the \texttt{mutant}, the \texttt{target} individual produces an offspring solution referred to as the \texttt{trial}. This new solution can have infeasible components to be dealt with an appropriate SDIS before its fitness value can be computed, as further commented in Section~\ref{sect:SDIS}. To implement this logic, a mutation strategy is required to produce the \texttt{mutant} solution. As previously mentioned, this operator works by linearly combining individuals selected from the population where a number of difference vectors are formed (from which the name of this optimisation paradigm) and added to a specific individual. The latter, as well as the number of difference vectors, depends on the adopted mutation strategy. With this in mind, classic DE variants are identified with the well-known notation \texttt{DE/a/b/c} where \texttt{a} indicates the mutation strategy, \texttt{b} the number of difference vectors employed in the mutation strategy, and \texttt{c} specifies the crossover strategy - two options are mainly used (i.e. the binomial \texttt{bin} and exponential \texttt{exp} crossover strategies) for \texttt{c} but a few more strategies also exist in the literature \citep{DAS20161,bib:DEAnalysis2022}. 

Some state-of-the-art DE algorithms, which we also study in this piece of research, slightly deviate from this structure. 
The Success-History based Adaptive DE (\texttt{SHADE}) \citep{bib:tanabe2013} can be seen as a variant of the popular \texttt{JADE} algorithm \citep{bib:zhang2009} where a memory system is introduced to store the weighted Lehmer average of successful $F$ values, and the weighted arithmetic average of successful $C_r$ values, from previous generations. Such values are randomly picked to adapt the control parameters, thus not relying only on the values from the previous generation (as in JADE) but also on the previous ones. Furthermore, the $p$ parameter for the `current-to-\textit{p}best' mutation strategy is randomly generated for each individual (this introduces an extra parameter $p_{min}$ to tune). These small changes, led to reported significant performance improvements with respect to previous established self-adaptive DE algorithms. When compared to the algorithmic structure of a classic DE, one can immediately observe clear difference for \texttt{SHADE}:
    \begin{itemize}
        \item control parameters are self-adapted;
        \item by design, a mechanism is in place for using an optional archive of less fit individuals to be entered in the population for preserving diversity;
        \item classic DE mutations are not employed, in favour of the `current-to-\textit{p}best/1' \citep{bib:zhang2009}, where the \textit{p}best vector is selected at random amongst the $p\%$ best individuals in the population, which is used in combination with the binomial crossover in \citep{bib:tanabe2013}.
    \end{itemize}
    The burden of tuning the population size is still present in \texttt{SHADE}, but is mitigated in its successor \texttt{L-SHADE} \citep{tanabe2014improving}, where an initial (usually large, i.e. $18\cdot n$) population size gets decreased linearly as a function of the number of fitness evaluations. Reducing the population size has shown to be beneficial in DE, see e.g. \citep{bib:zamuda2012}, and appears to make \texttt{L-SHADE} performs better in several benchmark problems.

\subsection{Infeasibility}\label{sect:infeasibility}
Referring to Eq.~\ref{eq:def_f}, a solution $\bf{x}\in\textbf{D}$ is said to be  \emph{feasible}, while it is \textit{infeasible} if $\bf{x}\not\in\textbf{D}$. In a box-constrained scenario, the last case occurs if at least one of its $\textnormal{i}^{th}$ design variables is either lower than $a_i$ or greater than $b_i$. Such \textit{infeasible} solutions cannot be evaluated in the vast majority of real-world applications, i.e. they represent physically impossible scenarios or require mathematically undefined calculations, and are purposely excluded by design. Also from the mathematical point of view, these solutions should not be considered as the function modelling the problem is undefined outside its domain - i.e. the problem does not exist outside \textbf{D}. Despite some confusion can arise while using common benchmark suite for optimisation such as e.g. \citep{hansen2021coco,wu2017problem}, which always return a value for $\bf{x}\not\in\textbf{D}$, these solutions should not be used to guide the search for solving test-bed problems.

The amount of infeasible solutions generated during the search depends both on the particularities of the problem (e.g. fitness landscape, problem size) and on the characteristics of the search heuristic. More specifically, the number of infeasible solutions increases with the problem size and with the probability $p_v$ of violating the bound constraints by a design variable, as the probability of generating an $n$-dimensional infeasible solution is $1-(1-p_v)^n$. The violation probability, $p_v$, depends on the distribution of the population elements in the bounding box and on the exact mutation or perturbation operator.

The question whether a well-performing algorithm should generate many infeasible points to solve the problem remains open: \citep{Boks2021} has demonstrated that highly competitive adaptive variants of the Differential Evolution algorithm (See Section~\ref{sect:DEs}) can indeed generate up to 93\% infeasible points throughout runs on more complicated BBOB functions. Similar results have been obtained in this paper (see Figure~\ref{fig:final_pois_30d_lshade}). With these results in mind, can we still claim that such optimisation methods \textit{efficiently utilise information contained within the population} if that many generated solutions need to be somehow brought back into feasibility? What \textit{actually} steers the search: optimisation algorithm or its feasibility-enforcing component? 

\subsection{Strategy of dealing with infeasible solutions}\label{sect:SDIS}
Following the discussions of Section~\ref{sect:boxformulation}, SDIS, also referred to as boundary constraint handling methods, are key operators for most algorithms and should be chosen accurately. The same way variation and recombination operators are carefully selected and combined during the algorithmic design phase, SDIS should too be considered in such process. The most logical activation of SDIS inside the algorithmic structure is before performing the objective function call, to make sure that the returned value is from a feasible solution. 

This is the approach followed in this study within Differential Evolution algorithm, where we follow the scheme depicted in the pseudocode from \citep{Kononova2020_outside}. Note that in DE, as in the vast majority of heuristics, this is recommended. As suggested in Section~\ref{sect:DEs}, in the DE framework there is only one operator that can produce an infeasible solution. After being generated, some of its components are transferred by a crossover operator to an existing individual, whose fitness value must be evaluated. Hence, placing a SDIS before crossover would too make sure that novel candidate solutions are feasible, but would also be unnecessary as the crossover might ignore completely most infeasible components from the \texttt{mutant}, which is never evaluated as being only an internal intermediate product. However, there might be some cases, as e.g. in some hybrid heuristic structures, where intermediary solutions are involved in driving the search process before a new individual is evaluated - which should not be allowed. So, in the most general case, one should always pay attention in activating SDIS every time a potentially infeasible solution is used to guide the search or has to be evaluated.

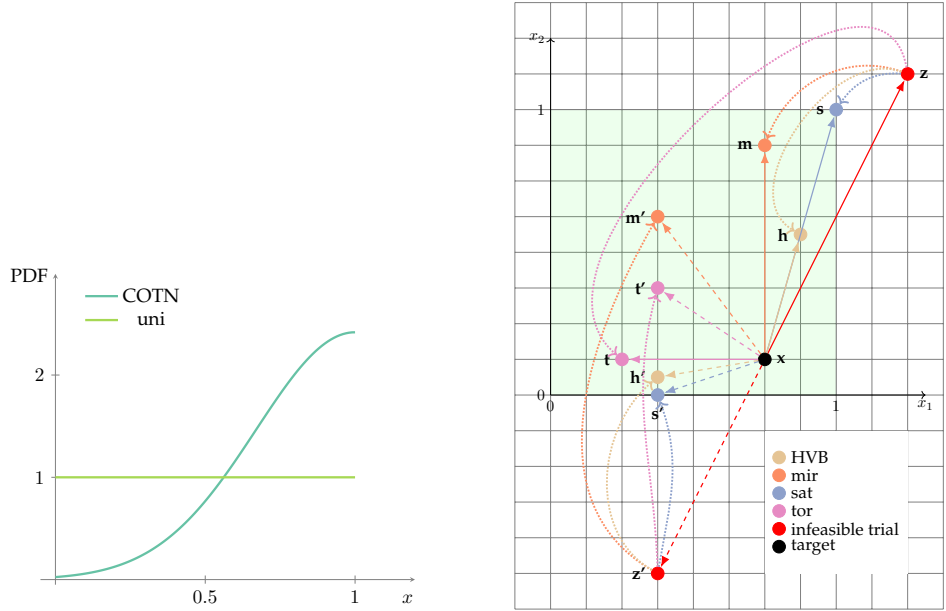
\begin{figure}[ht!]
    \centering
    \subfigure{\resizebox{0.45\textwidth}{!}{\tikzstyle{st2}=[rectangle, inner sep=0pt, minimum height=0pt, minimum width=4pt, draw]

\definecolor{cotn}{rgb}{0.4, 0.7607843137254902, 0.6470588235294118}
\definecolor{uni}{rgb}{0.6509803921568628, 0.8470588235294118, 0.32941176470588235}
\pgfmathdeclarefunction{gauss}{2}{\pgfmathparse{1/(#2*sqrt(2*pi))*exp(-((x-#1)^2)/(2*#2^2))}}
\pgfmathdeclarefunction{gauss2}{2}{\pgfmathparse{2/(#2*sqrt(2*pi))*exp(-((x-#1)^2)/(2*#2^2))}}
\pgfmathdeclarefunction{uniform}{3}{\pgfmathparse{(#1>=#2)*(#1<#3)*1/(#3-#2)}}
\newcommand*\GnuplotDefs{
    set samples 30;
    cdfn(x,mu,sd) = 0.5 * ( 1 + erf( (x-mu)/sd/sqrt(2)) );
    pdfn(x,mu,sd) = 1/(sd*sqrt(2*pi)) * exp( -(x-mu)^2 / (2*sd^2) );
    tpdfn(x,mu,sd,a,b) = pdfn(x,mu,sd) / ( cdfn(b,mu,sd) - cdfn(a,mu,sd));
}

\begin{tikzpicture}[scale=1,
uninode/.style={shape=rectangle, inner sep=0pt, minimum height=0pt, minimum width=4pt, thick, draw = uni},
cotnode/.style={shape=rectangle, inner sep=0pt, minimum height=0pt, minimum width=4pt, thick, draw = cotn}]
\begin{axis}[
    axis lines = center,
    axis line style = help lines,
    xmin = {-0.05}, xmax = {1.2}, ymin={-0.05}, ymax={2.99},
    xtick={0,0.5,1},
    xlabel={$x$}, ylabel={PDF},
    x label style={at={(axis description cs:0.98,-0.01)},anchor=north},
    y label style={at={(axis description cs:-0.03,0.95)},anchor=south},
    legend style={at={(0.1,0.89)},anchor=west,draw=none},
    legend entries={COTN,uni}
]

\addplot[domain=0:1,very thick,samples=100,cotn]{gauss2(1,0.33)};
\addplot[domain=0:1,very thick,samples=100,uni]{uniform(x,0,1)};

\end{axis}
\end{tikzpicture}}}\label{fig:a}
    \hspace{0.75cm}
    \subfigure{\resizebox{0.45\textwidth}{!}{\begin{tikzpicture}[x=0.8cm,y=0.8cm, point/.style = {circle, draw=#1, inner sep=0.1cm, fill=#1, color=#1, node contents={}}, every label/.append style = {font=\large}]]
           
\definecolor{sat}{rgb}{0.5529411764705883, 0.6274509803921569, 0.796078431372549}
\definecolor{mir}{rgb}{0.9882352941176471, 0.5529411764705883, 0.3843137254901961}
\definecolor{tor}{rgb}{0.9058823529411765, 0.5411764705882353, 0.7647058823529411}
\definecolor{hvb}{rgb}{0.8980392156862745, 0.7686274509803922, 0.5803921568627451}

\def\xmin{-1}
\def\xmax{11}
\def\ymin{-6}
\def\ymax{11}

\draw[style=help lines, ystep=1, xstep=1] (\xmin,\ymin) grid
  (\xmax,\ymax);

\draw[->] (0,0) -- (10.5,0) node[below] {$x_1$};
\draw[->] (0,0) -- (0,10) node[left] {$x_2$};

\node at (0, 0) [below] {$0$}; \node at (8, 0) [below] {$1$}; 
\node at (0, 0) [left] {$0$}; \node at (0, 8) [left] {$1$}; 
\draw[fill=green!70, opacity = 0.1, draw=black] (0,0) rectangle (8,8);
    
\node (x) at (6,1)      [point=black, label=right:$\textbf{x}$];

\node (z2) at (10,9)    [point=red, label=right:$\textbf{z}$];
\node (S2) at (8,8)     [point=sat, label=left:$\textbf{s}$];
\node (M2) at (6,7)     [point=mir, label=left:$\textbf{m}$];
\node (H2) at (7,4.5)   [point=hvb, label=left:$\textbf{h}$];
\node (T2) at (2,1)     [point=tor, label=left:$\textbf{t}$];

\draw[-{Latex[length=2.5mm]}, thick, red] (x) -- (z2);
\draw[-{Latex[length=2.5mm]}, thick, sat] (x) -- (S2);
\draw[-{Latex[length=2.5mm]}, thick, mir] (x) -- (M2);
\draw[-{Latex[length=2.5mm]}, thick, hvb] (x) -- (H2);
\draw[-{Latex[length=2.5mm]}, thick, tor] (x) -- (T2);
\draw[->,densely dotted,very thick,sat] (z2) to[bend right] (S2);
\draw[->,densely dotted,very thick,mir] (z2) to[out=160,in=70] (M2);
\draw[->,densely dotted,very thick,hvb] (z2) to[out=160,in=150] (H2);
\draw[->,densely dotted,very thick,tor] (z2) to[out=100,in=140] (T2);

\node (z1) at (3,-5)    [point=red, label=left:$\textbf{z'}$];
\node (S1) at (3,0)     [point=sat, label=below:$\textbf{s'}$];
\node (M1) at (3,5)     [point=mir, label=left:$\textbf{m'}$];
\node (H1) at (3,0.5)   [point=hvb, label=left:$\textbf{h'}$];
\node (T1) at (3,3)     [point=tor, label=left:$\textbf{t'}$];

\draw[-{Latex[length=2.5mm]}, thick, dashed, red] (x) -- (z1);
\draw[-{Latex[length=2.5mm]}, thick, dashed, sat] (x) -- (S1);
\draw[-{Latex[length=2.5mm]}, thick, dashed, mir] (x) -- (M1);
\draw[-{Latex[length=2.5mm]}, thick, dashed, hvb] (x) -- (H1);
\draw[-{Latex[length=2.5mm]}, thick, dashed, tor] (x) -- (T1);
\draw[->,densely dotted,very thick,sat] (z1) to[out=80,in=290] (S1);
\draw[->,densely dotted,very thick,mir] (z1) to[out=150,in=240] (M1);
\draw[->,densely dotted,very thick,hvb] (z1) to[out=150,in=220] (H1);
\draw[->,densely dotted,very thick,tor] (z1) to[out=92,in=255] (T1);

\fill[fill=white, opacity = 1] (6,-1) rectangle (10,-5);
\node at (6.4,-1.75) [point=hvb, label=right:HVB];
\node at (6.4,-2.25) [point=mir, label=right:mir]; 
\node at (6.4,-2.75) [point=sat, label=right:sat]; 
\node at (6.4,-3.25) [point=tor, label=right:tor]; 
\node at (6.4,-3.75) [point=red, label=right:infeasible trial]; 
\node at (6.4,-4.25) [point=black, label=right:target];

\end{tikzpicture}}}\label{fig:b}
    \caption{(a) Sampling distributions for the application of stochastic per-component SDIS \texttt{COTN} and \texttt{uni} in case of boundary violation on the right for the 1-dimensional unit domain; for boundary violation on the left, \texttt{COTN} curve will be symmetric with respect to $x=0.5$, while \texttt{uni} curve will be the same. (b) Effect of the application of de\-ter\-mi\-nis\-tic per-component SDIS \texttt{HVB}, \texttt{mir}, \texttt{sat}, \texttt{tor} (dotted lines) in case of one (dashed lines) and two (solid lines) infeasible components of a trial vector generated from a target vector for 2-dimensional unit hypercube domain (green area). }\label{fig:strategies}
\end{figure}

For this study, we select a varied range of existing strategies of dealing with infeasible solutions:
 \begin{itemize}
     \item `complete one-sided truncated normal', first introduced in \citep{Caraffini2019}, which is denoted as \texttt{COTN} in the remainder of this investigation;
     \item `halfway-to-violated-bounds', denoted as \texttt{HVB} here, which we define as the operator replacing infeasible components with the midpoint between the previous feasible components (before perturbation) and the violated problem's bound;
     \item  `mirror', as described in \citep{Kononova2020CEC,Kononova2020PPSN},
     which we denote as \texttt{mir} in this study;
     \item `saturation', see \citep{Caraffini2019} for pseudocode, 
     which is denoted as \texttt{sat} here;
     \item `toroidal', see \citep{Caraffini2019} for pseudocode, 
     which is denoted as \texttt{tor} here; 
     \item `uniform', as defined in \citep{bib:DEAnalysis2022}, which we referred to as \texttt{uni} in this article.
 \end{itemize}

Graphically, these employed SDIS are explained in Figure~\ref{fig:strategies}, which bring to attention the stochastic nature of \texttt{COTN} and \texttt{uni} while all the remaining strategies deterministically return the same feasible value when the same infeasible value is presented as input. Similarly, one can also observe that the way \texttt{COTN} operates resembles a stochastic counterpart of the \texttt{mir} strategy. These selected SDIS operators cover multiple and commonly used working mechanisms for dealing with infeasible solutions, which might appear in the literature under different names (see e.g. those reported in Table~\ref{tab:PreviousStudied} and Section~\ref{sect:SDIS_literature}).

It is also worth clarifying that, in the context of DE, \texttt{HVB} acts on the \texttt{trial} individual by using uses the \texttt{target} individual as its feasible counterpart. Hence, before calling the objective function, any infeasible design variable of the \texttt{trial} solution would get replaced with a feasible value located halfway from the position of the corresponding component in the \texttt{target} solution and the violated bound (upper or lower). In this light, this SDIS can be seen as a `component-wise' counterpart of the `projection to midpoint' repair strategy for DE used in \citep{bib:BIEDRZYCKI2019}, where the \texttt{mutant} vectors are manipulated into feasible \texttt{trial} solutions as further discussed in Section~\ref{sect:SDIS_literature}.

\subsection{State-of-the-art on strategies of dealing with infeasible solutions}\label{sect:SDIS_literature}
{\small 
\begin{table}[!tb]
\caption{Summary of ways in which infeasible solutions aspect is addressed in recent works proposing new DE variants\label{tab:source_code}}
\begin{tabular}{p{21mm}p{42mm}p{60mm}}
\toprule
\textbf{Source code not available} & {\textbf{1.} SDIS is mentioned in the paper} & \citep{2Deng2020}, \citep{3CHENG2021}, \citep{3wagdy2021}: \texttt{re\-ini\-tia\-li\-za\-tion};
  \citep{3liu2019}, \citep{3Stanovov2020}: \texttt{HVB};
 , \citep{3zhan2020}, \citep{3zhao2020}: \texttt{sa\-tu\-ra\-tion};
  \citep{2DENG2022}: \texttt{midpoint-base}.\\ 
\cmidrule{2-3}
  & \textbf{2.} The proposed algorithm is derived from SHADE, JADE and it might be assumed that SDIS is inherited & \citep{1Awad2018}, \citep{1CHENG2020}, \citep{2Meng2020}, \citep{1Wenchao2021}, \citep{2Zhong2021}, \citep{2KUMAR2022} \citep{1Zuo2022}: \texttt{HVB}.\\ 
\cmidrule{2-3}
  & {\textbf{3.} The proposed algorithm is a new or an enhanced DE variant and SDIS is not mentioned} & \citep{2TIAN2019}, \citep{2Choi2020}, \citep{2Mousavirad2020}, \citep{2Sun2020}, \citep{2WANG2020}, \citep{2Zhou2020}, \citep {2Mousavirad2021}, \citep{2SONG2021}.\\
\midrule
  {\textbf{Source code available}} & \textbf{1.} SDIS is mentioned in the paper & \citep{3wagdy2018}, \citep{3wagdy2019}:
  \texttt{re\-ini\-tia\-li\-za\-tion}, \citep{22brest2020}: \texttt{mirror}.\\
\cmidrule{2-3}
  & {\textbf{2.} SDIS is not mentioned in the  paper, but used in implementation} & \citep{Tomczak2020}: \texttt{sa\-tu\-ra\-tion}.\\
\bottomrule
\end{tabular}
\end{table}
}
Carried out as a part of the current paper, a review on recent publications which propose new or improved DE variants revealed that only a small proportion of papers consider the strategy of dealing with infeasible solutions as a mechanism influencing the search process and describe explicitly the used SDIS. As is illustrated in Table~\ref{tab:source_code}, several categories of papers have been identified in case of DE.

On one hand, most of the papers do not provide access to the source code containing the implementation of the proposed algorithm~\footnote{This, on it's own, implies reproducibility issues.}. In this case, the only source of information is the algorithm description provided in the paper. In few cases when the SDIS is explicitly specified in the paper, no strong motivation on its choice is typically provided (simplicity or popularity, being typically mentioned) and its influence on the algorithm behaviour is not discussed.
Another category of papers are those presenting variants of a state-of-the-art method (e.g. JADE or SHADE) and the reader might infer, in the absence of an explicit statement on SDIS, that the strategy used in the original algorithm (e.g. the so-called \texttt{midpoint-to-target}, or \texttt{HVB} in the terminology of this paper) is used in the proposed variant. However, this is just the guess of the reader and the reproducibility of the results is at least questionable. The third category of papers, which seems to be the most numerous one, includes descriptions of algorithms without any specification on how the solutions violating the box constraints have been treated. 

On the other hand, there are papers for which the source code is made available, thus even if the SDIS is not described in the paper it can be identified in the code. However the reasoning behind choosing one strategy over the other ones is still missing.

It should be however mentioned that there are several works devoted to the com-
\begin{landscape}
\centering
\begin{adjustbox}{width=1.6\textwidth}
\begin{threeparttable}
{\small
\caption{Previous studies on the influence of SDIS on DE, with dictionary for SDIS alternative terminology\label{tab:PreviousStudied}}
\begin{tabular}{p{0.085\linewidth}p{0.13\linewidth}p{0.13\linewidth}p{0.145\linewidth}p{0.15\linewidth}p{0.145\linewidth}p{0.11\linewidth}p{0.13\linewidth}}
\toprule
\textbf{Paper} & \cite{Arabas2010} & \cite{Padhye2015} & \cite{Kreischer2017} & \cite{bib:BIEDRZYCKI2019}& \citep{Martinez2020} & \cite{Boks2021}& \multicolumn{1}{c}{This paper} \\ 
\midrule
\textbf{Benchmark} & CEC 2005 & 4 functions\footnotemark & CEC 2017 & CEC 2017 & real-world & BBOB & BBOB \\ 
\midrule
\textbf{Standard DE variants, $(F,C_r)$} & DE/rand/1/bin,  $(0.8,0.9)$ & DE/best/1/bin, $(0.7,0.5), (0.8,0.9)$\footnotemark & DE/rand/1/bin, $(0.9,1)$, DE/local-to-best/1/bin & DE/target-to-best/1/bin, $(0.8,0.9)$ & DE/rand/1/bin, (${\cal U}[0.3,0.9]$, ${\cal U}[0.8,1]$) &  15 mutation $\times$ 2 crossover ope\-ra\-tors, SHADE adaptation & DE/rand/1/ with 2 crossovers, 25-50 uniformly spaced $(F,C_r)$ values\\ 
\midrule
\textbf{Advanced DE variants} & -- & -- & -- & SADE, JADE, jSO, DES, BBDE & -- & -- & SHADE, L-SHADE\\ 
\midrule
\textbf{Coordinate-} & reflection & -- & reflection & reflection  & reflection & reflection & mirror \\ 
\cmidrule{2-8}
\textbf{wise SDIS} & projection & set on boundary & projection & projection & projection & projection & saturation \\ 
\cmidrule{2-8}
& wrapping & periodic & -- & wrapping & -- & wrapping & toroidal \\ 
\cmidrule{2-8}
& reini\-tia\-lisation & random & reini\-tia\-lisation & reini\-tia\-lisation & random & reini\-tia\-lisation & uniform \\ 
\cmidrule{2-8}
& -- & -- & -- & midpoint to target & midpoint target & midpoint target & halfway \\ 
\cmidrule{2-8}
& -- & -- & midpoint to base & midpoint to base & -- & midpoint base & -- \\ 
\cmidrule{2-8}
& -- & expo\-nen\-tia\-lly confined, spread & -- & -- & -- & -- & COTN \\ 
\cmidrule{2-8}
& -- & -- & -- & -- & other SDIS\footnotemark & other SDIS\footnotemark & -- \\ 
\midrule
\textbf{Vector-wise SDIS} & -- & shrink\footnotemark & scaled mutant\footnotemark & projection to midpoint or base\footnotemark & centroid $K+1$ \footnotemark & -- & -- \\ 
\cmidrule{2-8}
& -- & inverse parabolic (confined, spread) & -- & -- & -- & -- & -- \\ 
\bottomrule
\end{tabular}
\begin{tablenotes}
\item[5] Unimodal test functions (Ellipsoidal, Schwefel, Ackley, Rosenbrock)
\item[6] For $n=20$ and $n\in\{50, 100, 200, 300, 500\}$, respectively
\item[7] Conservatism, resampling, evolutionary
\item[8] Death-penalty, resampling, boundary-transformation, rand-base, projection-to-midpoint, projection-to-base
\item[9] Linear combination between the trial and base elements
\item[10] Linear combination between the trial and a reference element in the feasible region
\item[11] Linear combination between the trial and the domain midpoint or base element
\item[12] Linear combination between feasible and component-wise corrected elements
\end{tablenotes}
}
\end{threeparttable}
\end{adjustbox}
\end{landscape}

\noindent parative analysis of different strategies to deal with infeasible solutions in the context of various metaheuristics: CMA-ES \citep{Wessing2013,Biedrzycki2020}, Particle Swarm Optimization \citep{Cheng2011,Helwig2013,Castillo2017,Oldewage2018}, Differential Evolution (for DE see a summary in Table~\ref{tab:PreviousStudied}).

In relation to Differential Evolution, the first paper presenting a comparison between the performance of various SDISs applied to \texttt{DE/rand/1/bin} \citep{Arabas2010} analyses the following SDISs: \texttt{sat} (referred to as \texttt{projection}), \texttt{tor} (referred to as \texttt{wrapping}), \texttt{uniform-resampling} (referred to as \texttt{reinitialisation}) and \texttt{mir} (referred to as \texttt{reflection}). The main observation is that the choice of SDIS might have an influence on the DE performance but the amount of impact depends on the problem characteristics (e.g. position of the optimum and problem size): (i) \texttt{sat} and \texttt{mir} work well when the optimum is near the bounds; (ii) for small size problems (e.g. $n=10$) the amount of corrected elements is not significantly influenced by the used SDIS and there are no significant differences between the performance of various SDISs; (iii) for larger size problems (e.g. $n=30$) a higher effectiveness has been observed for \texttt{sat} and \texttt{mir} when compared with \texttt{uniform-resampling}. 
     
In the study \citep{Padhye2015} addressing the influence of SDIS on Particle Swarm Optimisation, Differential Evolution (\texttt{DE/best/1/bin}) and Genetic Algorithms, it is stated that deterministic methods, as for instance \texttt{sat}, lead to a loss in population diversity, while \texttt{random-reinitialisation} loses useful information carried by the current population. The main remark on the performance of DE combined with a SDIS is that when the optimum is near the midpoint of the feasible domain there is no signi\-fi\-cant difference between the impact of various strategies. On the other hand, when the optimum is close to the boundary then a parameterised vector-wise non-deterministic strategy (\texttt{inverse-parabolic}) behaves the best with the experimental setup of the paper.

Currently the most extensive study on boundary constraints handling for standard DE is \citep{bib:BIEDRZYCKI2019} where experimental results are presented based on CEC 2017 benchmark suite and is analysed the impact of different strategies of dealing with infeasible solutions (penalty functions, repairing methods, feasibility preserving mutation) on the dynamics of the population, convergence speed and global optimisation efficiency. The main insights reported in \citep{bib:BIEDRZYCKI2019} are: (i) the highest influence on the mean and variance of the mutant population distribution is induced by \texttt{uniform-resampling} and \texttt{tor} SDISs; (ii) the performance sensitivity of standard DE to SDIS is higher for larger size problems; (iii) the adaptive variants (e.g. JADE, SADE and jSO) are less sensitive to the choice of SDIS than non-adaptive DE; (iv) the influence of a SDIS depends on the DE variant with which it is combined, but overall the best behaviour is induced by \texttt{midpoint} strategies (which use values between the corresponding component of the target or base vector and the violated bound) and by \texttt{mir}. The authors of \citep{bib:BIEDRZYCKI2019} consider that the results of the experimental analysis can be explained by the fact that SDISs with better performance lead to a lower discrepancy between the distributions of the before and after correction populations. 
It should be however emphasised that best performance is obtained when the mutant construction is repeated by sampling new parents until a feasible trial vector is obtained. 

Same conclusion has been reported also in \citep{Kreischer2017} where \texttt{DE/target-to-best/1/bin} algorithm (with $F=0.8$ and $C_r=0.9$) is combined with the same SDISs mentioned above and tested on CEC 2017 benchmark. The authors of \citep{Kreischer2017} also recommends \texttt{mir} and \texttt{projection} to an interior point of the feasible region as well performing strategies in the case of CEC 2017 benchmark, followed by \texttt{HVB}.

An experimental study on the influence of nine strategies is presented in \citep{Martinez2020} aiming to deal with boundary constraints when combined with the so-called Deb feasibility rules \citep{DEB2000b} to solve four real-world constrained optimization problems related to mechanical design. The analysed strategies are \texttt{midpoint-target} (\texttt{HVB}), \texttt{reflection} (\texttt{mir}), \texttt{projection} (\texttt{sat}), \texttt{random-scheme} (\texttt{unif-resample}), \texttt{full reinitialisation} (all components, including the feasible ones, are randomly reinitialised), \texttt{conservatism} (the trial vector is just discarded), \texttt{resampling} (a new mutant is constructed by using other randomly selected parents), \texttt{evolutionary}~\footnote{Name as given by \citep{Martinez2020}} (the infeasible component is replaced with a convex combination between the violated bound and the corresponding component from the best element in the population), \texttt{centroid K+1} (average of a set including an element selected from the populations and $K$ other elements obtained by applying \texttt{unif-resample} to the infeasible components). The main conclusion is that the influence of the SDIS is highly dependent on the problem to be solved but overall, \texttt{sat} proved to be the most effective strategy.

Similar to the setup in \citep{bib:BIEDRZYCKI2019}, \citep{Boks2021} has investigated the effect of SDIS on performance on the single-objective noiseless version of the BBOB benchmark~\citep{finck2010real} in DE for a wide selection of operators, in a fully modular fashion: \texttt{rand/1}, \texttt{best/1}, \texttt{target-to-best/1}, \texttt{best/2}, \texttt{rand/2}, \texttt{target-to-best/2}, \texttt{target-to-$p$best/1}, \texttt{rand/2/dir}, \texttt{NSDE}, \texttt{trigonometric}, \texttt{2-opt/1}, \texttt{2-opt/2}, \texttt{proximity-based rand/1}, \texttt{ranking-based-target\--to\--$p$best/1} mutations with \texttt{bin} and \texttt{exp} crossovers with SHADE-based adaptation of control parameters for a wide range of SDIS methods: \texttt{death-penalty}, \texttt{resampling}, \texttt{reinitialisation}, \texttt{pro\-jec\-tion}, \texttt{reflection}, \texttt{wrapping}, \texttt{boundary-transformation}, \texttt{rand\--base}, \texttt{midpoint-base}, \texttt{midpoint-target}, \texttt{projection-to-midpoint}, \texttt{pro\-jec\-tion-to-base}, \texttt{conservatism}. This paper appears to be the only analysis of performance dependency on SDIS on the BBOB benchmark. The main conclusions of \citep{Boks2021} are: (i) no SDIS appears to be optimal for all DE configurations considered; (ii) SDIS ranks differ greatly between configurations and BBOB function groups; (iii) to some extent, the best SDIS tends to depend on crossover. As a rule of thumb, for similar setups,  practitioners are therefore advised to consider \texttt{conservatism} for \texttt{exp} crossover and \texttt{reinitialisation} for \texttt{bin} crossover as they perform best with many configurations; such policy, however, does not always give the optimal result. Op\-tio\-nal\-ly, \texttt{midpoint-target} in \texttt{bin} configurations rarely performs best but always performs well and \texttt{projection-to-midpoint} is a reliable second option for \texttt{exp} configurations. Finally, over all cases considered, \texttt{resampling} SDIS has been deemed successful the highest number of times.

As regards DE implementations incorporated in various popular open-source \textit{libraries}, the most common SDIS is \texttt{uniform-reinitialisation}, as in SciPy~\footnote{Package \texttt{scipy.optimize.differential\_evolution} \citep{2020SciPy-NMeth}},

PyMOO~\footnote{\url{pymoo.org/algorithms/soo/de.html}} and PAGMO~\footnote{\url{esa.github.io/pagmo2/docs/cpp/algorithms/de}}, followed by \texttt{sat}, as in MOEA framework~\footnote{\url{moeaframework.org}}, PyADE~\footnote{\url{github.com/xKuZz/pyade}}, and \texttt{reflection} in PyMOO. Finally, a \textit{notable exception} is a highly modular AutoDE lib\-ra\-ry~\footnote{\url{github.com/rickboks/auto-DE} \citep{Boks2021,BoksThesis}} based on the aforementioned paper \citep{Boks2021} which provides many standard DE configurations and a hyper adaptive version of SHADE \citep{BoksThesis}, all with 13 SDIS variants (see the list above).

\section{Search direction}\label{sect:SearchDirection}

Any heuristic iterative optimiser can be considered an \textit{adaptive sampler} which is guided according to some logic by differences in values of objective function (or derivatives thereof) evaluated in the previously sampled points. It therefore makes sense to consider a `path' taken by an optimiser in the search domain or, more practically, a sequence of search directions over sampled points. 

This section defines the search direction induced by the DE operators and proposes to quantify the influence of a SDIS by computing the cosine similarity between the unconstrained search direction and the search direction resulting after applying a SDIS.

\subsection{Definition of search direction in Differential Evolution}\label{sect:DESearch}
While it is easy to define the search direction for iterative single-solution methods, such task gets excessively complicated in case of general population-based iterative heuristic optimisers where multiple solutions steer the generation of subsequent solutions.
 However, Differential Evolution (see Section~\ref{sect:DEs}) lends itself to such analysis easier thanks to its survivor selection mechanism based on a `1-to-1 spawning' -- eve\-ry new solution updates its direction from the direction of a parent and the whole population represents a repeatedly updated \textit{ensemble} of search directions with a clear `inheritance' scheme -- note that this happens despite the fact that new solutions also incorporate information from other solutions in the population apart from their direct parent and implicitly capitalise on the information contained in the population, following the spirit of population-based heuristic optimisation \citep{bib:Bennett2010}.

In this context, we consider for each population element, interpreted as a target element, a search direction is defined as the difference between the corresponding trial element and the target element. When an unfeasible trial element is corrected by applying a SDIS the search direction will be altered. 
Changes in such `sequence of search directions' can then be measured via, e.g. cosine similarity. 

\subsection{Cosine as a measure of similarity between search directions}\label{sec:cs} 

The Cosine Similarity (CS) between two non-zero vectors $\mathbf{v_1}$ and $\mathbf{v_2}$ is defined as their `normalised' inner product:
\begin{equation}\label{eq:cs}
\textnormal{CS}(\mathbf{v_1},\mathbf{v_2})=\frac{\mathbf{v_1}^\textnormal{T}\mathbf{v_2}}{\|\mathbf{v_1}\|\|\mathbf{v_2}\|} 
\end{equation}
thus not depending on their magnitudes and registering only differences based on the angle $\theta$ between them. For this reason, it can be seen as an angular distance, and can be used to determine whether two directed vectors $\mathbf{v_1}$ and $\mathbf{v_2}$ are pointing to the same oriented direction. Indeed, let us note that  $\textnormal{CS}(\mathbf{v_1},\mathbf{v_2})=\cos(\theta)\in[-1,1]$ and observe that:
\begin{itemize}
    \item  $\textnormal{CS}(\mathbf{v_1},\mathbf{v_2})=0\iff \mathbf{v_1}\textnormal{ and }\mathbf{v_2}$ are orthogonal  (i.e. the most dissimilar); 
   \item $\textnormal{CS}(\mathbf{v_1},\mathbf{v_2})=1\iff \mathbf{v_1}\textnormal{ and }\mathbf{v_2}$ are parallel (i.e. aligned and pointing to the same oriented direction);
      \item $\textnormal{CS}(\mathbf{v_1},\mathbf{v_2})=-1\iff \mathbf{v_1}\textnormal{ and }\mathbf{v_2}$ are anti-parallel (i.e. have opposite oriented directions).
\end{itemize}

In this study, we detect changes in the \textit{search direction} during the evolution in DE algorithms through the CS between two vectors~\footnote{For feasible solutions, CS values are not computed and thus excluded from the analysis.}, namely:
\begin{itemize}
    \item $\mathbf{d}$: obtained as the difference between the trial individual (before a SDIS is applied) and the target vector - i.e. the \texttt{target-to-trial} directed vector,
    \item $\mathbf{d_C}$: obtained as the difference between the trial individual (after a SDIS is applied) and the target vector - i.e. the \texttt{target-to-feasibleTrial} directed vector,
\end{itemize}
to \textit{observe if the employed SDIS is responsible for a change in the search direction}. In a single run, CS values are computed after a new trial individual is available, thus generating a sequence of angular distances over time whose length is equal to the computational budget used minus the population size. Sequences obtained across multiple runs for the same algorithmic configuration are then aggregated for analysis purposes as indicated in Section~\ref{sect:CSAnalysis}.

\subsection{Strategy of dealing with infeasible solutions as source of search disruptiveness}\label{sect:disrupt} 
In the case of an unconstrained search, Differential Evolution can generate, for given values of the control parameters, only a finite set of trial elements which define the corresponding set of search directions. When a SDIS is incorporated in DE, these search directions can be altered through the introduction of new moves/perturbations, hence a SDIS can be viewed as a source of disruptiveness in the DE search process. When analyzing the impact of a SDIS on the search process there are at least \textit{two questions} that arise: (i) how much is the search process altered? (ii) does a SDIS have a beneficial or a detrimental influence on the performance of the search?

The amount of disruptiveness depends both on the amount of components in the trial elements which are corrected by the SDIS and on the characteristics of the SDIS. One way to quantify the influence of the SDIS on the search process is to compute the cosine between the uncorrected and the corrected search directions, as it is easy to compute and can provide useful information even for high dimensions. As this value is closer to one, more of the search direction is preserved. Besides the influence on the search direction, different SDISs lead to different positions of the corrected elements and, as the chance of generating such elements through the DE mutation (based on the current population distribution and values of the control parameters) is smaller, the disruptive character of the SDIS can be considered more significant.

The second question is more difficult to answer, as the interference between the SDIS and the DE mechanisms can hinder but also help the search. At a first sight, at least the coordinate-wise SDISs are characterised by an inappropriate usage of the information contained in the population, as it alters the self-referential search direction induced by the difference-based mutation and decouples the components (with impact on the ability of some DE variants to be rotationally invariant \citep{Bujok2014,Fabio2018}). From this point of view, the vector-wise SDISs, as that proposed in \citep{Kreischer2017}, can be considered less disruptive.

On the other hand, interfering with the DE search might be beneficial at least with respect to: (i) generation of trial elements which would otherwise not be generated by DE operators, thus increasing in this way the pool of candidate solutions (this might be particularly useful in the case of small populations); (ii) influence on the population diversity by including components/elements which do not fully rely on differences (this might be useful to avoid premature convergence or stagnation which are two of the main causes for lack of performance of DE). So we can say that a \textit{SDIS might turn into a diversity increasing mechanism} which, however, is \text{blind}. Some of these aspects are discussed in Section~\ref{sect:diversityTheor} and illustrated further in Section~\ref{sect:diversity}, but there are some more complex issues related to the control of diversity which are not discussed here.

\section{Theoretical analysis of strategies of dealing with infeasible solutions}\label{sect:theoretic} 

Under the usual assumption that the scale factor $F$ is less than one, all mutant vectors generated by \texttt{DE/rand/1} will have the components inside the extended domain $[2a_i-b_i,2b_i-a_i]$, thus the infeasible components will belong to $[2a_i-b_i,a_i)\cup (b_i,2b_i-a_i]$~\footnote{Objective function $f: \mathbf{D} = \bigtimes_{i=1}^n [a_i,b_i] \rightarrow \mathbb{R}$, see Eq. ~\ref{eq:def_f}}. To quantify the influence of a SDIS on the search dynamics, several elements can be taken into account: (i) the amount of corrected components (in the case of component-wise strategies); (ii) the difference between the infeasible trial element and the corrected one; (iii) the (dis)similarity between the search direction as it would be in an unconstrained search  and the corrected search direction; (iv) the impact of SDIS on the population diversity. Different SDISs can behave differently with respect to these aspects. A theoretical analysis, even if conducted under some simplifying assumptions, might allow to extract some insights which could explain empirical observation or provide guidelines in selecting a SDIS.

The results presented in this section correspond to \texttt{DE/rand/1/*}~\footnote{Symbol * means that either \texttt{bin} or \texttt{exp} crossover can be used.} and are obtained under some \textit{simplifying assumptions}: (i) absence of a selection pressure (this is equivalent with using a flat fitness function, i.e. it is considered that all new trial elements are accepted as soon as they are generated); (ii) the analysis of population diversity is conducted component-wise; (iii) the population of current elements is almost uniformly distributed. The last assumption is used only for the estimation of the probability to generate infeasible components and of the variance of the population of elements which are corrected by using \texttt{mir}.

Furthermore, results discussed in this section under the aforementioned simplifying assumptions are contrasted with empirical results on benchmark functions in Sections~\ref{sect:exponf0},~\ref{sect:BBOBexperiments} which are generally free of such assumptions. Formal proofs for statements in this section are provided in Appendix ~\ref{sect:appendix_proofs}.

\subsection{Amount of infeasible components \label{sect:ViolProb}}
The probability that a \texttt{DE/rand/1} trial element contains an infeasible component depends on the mutation probability ($p_m$) and on the probability to generate an outside the bounds component ($p_v$). The mutation probability depends on $C_r$, on the problem size, $n$, and on the crossover type \citep{zaharie2009influence}. On the other hand, as long as the current population is almost uniformly distributed, the probability that a component violates the bounds is close to $F/3$ \citep{Zaharie2017}. The distribution of the trial population is influenced both by the mutation operator and by the used SDIS. One question is if the usage of a SDIS in the current generation increases or decreases the risk of generating infeasible components in the next generation. This question is rather difficult to answer in the general case, but it can be addressed in the case of \texttt{sat} strategy.  In this case, the current population, $P$, consists of three subpopulations: $P_w$ (elements within the bounds), $P_{lb}$ (elements on the lower bound) and $P_{ub}$ (elements on the upper bound). If $p_v(g)$ denotes the violation probability corresponding to generation $g$ (for the initial population it is considered that $p_v(1)=F/3$) then it can be proved (Proposition~\ref{prop:ViolProb} in Appendix~\ref{sect:appendix_proofs}), under the assumption that $P_w$ remains almost uniformly distributed, that \begin{equation}p_v(g+1)=p_v(g)/2+(1-p_v(g))(p^2_v(g)F/4+(1-p_v^2(g))F/3)
\end{equation}
This sequence, $p_v(g)$, of probabilities converges to a value which is between $F/3$ and $2F/3$ (see Figure~\ref{fig:ViolProbSaturate} in Appendix~\ref{sect:appendix_proofs}). 
This suggests that, in the absence of selection pressure, \texttt{sat} increases the risk of generating infeasible components, but not by a large amount. 
This is confirmed by the empirical results presented in Section \ref{sect:POIS_f0}, Figure \ref{fig:ViolProbEmpirical}.

\subsection{Difference between infeasible and corrected elements}

The simplest way to quantify the impact of a SDIS is to just compute the Euclidean distance between an infeasible trial element and its corrected version, $\|z-c(z)\|^2$. Using Eq. (\ref{eq:components}) it follows that the expected value of $\|z-c(z)\|^2$ is $p_mp_v\sum_{i=1}^n(y_i-SDIS(y_i))^2$.  
\begin{equation} \label{eq:components}
   z_i=\left\{\begin{array}{ll}
    x_i & \hbox{with probab. } 1-p_m\\
    y_i & \hbox{with probab. } p_m\\
    \end{array}
    \right. \quad
    c(z_i)=\left\{\begin{array}{ll}
    x_i & \hbox{with probab. } 1-p_m\\
    y_i & \hbox{with probab. } p_m(1-p_v)\\
    SDIS(y_i) & \hbox{ with probab. } p_mp_v\\
    \end{array}
    \right.
\end{equation}
For a given component, the difference $(y_i-SDIS(y_i))^2$ is obviously the smallest in the case of \texttt{sat}. If $0<y_i-b_i\leq (b_i-a_i)/2$ or $0<a_i-y_i\leq (b_i-a_i)/2$ then the difference is smaller for \texttt{mir} than for \texttt{tor}. This always happens if $F\leq 0.5$. In the case of \texttt{uni}, as the corrected value can be anywhere in $[a_i,b_i]$, the difference can take any value in $(0,2(b_i-a_i))$. However, as the violation probability induced by \texttt{sat} is usually higher than in the case of the other SDISs it means that the expected value of the Euclidean distance corresponding to \texttt{sat} is not necessarily the smallest one.  

\subsection{(Dis)similarity between search directions \label{sect:cosine}}

The similarity between the DE search direction, i.e. $d=z-x$ (differences between the infeasible trial element, $z$, and the target element, $x$) and the corrected search direction, $d_{SDIS}=SDIS(z)-x$, can be analysed using the cosine of the angle between $d$ and $d_{SDIS}$. In the case of \texttt{sat}, the components of the corrected search direction, $d_S$, have the same sign as the components of the DE search direction (either $z_i<a_j=c_S(z_i)\leq x_i$ or $z_i>b_i=c_S(z_i)\geq x_i$, where $c_S$ denotes the transformation corresponding to \texttt{sat}). Thus, for \texttt{sat} the cosine similarity is always larger than $0$. The highest value of the similarity is obtained when $SDIS(z)$ is a convex combination between $x$ and $z$, meaning that the search direction is not altered by the SDIS.

The symmetry between the corrected solutions obtained by \texttt{mir}, $c_M$, and \texttt{tor}, $c_T$, i.e. $c_M(z_i)+c_T(z_i)=a_i+b_i$ allows to prove (see Proposition~\ref{prop:cosine} in Appendix~\ref{sect:appendix_proofs}) that when $F\leq 0.5$ (which ensures that $d^Td_M\geq d^Td_T$) and $x$ and $c_M(z)$ are in the same quadrant, i.e. $(c_M(z_i)-(a_i+b_i)/2)( x_i-(a_i+b_i)/2)\geq 0$ for $i=\overline{1,n}$
(which ensures that $\|d_M\|\leq\|d_T\|$) then $\cos(d,d_{M})\geq \cos(d,d_{T})$. When $x$ and $c_M(z)$ are not in the same quadrant then it is not necessary that $\|d_M\|\leq\|d_T\|$, thus the relationship between $\cos(d,d_{M})$ and  $\cos(d,d_{T})$ cannot be inferred so easy.

For the other SDISs it is difficult to prove statements on the search directions in the general case, since the target value, $x$, can be placed anywhere in the feasible domain. However, in the case when only one component is corrected and the Euclidean norm of the uncorrected search direction is not fully determined by the infeasible component, then \texttt{sat}  preserves the search direction better than any other SDIS which generates values inside the open domain, i.e. $(a_i,b_i)$.  More specifically, if $k$ denotes the index of the infeasible component (e.g. $z_k>b_k$) and if $\|d\|^2=\sum_{i=1}^n(z_i-x_i)^2\geq 2(z_k-x_k)(z_k-b_k)$ then $\cos(d,d_{sat})\geq \cos(d,d_{other})$ (see proof of Proposition~\ref{prop:cosineOneComponent} in Appendix~\ref{sect:appendix_proofs}). The above constraint on $\|d\|^2$ might be violated in the case when $C_r$ is small, which leads to very few mutated components (in the extreme case only one component is mutated and this is also infeasible) and $F$ is large, leading to a large deviation of $z_k$ with respect to $x_k$.

\subsection{Influence of the strategy of dealing with infeasible solutions on population diversity \label{sect:diversityTheor}}

The population diversity can be quantified using the variance of the population computed component-wise. In the case of \texttt{DE/rand/1} variant, the expected value of the trial population variance (after applying SDIS, $c(Z)$) depends  on the variance of the current population, $X$, as it is given \citep{Zaharie2017}, in Eq.\ref{eq:variance} :
\begin{equation}\label{eq:variance}
\mathbb{E}[\hbox{var}(c(Z))]= \alpha(p_m,p_v,N,F)\cdot \hbox{var}(X)+\beta(p_m,p_v,N,SDIS).  
\end{equation}
The first coefficient, $\alpha(p_m,p_v,N,F)$ is not influenced by the used SDIS, depending only on the mutation probability ($p_m$), violation probability ($p_v$), population size ($N$) and scale factor ($F$). The SDIS influence is incorporated in the free term which depends on the average and variance of the corrected values, as is specified in Eq.~\ref{eq:varBeta}, where $\overline X$ denotes the midpoint of the current population, $\hbox{mean}(SDIS)$ and  $\hbox{var}(SDIS)$ denotes the average and the variance of the corrected values, respectively.
\begin{eqnarray}\label{eq:varBeta}
 \beta(p_m,p_v,N,SDIS)&\simeq& p_mp_v(1-p_mp_v)\frac{N-1}{N}(\overline X-\hbox{mean}(SDIS))^2 \nonumber \\
  & & +p_mp_v\left(1-\frac{1-p_mp_v}{N}\right)\hbox{var}(SDIS)   
\end{eqnarray}
In the case of \texttt{uni} strategy, $\hbox{mean(\texttt{uni})}=(a+b)/2$ and $\hbox{var(\texttt{uni})}=(b-a)^2/12$. In the case of \texttt{sat} strategy one can consider, in the absence of the selection pressure, that the set of corrected elements follows a Bernoulli distribution with values $a$ and $b$ and probabilities $p_{lb}+p_{ub}=1$ to violate the bounds. Since, in the case of a flat function, there is no incentive to bias the population toward one of the bounds, one can consider that $p_{lb}=p_{ub}=1/2$. In this case one obtains $\hbox{mean(\texttt{sat})}=(a+b)/2$ and $\hbox{var(\texttt{sat})}=(b-a)^2/4$. This suggests that, at least under these simplifying assumptions, the SDIS influence on the population variance is larger in the case of \texttt{sat} than in the case of \texttt{uni} strategy, as the second term of $\beta(p_m,p_v,N,SDIS)$ is larger in the case of \texttt{sat}. 

For \texttt{mir} and \texttt{tor} strategies, the distribution of the population of corrected elements is directly related to the distribution of the trial population ($Z$). In the case of \texttt{mir}, a corrected element, $c_M(z)$ is $2b-z$ with probability $p_{ub}$ and $2a-z$ with probability $p_{lb}$. Since for \texttt{DE/rand/1} the expected mean of the mutant population ($\mathbb{E}[\overline Z$]) is the same as the expected mean of the current population ($\mathbb{E}[\overline X]$), it follows that $\hbox{mean(\texttt{mir})}=(a+b)-\mathbb{E}[\overline X]$. Thus $(\overline{X}-\hbox{mean(\texttt{mir})})^2$ is close to $4\left(\overline{X}-\frac{a+b}{2}\right)^2$, suggesting that the first rhs term in Eqs.\ref{eq:varBeta} is four times larger in the case of \texttt{mir} strategy than in the case of \texttt{sat} and \texttt{uni} strategies. However, if $\overline{X}-\frac{a+b}{2}$ is small, the influence of this term is also small. On the other hand, the variance of the population of elements which violate one of the bounds and are corrected using the \texttt{mir} strategy is around $F^2/10-F/4+1/4$ (under assumptions that $a=0$ and $b=1$ and as long as the current population, $X$, is almost uniformly distributed) which is smaller than $1/4$ (variance of the elements corrected by \texttt{sat}) but larger than $1/12$ (variance of the elements corrected by \texttt{uni}). For details see the proof of Proposition~\ref{prop:diversity} in Appendix~\ref{sect:appendix_proofs}.

Thus, is to be expected that the highest impact on the population diversity is induced by  \texttt{sat} followed by \texttt{mir} and then by \texttt{uni}.

Due to the symmetry between \texttt{mir} and \texttt{tor} (as $c_M(z)+c_T(z)=a+b$) the  population of corrected elements has the same variance, thus it is expected that these two strategies have the same influence on the population diversity.

\section{Experimentation on $f_0$}\label{sect:exponf0}

Understanding the `behaviour' of a heuristic algorithm during the search for (near-) optimal positions is challenging. As the fitness landscape of the problem at hand is the main driving force of this process, decoupling its effect on the internal dynamics of the candidate solutions is key to observing how the algorithm's working logic operates in figuring out promising search directions to follow. The $f_0$ function of Eq.~\ref{eq:f0} was first proposed in \citep{Kononova2015} to serve this purpose:
\begin{equation}\label{eq:f0}
        \bf{f_0}:[0,1]^n\to[0,1],  \textnormal{ where } \forall \bf{x}, \bf{f_0(x)} \sim \mathcal{U}(0,1).
\end{equation}
Because of its truly stochastic nature, $f_0$ separates otherwise highly interconnected effects from the fitness landscape and the location of the optima, thus being suitable for analysing structural implications to the search. Previous studies used $f_0$ for understanding how often solutions are generated outside the search space in classic DE variants \citep{Kononova2020_outside}, relevant observations are reported in Section~\ref{sect:POIS_f0}, and for finding structural biases in DE variants \citep{Caraffini2019lego,Caraffini2019,vStein2021_emergence,bib:DEAnalysis2022} as well as several other algorithmic frameworks for heuristic optimisation \citep{Kononova2020CEC,Kononova2020PPSN,Vermetten2021_anisotropy}, as summarised in Section~\ref{sect:SB}. In this piece of work, we relate to these aspects and further employ $f_0$ for collecting the CS angular distances in the two most known DE configurations and study their distribution.

Where possible, results on function $\bf{f_0}$ discussed in this section are contrasted with theoretical conclusions obtained under a number of simplifying assumptions made in Section~\ref{sect:theoretic}.

\subsection{Setup for the experimentation on $f_0$}\label{sect:f0_exp}

The two well-known \texttt{DE/rand/1/bin} and \texttt{DE/rand/1/exp} algorithms 
are executed on $f_0$ function in $n=30$ dimensions~\footnote{This dimensionality value is consistent to previous studies on $f_0$, thus allowing for a direct comparison.} for a fixed number of $10000\cdot n$ function evaluations. This is done for multiple configurations of the two DE variants, each one characterised by a SDIS from the set $\{\texttt{COTN},\texttt{mir},\texttt{sat},\texttt{tor},\texttt{uni},\texttt{HVB}\}$ (see Section~\ref{sect:SDIS}), a specific population size $N\in\{5,$ $20,$ $100\}$ and a pair of control parameters. 

To practically identify these configurations in our analysis we use the notation \texttt{DE/rand/1/$\star$-p$N$}, where $\star$ can either be the binomial \texttt{bin} or the exponential \texttt{exp} crossover operator. As each configuration is tested over a wide range of $F-C_r$ pairs, these are not included in the notation but more effectively reported in the graphical results. Note that the same set of scale factor values, i.e. $F \in \{0.05,0.285,0.52,0.755,0.99\}$, are considered for \texttt{bin} and \texttt{exp} DE variants - which offers a good discretisation of the space of its typical admissible values. Conversely, two different sets of values are used for $C_r$ depending on the employed crossover operator, as indicated below
\begin{itemize}
        \item $C_r^{bin} \in \{ 0.05,0.285,0.52,0.755,0.99\}\cup\{0.0891,0.1283, 0.1675, 0.2067, 0.2458\}$,
         \item $C_r^{exp} \in \{0.05,0.285,0.52,0.755,0.99\}$,
    \end{itemize}
where the admissible range $(0,1]$ is first discretised equally for the two recombination strategies and a further set is added to better cover the range $(0.05,0.285)$ when the \texttt{bin} operator is employed. The latter additional range is of interest to our analysis as it allows for spotting the nature of changes (smooth or sudden) in the algorithm behaviour, measured here in terms of cosine similarity between search directions (as indicated in Section~\ref{sec:cs}), when only one or very few components are inherited and corrected. It should be remarked that doing this while employing the \texttt{exp} crossover operator would be irrelevant. Indeed, with such a low $C_r$ values this operator would be quite unlikely to exchange any component from the \texttt{mutant} to the \texttt{target} on top of the one that gets necessarily replaced by design -- for clarifications on the behaviour of the \texttt{exp} operator see \citep{Caraffini2019,Kononova2020_outside,bib:DEAnalysis2022}.

In total, the \texttt{DE/rand/1/exp} configurations obtained from combining the $6$ SDIS operators, the $3$ population sizes, the $5$ scale factor values and the $5$ crossover rate values (i.e. $6\cdot3\cdot5\cdot5=450$ configurations), plus those obtained with similar settings but with $5$ more $C_r$ values for \texttt{DE/rand/1/bin} (i.e. $6\cdot3\cdot10\cdot5=900$), results in $1350$ optimisation processes. Each DE configuration is executed $30$ times to produce robust results over $30$ independent runs. Relevant information stored during each run include CS measure values for solutions where SDIS has been applied, percentages of infeasible solutions (see Section~\ref{sect:pois}), population and fitness diversity measures, etc. This experimentation has been performed with the SOS software platform \citep{CaraffiniSOS_2020}, whose source code is available in its GitHub repository, where a permanent release of the current state of the SOS platform is also made available prvided~\footnote{\url{github.com/facaraff/SOS/releases/tag/ECJ-ReproducibilityInEC}}. This is accompanied by detailed clarifications on how to find software classes within the platform and on how to reproduce the entire employed dataset, which we have stored in the zenodo repository \cite{zenodo_tiobr} along with the full sets of static versions of all processing scripts. 

\subsection{Analysis of results on distributions of cosine similarity}\label{sect:CSAnalysis}
\begin{figure}
    \centering
    \includegraphics[width=\textwidth,trim=12mm 9mm 11mm 8mm,clip]{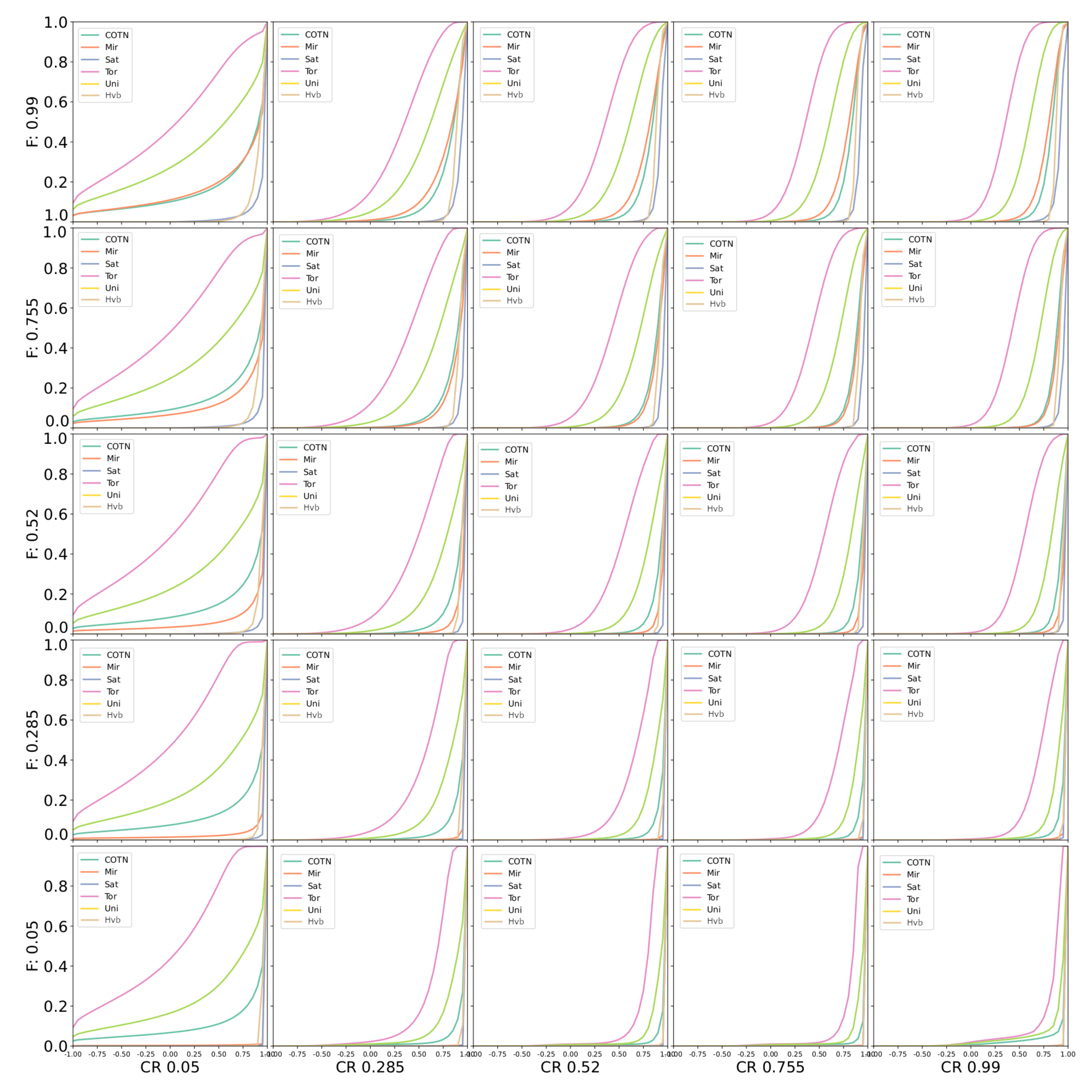}
    \caption{ECDF curves of the cosine similarity values on $f_0$ for \texttt{DE/rand/1/bin-p100} for different SDIS and values of $F$ and $C_r$, $30$ runs each. }\label{fig:ECDF_CS}
\end{figure}

To analyse the distribution of cosine similarity for infeasible solutions when using different SDIS, we can make use of the Empirical Cumulative Distribution Function (ECDF), which shows for each value of $x$ what fraction of infeasibility corrections had a CS of at most $x$. 

In Figure~\ref{fig:ECDF_CS}, we show the ECDF for each SDIS, based on the values $F$ and $C_r$, for \textit{all infeasible} solutions generated during full runs of \texttt{DE/rand/1/bin-p100} on $f_0$. In this figure, SDIS with a larger disruptiveness will have a curve which is closer to the upper left corner of the plot, so larger areas under the ECDF correspond to larger values of the cosine between unconstrained and corrected search directions, i.e. larger amounts of disruptiveness. When comparing two strategies, in the case when $ECDF(SDIS_1)\geq ECDF(SDIS_2)$ then it can be proved (see Proposition~\ref{prop:ECDF} in Appendix~\ref{sect:appendix_proofs}) that it is more likely that the cosine values corresponding to $SDIS_1$ are smaller than those corresponding to $SDIS_2$.

We see from Figure~\ref{fig:ECDF_CS} that even though the shape of the curves can change significantly based on the used parameter settings, the ordering remains consistent: \texttt{tor} is the most disruptive, followed by \texttt{uni}, \texttt{COTN} and \texttt{mir}, while \texttt{HVB} and \texttt{sat} are the least disruptive. This ordering is also preserved when changing population size or crossover type for \texttt{DE/rand/1} -- those figures are included in the Figshare repository of this paper~\citep{figshare_tiobr}. 

\textbf{Relation to theoretical results}. The experimental analysis in Figure~\ref{fig:ECDF_CS} leads to remarks which are \textit{in line} with the results derived in Section~\ref{sect:cosine}: (i) \texttt{sat} is more likely to lead to the highest value of the cosine similarity (provable in the case when only one component is corrected - this might be related with small values of $F$ and $C_r$); (ii) in the case when $F\leq 0.5$ it is more likely that \texttt{tor} induces a higher disruption than \texttt{mir} (provable when the element obtained by mirroring and the target element are in the same quadrant); (iii) the cosine similarity between the initial search direction and that corresponding to \texttt{sat} is always positive and higher than that corresponding to \texttt{HVB} \citep{Mitran2021}.

\subsection{Implications of direction-disruptive SDIS operators on population diversity}\label{sect:diversity}
To measure further the impact of SDIS, we track the \textit{population diversity} on $f_0$ in each generation and show the results in Figure~\ref{fig:diversity_matrix} for multiple runs. The population diversity is defined here as the average of the standard deviations of solutions in the population within the domain in each of the $30$ dimensions.
\begin{figure}
    \centering
    \includegraphics[width=\textwidth,trim=7.5mm 9.5mm 7.5mm 7.5mm,clip]{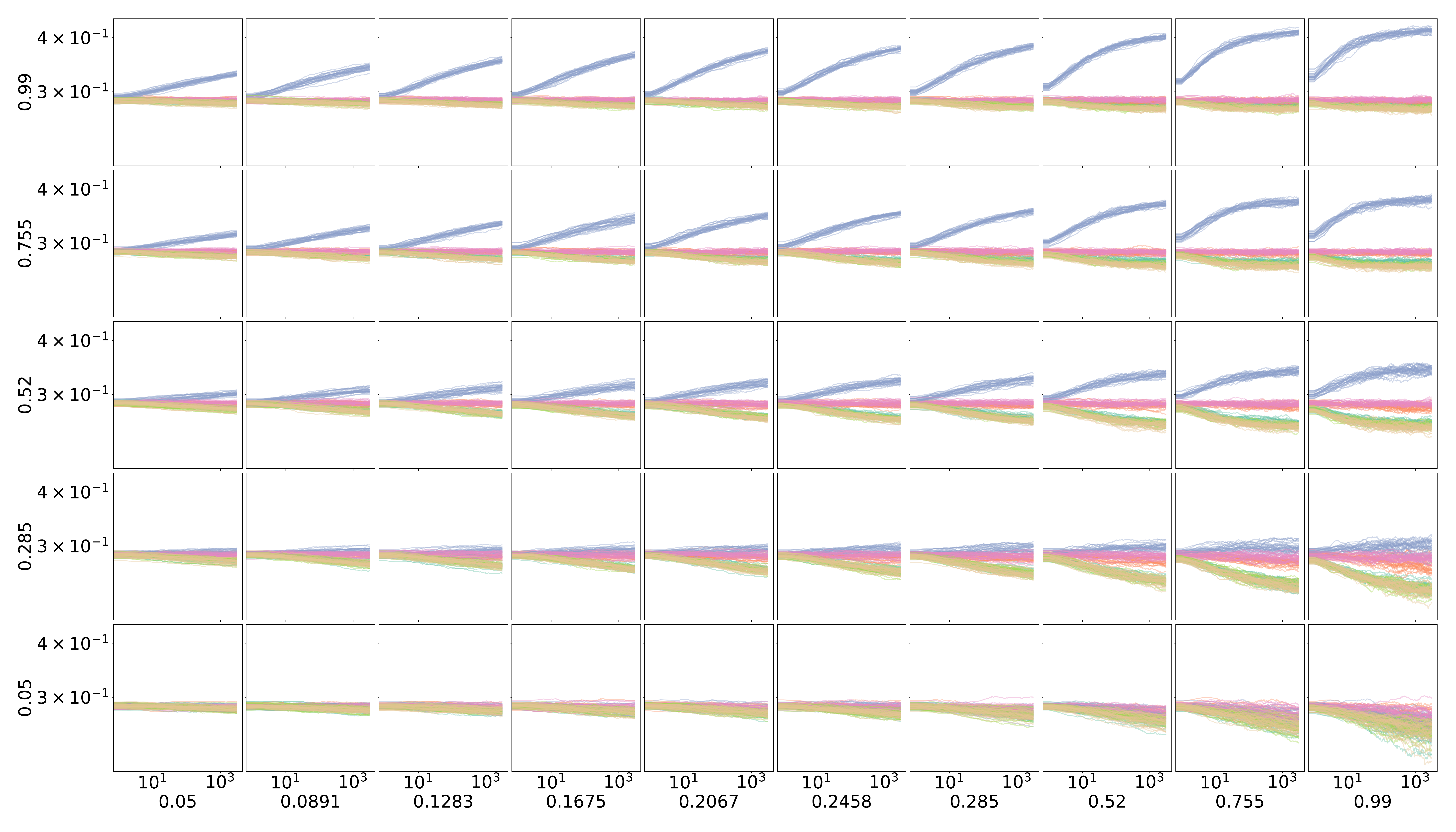}
    \caption{Evolution of population diversity per generation on $f_0$ for $30$ runs of \texttt{DE/rand/1/bin-p100} for different SDIS, values of $F$ (subfigures in vertical direction), $C_r^{bin}$ (subfigures in horizontal direction). Colours correspond to those in Figure~\ref{fig:ECDF_CS}.}\label{fig:diversity_matrix}
\end{figure}

Figure~\ref{fig:diversity_matrix} demonstrates a rather consistent picture for \texttt{DE/rand/1/bin-p100} on $f_0$ for all values of $F$ excluding the smallest and for all values of $C_r$: (i) diversity of the \texttt{sat} variant shown in blue is consistently the highest and increasing during the runs; (ii) the second highest values of diversity are attained consistently by the \texttt{tor} and \texttt{mir} variants shown in pink and orange, respectively; these two variants are near-constant during the runs and indistinguishable in terms of diversity; (iii) the third group of variants, also indistinguishable in terms of diversity, is demonstrated by \texttt{COTN}, \texttt{unif-resample} and \texttt{HVB}; diversity of these variants is consistently decreasing during the runs. Meanwhile, for the smallest value of $F=0.05$, diversity of all variants is largely indistinguishable and decreasing over time. In general, increase and decrease of diversity during the run depend on both DE control parameters, however increase of diversity in \texttt{sat} variant is more drastic than the decrease of variants from the third group discussed above. Meanwhile the standard deviation appears to depend on $F$, with smaller $F$ leading to higher standard deviations of diversity values. 

Furthermore, by comparing the results on disruptiveness (Section~\ref{sect:CSAnalysis}) with those related to diversity, it turns out that a more disruptive strategy \textit{does not} necessarily induce an increase in the population diversity.

\textbf{Relation to theoretical results}. The theoretical results on population diversity  presented in Section~\ref{sect:diversityTheor} (for \texttt{sat}, \texttt{uni} and \texttt{mir}) state that the variance of the population of {\it corrected components} is the largest in the case of \texttt{sat} and smallest in the case \texttt{uni}, being equal to $1/4$ and $1/12$, respectively. The variance of the population of components corrected by \texttt{mir} decreases with $F$ from the value corresponding to \texttt{sat} toward that corresponding to \texttt{uni} (Fig.~\ref{fig:VarianceCorrection}, Appendix~\ref{sect:appendix_proofs}).
On the other hand, as it is also stated in  Section~\ref{sect:diversityTheor}, the symmetry between \texttt{mir} and \texttt{tor} leads to the same variance of the po\-pu\-la\-tion of components corrected by these two SDISs. Thus, as it is illustrated in Fig.~\ref{fig:TheoreticalDiversity} (Appendix~\ref{sect:appendix_proofs}) the experimental results are \textit{consistent} with what is to be expected from the theoretical point of view.

The \textit{difference in the shape} of standard deviation curves as they are illustrated in Figures~\ref{fig:diversity_matrix} and~\ref{fig:TheoreticalDiversity} can be explained by the fact that the theoretical curves have been estimated in the case of a flat function, i.e. all new elements are incorporated into the population after correction, while in the case of $f_0$ acceptance is based on a random decision, thus the change in the populations corresponding to two consecutive generations is expected to be smaller than in the case of a flat function. This might explain the more gradual change of the diversity measure in the case of empirical analysis (Fig.~\ref{fig:diversity_matrix}) than in the case of theoretical estimations (Fig.~\ref{fig:TheoreticalDiversity}) where the limit values are rather quickly reached.

\subsection{Relation to the analysis of boundary violation probabilities}\label{sect:POIS_f0}

To collect information on the influence of SDIS on the amount of infeasible solutions, an experimental analysis has been conducted on $f_0$ using a population of $N=100$ elements and collecting the number of infeasible components during the first $100$ generations of DE/rand/1/*, in the case of the $5$ values of $F$, as used in the previous experiments. The bound violation probability has been estimated as the averaged frequency of infeasible components and the estimations for five SDISs are illustrated in  Figure~\ref{fig:ViolProbEmpirical}. The main remark is that \texttt{sat} generates the highest number of infeasible components followed by \texttt{uni}, while \texttt{mir} and \texttt{tor} lead to almost the same bound violation probability. The smallest amount of infeasible components seems to be generated by \texttt{COTN}.

\begin{figure}[!tb]
    \centering
    \includegraphics[width=0.75\textwidth]{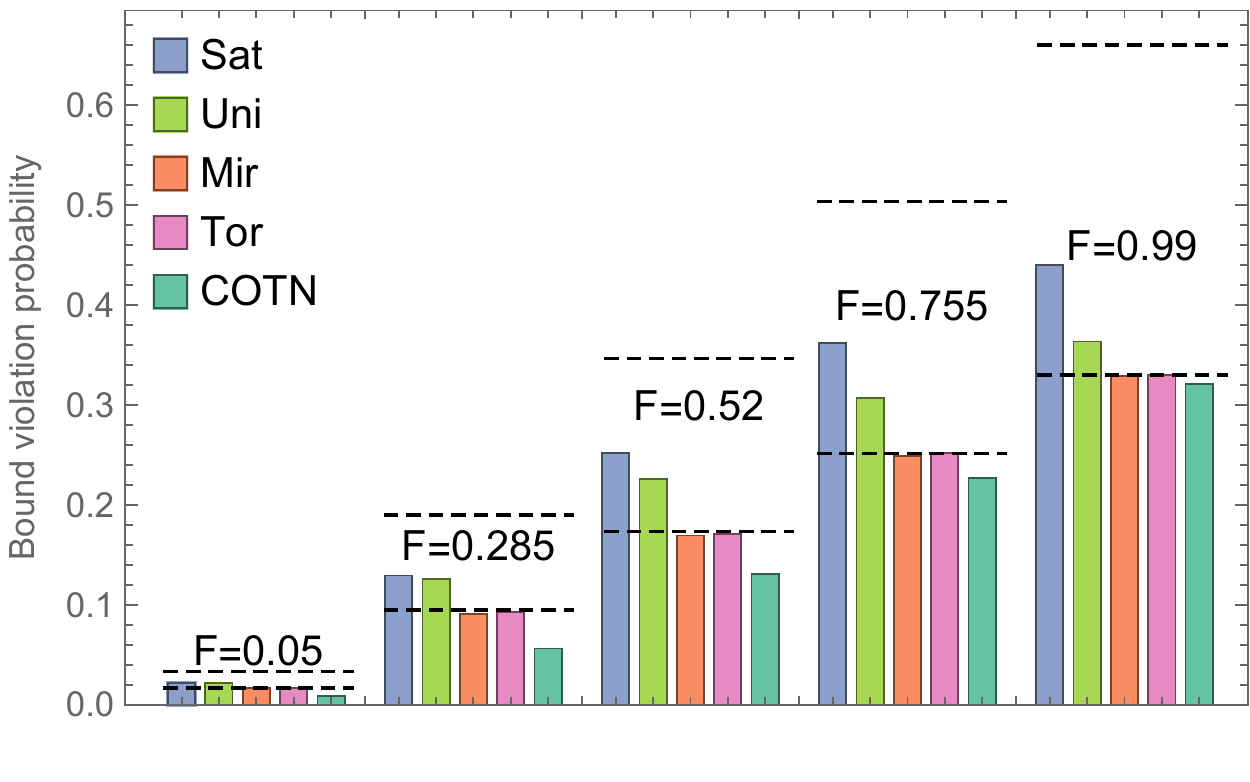}
    \caption{Empirical probability of bounds violation (averages of the relative frequencies of infeasible components) in \texttt{DE/rand/1/*-p100} on $f_0$ of dimensionality $n=30$, estimated over $100$ generations and $30$ independent runs. The dashed lines correspond to the theoretical lower and upper bounds ($F/3$ and $2F/3$, respectively).}\label{fig:ViolProbEmpirical}
\end{figure}

\textbf{Relation to theoretical results}.

In the absence of a selection induced bias, it is expected that the probability of generating, by \texttt{DE/rand/1} mutation, components which are outside the bounds is close to that estimated theoretically, i.e. $p_v=F/3$. However, the inclusion of corrected components in the population, during the evolutionary ge\-ne\-ra\-tions, can lead to changes of the bounds violation probability, as it has been inferred in Proposition~\ref{prop:ViolProb} for \texttt{sat}. Except for \texttt{COTN} strategy, the bound violation probability is between the theoretically estimated limits, i.e. $F/3$ and $2F/3$.

\subsection{Relation to the analysis of structural bias}\label{sect:SB}
\begin{figure}
    \centering
    \subfigure[\texttt{COTN}]{\includegraphics[width=0.1945\textwidth,trim=12mm 10mm 40mm 58mm,clip]{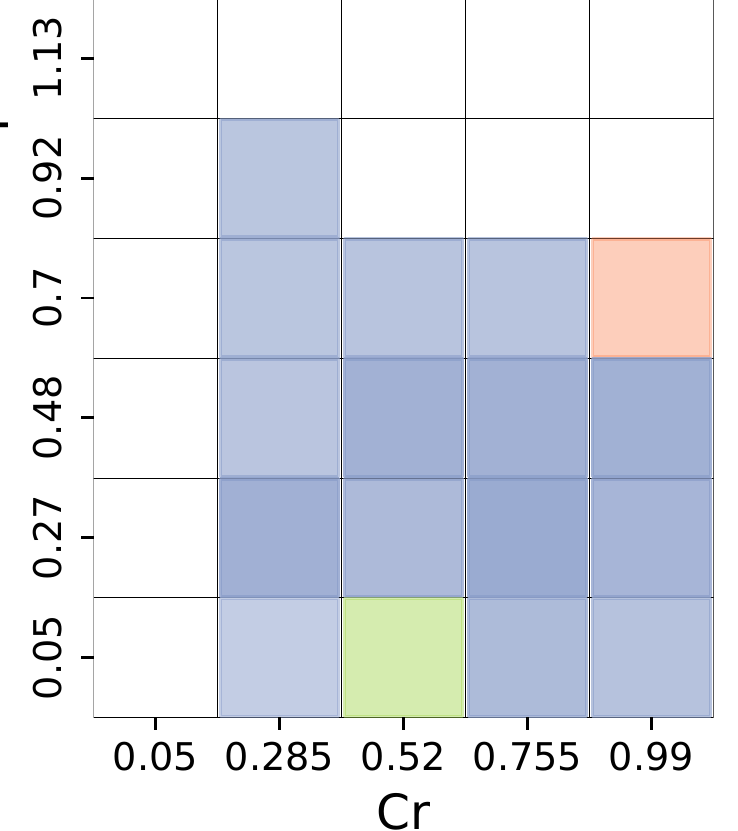}}
    \subfigure[\texttt{mir}]{\includegraphics[width=0.1945\textwidth,trim=12mm 10mm 40mm 58mm,clip]{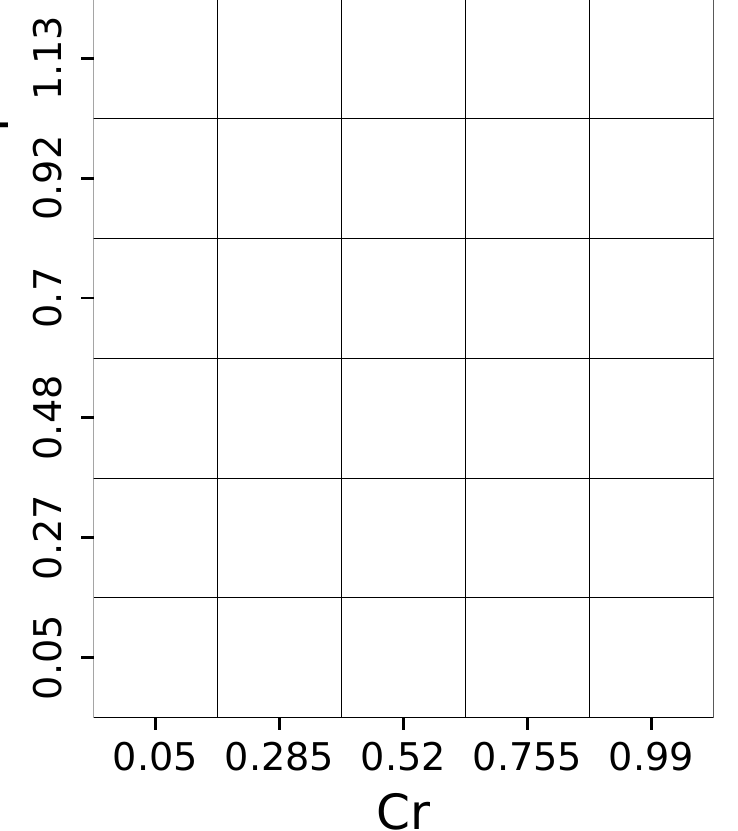}}
    \subfigure[\texttt{sat}]{\includegraphics[width=0.1945\textwidth,trim=12mm 10mm 40mm 58mm,clip]{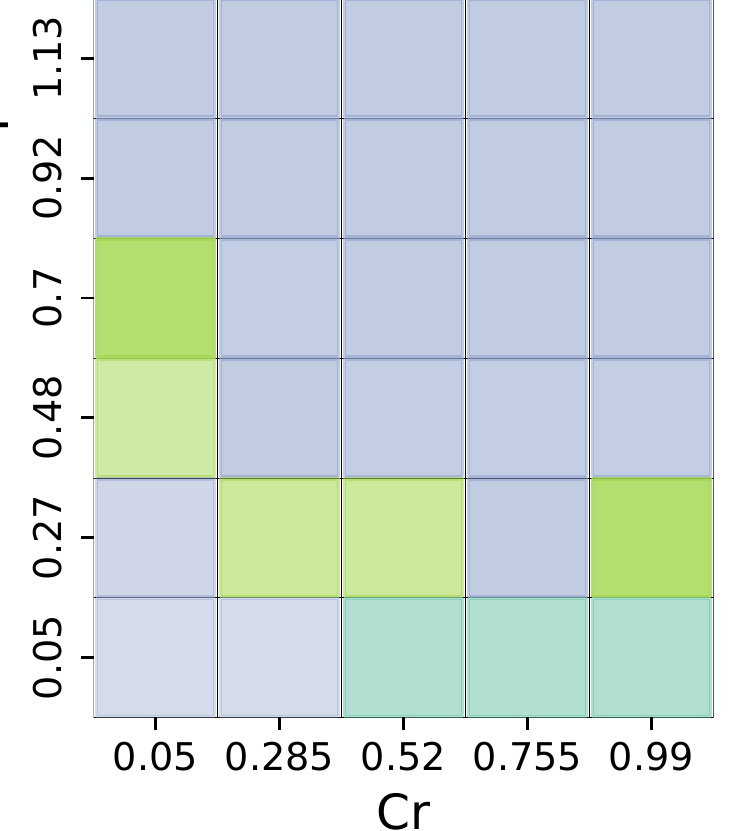}}
    \subfigure[\texttt{tor}]{\includegraphics[width=0.1945\textwidth,trim=12mm 10mm 40mm 58mm,clip]{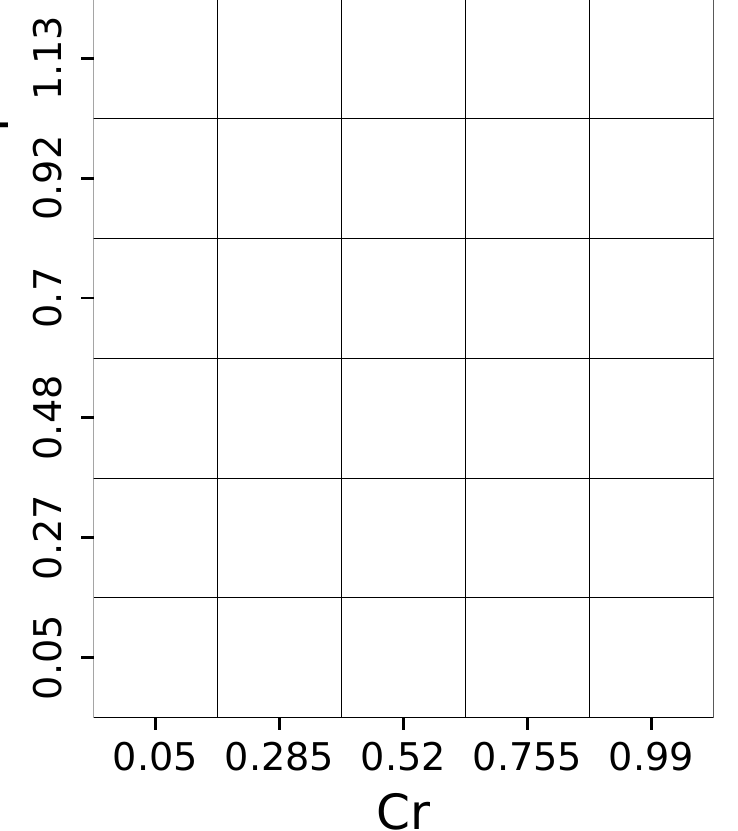}}
    \subfigure[\texttt{uni}]{\includegraphics[width=0.1956\textwidth,trim=12mm 10mm 40mm 58mm,clip]{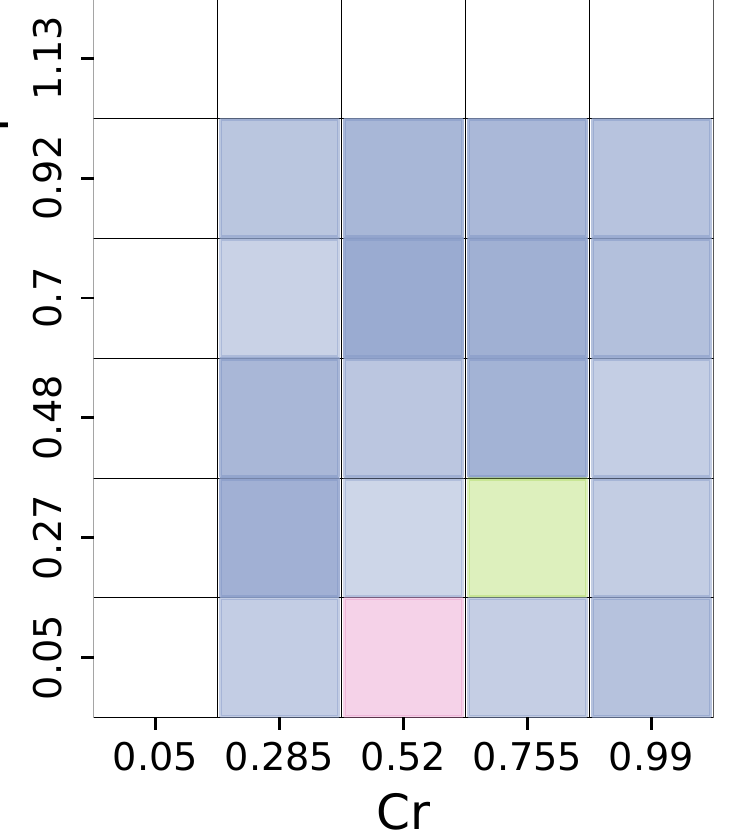}}
    \subfigure{\includegraphics[angle=90,width=0.5\textwidth,trim=90mm 10mm 15mm 9mm,clip]{IMG/Heatmap_r_o_b_u_p100.pdf}}
    \caption{Predictions on the kind of structural bias (with confidences as colour intensities) produced by the random forest model from the \texttt{BIAS} toolbox \citep{techrxiv_bias} for \texttt{DE/rand/1/bin-p100} with 5 SDIS variants, $F\in[0.05,1.13]$ and $C_r\in[0.05,0.99]$. White squares correspond to configurations with no structural bias detected. Figure is taken from \citep{bib:DEAnalysis2022}, which has not considered the \texttt{HVB} SDIS.}
    \label{fig:SB_toolbox}
\end{figure}  

Test function $f_0$ used in this study has been originally introduced~\citep{Kononova2015} for studying the so-called structural bias (SB) of a heuristic optimisation algorithm which is an inability of an algorithm to explore different parts of the search domain to equal extent regardless of the objective function. Such study requires decoup\-ling the effects of the landscape of the objective function from that of SB. It is precisely the random nature of $f_0$ and thus, known distribution of locations of it's optima in a series of independent runs, that allows identification of SB: an algorithm is said to suffer of SB if locations of the final best points found in a series (of a reasonable size) of independent runs minimising $f_0$ produced within a realistic budget of fitness evaluations deviates from uniform~\citep{Kononova2015}.

As the nature of SB appears to originate from the iterative application of a limited set of algorithm's operators, its identification is not straightforward~\citep{Kononova2020CEC,Kononova2020PPSN,techrxiv_bias}. Among a number of algorithms investigated in li\-te\-ra\-tu\-re so far, results on SB in DE show clear patterns in time~\citep{vStein2021_emergence}, dimensionality~\citep{Vermetten2021_anisotropy} and parameter space~\citep{bib:DEAnalysis2022}. Referring to the latter, Figure~\ref{fig:SB_toolbox} shows such results on the presence and type of SB identified for \texttt{DE/rand/1/bin-p100} for various values of $F$ and $C_r$ with $5$ variants of SDIS considered in this paper. These results support the picture in Figure~\ref{fig:diversity_matrix}: (i) \texttt{mir} and \texttt{tor} indeed stay constant in terms of diversity, which is what we would expect from an unbiased algorithm; (ii) the `most' biased variant (\texttt{sat}) also has the largest difference in terms of diversity, while and \texttt{uni} and \texttt{COTN} are somewhat in-between.

\textbf{Relation to theoretical results}.
The absence of structural bias in the case of \texttt{mir} and \texttt{tor} suggests that the assumption of uniformly distributed populations, used in the theoretical analysis, is not too restrictive, at least for these strategies. This aspect is reflected by the theoretically estimated value for the bound violation probability ($p_v=F/3$) which is very close to the empirically estimated value (Figure~\ref{fig:ViolProbEmpirical}) in the case of \texttt{mir} and \texttt{tor} strategies. On the other hand, in the case of \texttt{sat}, \texttt{uni} and \texttt{COTN}, the presence of structural bias induces a deviation from the uniformity assumption and consequently, the empirical violation probabilities are not so close to the theoretical values obtained based on the uniformity assumption.

\section{Experimentation on the BBOB suite}\label{sect:BBOBexperiments}

In order to assess whether the observed differences in SDIS have any impact on the performance of DE outside of $f_0$, we run a benchmarking study on COCO bench\-mar\-king suite~\citep{hansen2021coco}. We make use of the single-objective, noiseless version of BBOB function set, which contains 24 distinct functions~\citep{finck2010real}, each of which can be instantiated with different transformations, referred to as instances of these functions. The BBOB-functions are all defined using box-constraints of $[-5,5]^n$, where $n$ is the dimension of the problem, and are accessed here using IOHexperimenter~\citep{iohexperimenter_arxiv}.

Where possible, results on BBOB functions discussed in this section are contrasted with theoretical conclusions obtained under a number of simplifying assumptions made in Section~\ref{sect:theoretic}.

\subsection{Setup for the experimentation on the BBOB suite}\label{sect:BBOB_algs}
To run Differential Evolution algorithms on the BBOB-functions, we make use of the \texttt{pyade} package~\footnote{\url{github.com/xKuZz/pyade}}, which is a python-based implementation used in the field~\citep{pyade_usecase1, pyade_usecase2} that incorporates several variants of DE, including \texttt{SHADE} and \texttt{L-SHADE} employed for this study -- see Section~\ref{sect:DEs} for their description. 
We made some minor modifications to the base-code of \texttt{pyade}:
\begin{itemize}
    \item The initialisation of the first population was changed from using the normal distribution to uniform within the box-constrains to ensure a better coverage of the search space and avoid generating infeasible solutions in the initial population. This modification reverts the \texttt{pyade} implementation to the original DE specification \citep{storn1997differential}.
    \item The SDIS methods described in Section~\ref{sect:SDIS} were implemented, with the exception of \texttt{HVB}~\footnote{As a result of the overall design structure of \texttt{pyade}, the SDIS does not have access to the information about the point which generated the uncorrected trial solution. As such, \texttt{HVB} was not implemented here.}.
    \item The overall code structure was changed to allow for easier logging of infeasible solution generation and population diversity during the algorithm's run.
    \item Parameter settings for \texttt{SHADE} and \texttt{L-SHADE} are the defaults as set in \texttt{pyade}.
\end{itemize}

The code for both this modified version of \texttt{pyade} and the complete benchmarking setup on BBOB can be found on GitHub~\footnote{\url{github.com/Dvermetten/DE_TIOBR}}. In addition, the full reproduciblity steps can be found on the same GitHub page. For BBOB-based experiments described in this section, we make use of both the 5- and 30-dimensional version of the problems, using instances 1-5, and perform 5 runs per instance. We make use of a total budget of $10000\cdot n$ fitness evaluations.

\subsection{Analysis of performance benchmarking}\label{sect:BBOB_perf}
\begin{figure}
    \centering
    \includegraphics[width=\textwidth,trim=10mm 73mm 6mm 3mm,clip]{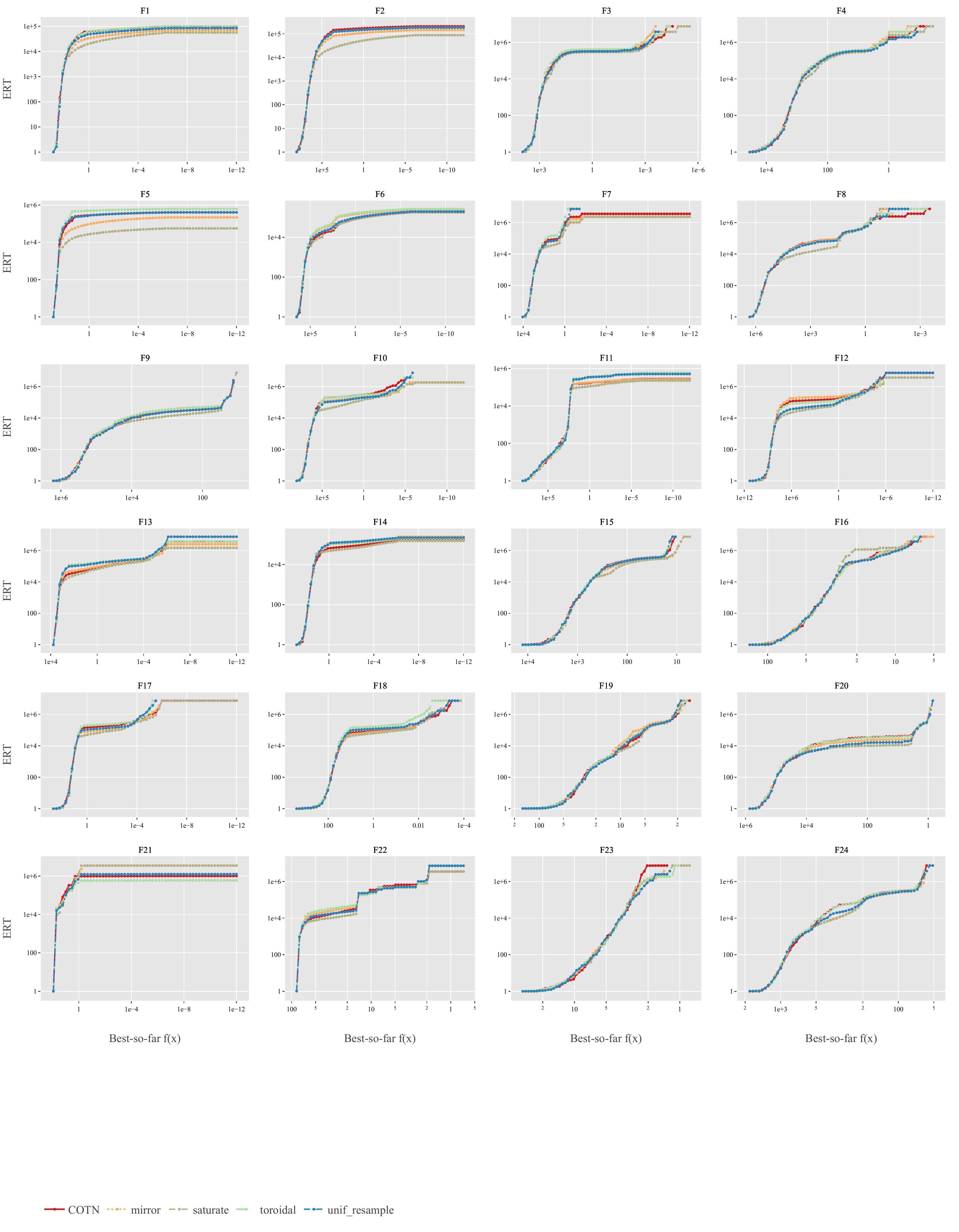}
    \caption{Overview of the ERT for \texttt{L-SHADE} with 5 SDIS methods on the BBOB-functions in 30D, 5 instances 5 runs each. Colours correspond to those in Figure~\ref{fig:ecdf}. }
    \label{fig:30D_overview}
\end{figure}

\begin{figure}
    \centering
    \includegraphics[width=\textwidth]{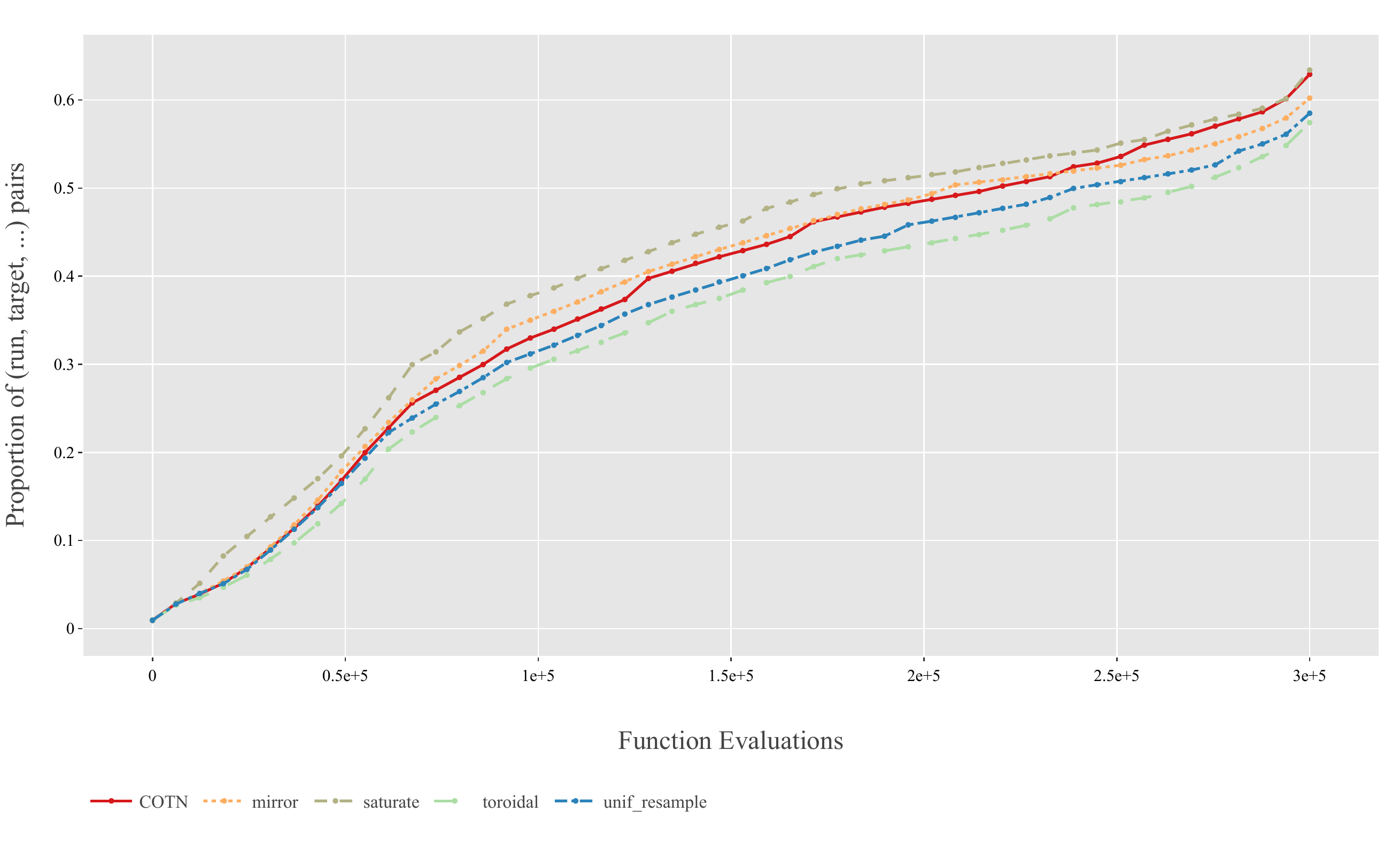}
    \caption{ECDF curves of \texttt{L-SHADE} with 5 different SDIS, aggregated across all 30-dimensional BBOB functions with 5 instances and 5 runs each, using 51 logarithmically spaced fitness targets between $10^2$ and $10^{-8}$.}\label{fig:ecdf}
\end{figure}

The most common way of analyzing the performance of an iterative optimisation heuristic is by considering the \emph{Expected Running Time (ERT)}, defined as follows.

Given an algorithm $A$, function instance $F$ and target $\phi$ and run-number $i$, we define the hitting time $t_{i}(A, F, \phi)$ to be the number of function evaluations that the algorithm used in its $i^{th}$ run before a solution of quality of at least $\phi$ was evaluated.
Based on this, we can then define the Expected running time over a set of $I$ runs and $J$ problem instances as follows:
$$\text{ERT}(v, \mathcal{F}, \phi) = \frac{\sum_{i=1}^I\sum_{j=1}^J \min\{ t_{i}(v, f^{(j)}, \phi), B\}}{\sum_{i=1}^I\sum_{j=1}^J \mathds{1}\{t_{i}(v, f^{(j)}, \phi) <\infty\}}$$
\noindent where $B$ denotes the computational budget with respect to the number of function evaluations.

In Figure~\ref{fig:30D_overview}~\footnote{Showing the fraction of runs shown in the horizontal axis which reach the specified targets on their respective function/instance within the number of function evaluations shown in the vertical axis.}, we show the ERT achieved by \texttt{L-SHADE} algorithm (see Section~\ref{sect:DEs}) with 5 different SDIS variants on the 30-dimensional BBOB functions, per function. From this figure, we can clearly see that for a lot of functions, the effect of SDIS on overall performance is relatively minor. However, there are some outliers, in particular $f_5$ (Linear slope). Since for $f_5$ the optimum is located on the bound in each coordinate, the SDIS method will have a significant impact on the ability of DE to converge. When zooming in on this function, we clearly see that \texttt{sat} outperforms all other SDIS, followed by \texttt{mir}, and then the other SDISs. This performance ordering is also present for some other unimodal functions such as $f_1$ (Sphere) and $f_2$ (Ellipsoidal). 

When aggregating the performance across all functions, which we can do using the ECDF as seen in Figure~\ref{fig:ecdf}, we indeed observe clear differences in overall performance between the different \texttt{L-SHADE} versions. While the ECDF is obviously impacted by the linear slope function, even when removing this from consideration, the ordering of SDIS-variants remains the same, and matches the order of least disruptiveness on the search as discussed in Section~\ref{sect:f0_exp}. In addition to running \texttt{L-SHADE} on the 30-dimensional BBOB-functions, we have also collected data on the 5D version, for both \texttt{L-SHADE} and \texttt{SHADE}. All figures are made available on figshare~\citep{figshare_tiobr}, while the extended dataset is available directly on the IOHanalyzer GUI~\footnote{Available as  dataset source `TIOBR\_DE' on \url{iohanalyzer.liacs.nl}}.

Overall, we can see that the choice of SDIS indeed has an impact on the final performance of the algorithm, and although this impact is not present on all functions, it is a clear indication that SDIS should be considered a part of the  specification of the algorithm. Furthermore, the fact that the differences are not equally present on all functions indicates that methods for per-instance algorithm configuration could benefit from including the SDIS in their search space.

\subsection{Cosine similarity distributions on BBOB}
\begin{figure} 
    \centering
    \includegraphics[width=0.95\textwidth,trim=0mm 0mm 0mm 0mm,clip]{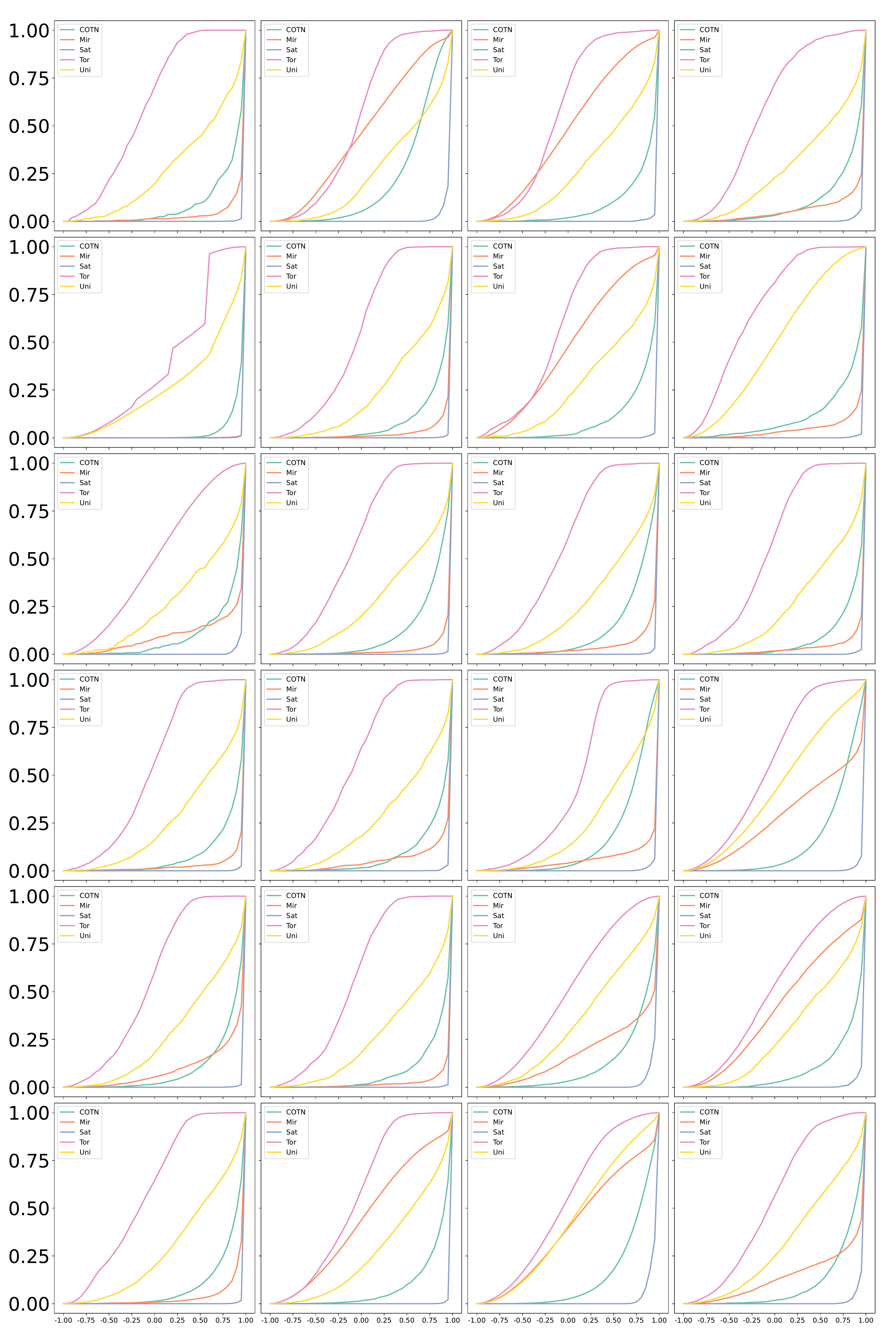}
    \caption{ECDF curves of cosine similarity values of corrected infeasible solutions generated during 5 independent full-budget runs of \texttt{L-SHADE} on 24 BBOB functions for 5 different SDIS variants (all instance 1, one subfigure per function, left to right top to bottom).} \label{fig:BBOB_CS}
\end{figure}

The distributions of the cosine similarity values corresponding to five SDIS variants (\texttt{COTN}, \texttt{mir}, \texttt{sat}, \texttt{tor} and \texttt{uni}) have been generated for 24 BBOB functions both when L-SHADE is used (Figure~\ref{fig:BBOB_CS}) and when SHADE is used - omitted for space limitation, see extensive set of graphical results in \citep{figshare_tiobr}. 

The patterns of behaviour are only slightly different between the two methods. When comparing the ECDFs illustrated in Figure~\ref{fig:BBOB_CS} with those obtained by applying \texttt{DE/rand/1/$\star$} on $f_0$ (Figure~\ref{fig:ECDF_CS}), we can see that the curves have shifted as a result of the different objective functions landscapes, but the global ordering is preserved almost everywhere with \texttt{sat} SDIS characterised by the largest CS values and \texttt{tor} by the smallest ones. 

The presence of different patterns for different functions, support the idea that SDIS should indeed be considered a separate algorithmic component.

\textbf{Relation to theoretical results}. When compared with theoretical insights, results obtained for BBOB confirm the fact \texttt{sat} is more likely to preserve the search direction than the other SDISs. When comparing \texttt{mir} and \texttt{tor}, even if it has been proved only under some particular assumptions that \texttt{mir} preserves more of the search direction than \texttt{tor}, the experiments show that this might be true in most of BBOB functions. 

\subsection{Analysis of final percentage of infeasible solutions}\label{sect:pois}
\begin{figure}
    \centering
    \subfigure[5D]{\includegraphics[height=0.23\textwidth,trim=11mm 13mm 97mm 12mm,clip]{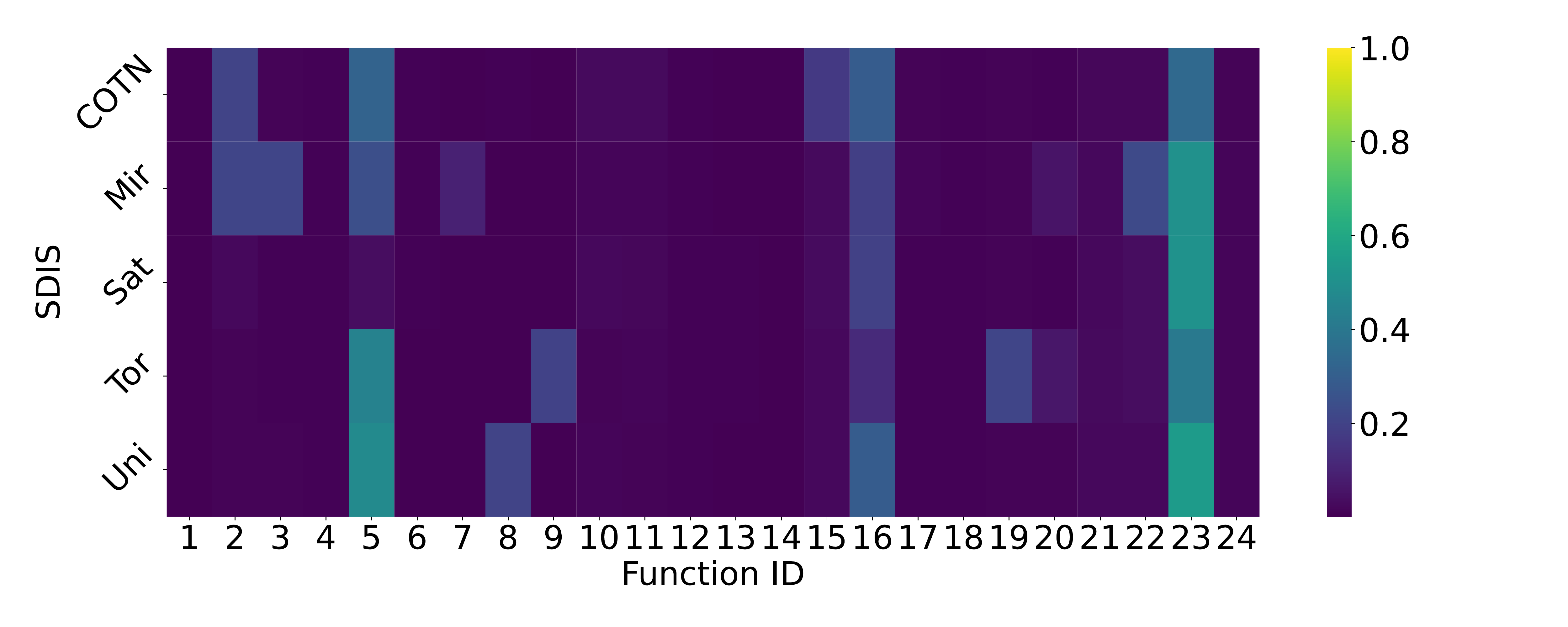}\label{fig:final_pois_5d_lshade}}
    \subfigure[30D]{\includegraphics[height=0.23\textwidth,trim=82mm 13mm 97mm 12mm,clip]{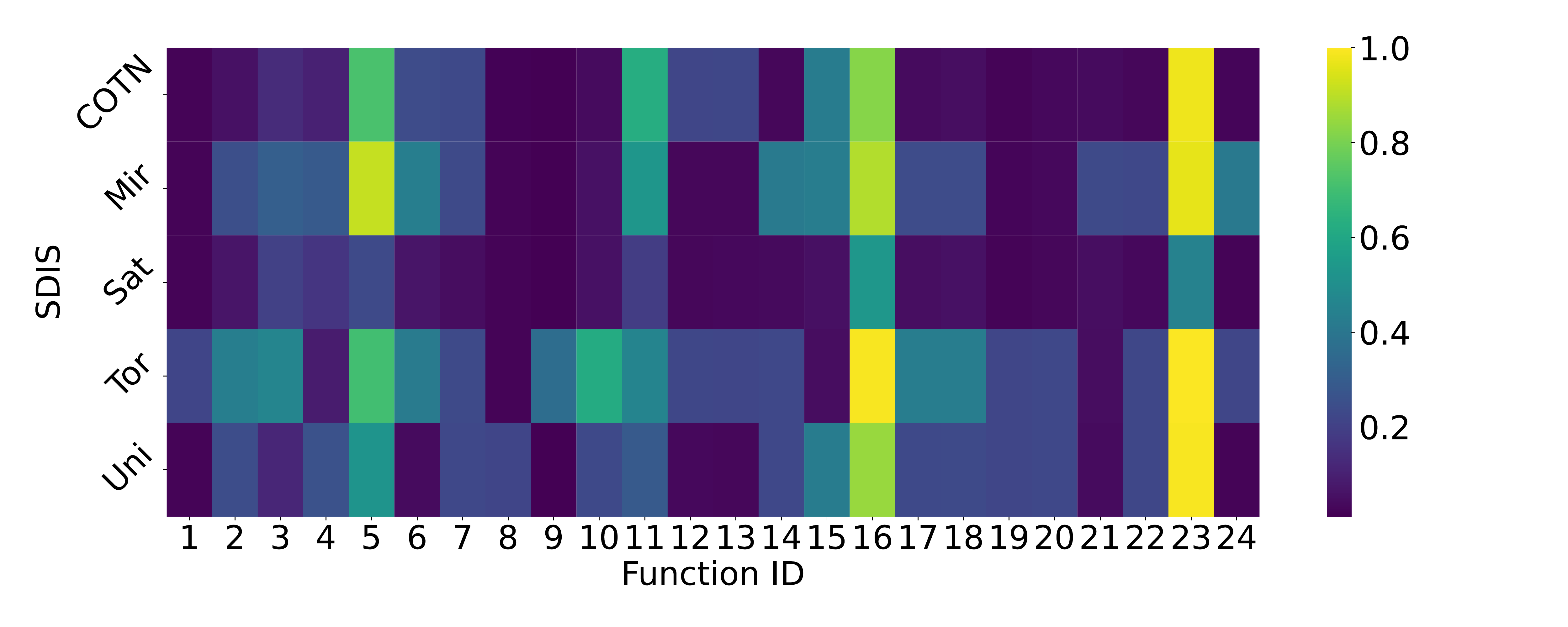}\label{fig:final_pois_30d_lshade}}
    \subfigure{\includegraphics[height=0.23\textwidth,trim=430mm 13mm 47mm 12mm,clip]{IMG/Final_POIS_30D_lshade_I1.pdf}\label{fig:final_pois_30d_lshade}}
    \caption{Final percentages of infeasible solutions generated in full-budget runs of \texttt{L-SHADE} with different SDIS variants on 24 functions of the BBOB suite (instance 1 only) averaged over 5 runs per variant per function. Vertical labels apply to both plots.}\label{fig:final_pois_lshade}
\end{figure}

In addition to looking at the performance of DE with different SDIS, we can also zoom in more on the related behaviour of the algorithm. Since SDIS are only activated when solutions outside of the bounds are generated, we can consider the Percentage Of Infeasible Solutions (POIS) generated throughout the optimisation process. Values of final POIS (i.e. number of infeasible solutions generated by the end of the run to the total fitness evaluation budget) from the L-SHADE algorithm produced on BBOB functions is shown in Figure~\ref{fig:final_pois_lshade}, from which we can clearly see that for the higher dimensionality ($n=30$), the fraction of infeasible solutions generated during the search is significantly higher than for lower dimensionality ($n=5$). This can be explained by considering we count a solution to be infeasible when one or more components are outside their boundaries, which is more likely to occur when there are more components in a solution (see also Section~\ref{sect:infeasibility}). 

Next to the obvious differences between dimensionalities, we note the large differences between the individual functions. Particularly of the note are that some functions seem to end up with a POIS of close to $1$ (e.g. $f_{23}$ in 30D), indicating that for the whole duration of optimisation runs the algorithm rarely manages to create a point \textit{inside} the domain when relying purely on the `standard' operators of DE. \textit{This highlights the importance of SDIS}, as in these cases the method of correcting infeasible solutions will by necessity influence the position of almost all points used during the optimisation process.

Finally, in Figure~\ref{fig:final_pois_lshade}, we note that there are some differences in POIS between SDIS: \texttt{sat} variant shows a clearly lower amount of infeasible points generated, which could be an indication that it has a \textit{lowest disruptiveness} as discussed in Section~\ref{sect:f0_exp}. 
This remark might look to be in contradiction with the theoretical results on bound violation probability (Section~\ref{sect:ViolProb}) and those experimentally obtained for $f_0$ (Section~\ref{sect:POIS_f0}). A possible explanation is that most of the infeasible elements are generated based on those already on the bounds and in the absence of selection pressure, the set of elements with components on the boundary is self-sustained. In the presence of selection pressure, if the optima are not on the bounds then the population elements are driven away from the boundary leading to a decrease in the POIS. On the other hand, if the optima have components on the boundary the search process is stopped as soon as the boundary is reached. 

In addition, the fact that several functions indicate a relatively large POIS for almost all SDIS variants indicates that these might be more sensitive to the selection of this operator, since it will be applied almost every time a candidate solution is generated.

\subsection{Analysis of windowed percentage of infeasible solutions and population diversity}
\begin{figure}
    \centering
    \subfigure[Windowed POIS, \texttt{SHADE}, 5D ]{\includegraphics[width=0.495\textwidth,trim=7mm 15mm 7mm 7mm,clip]{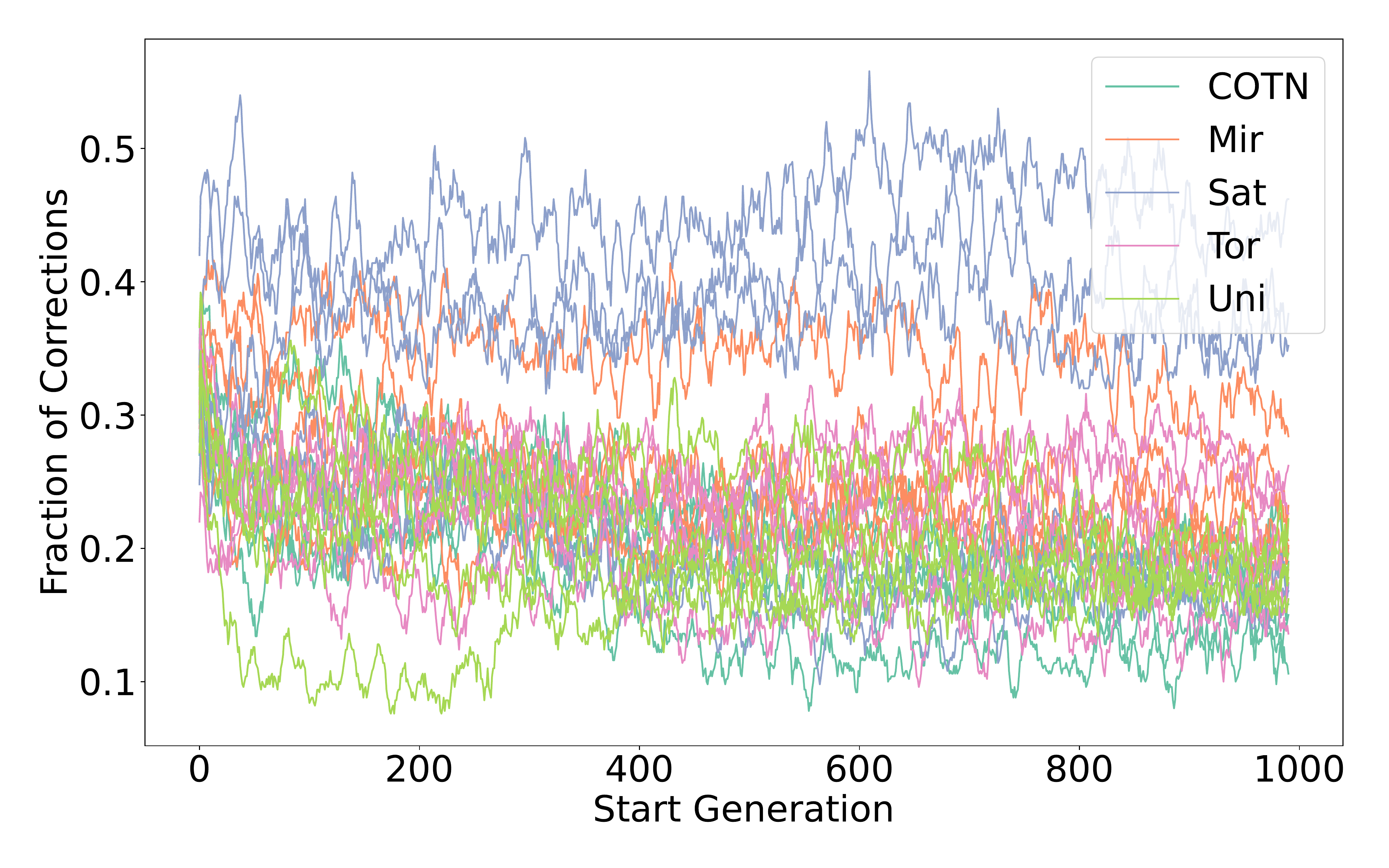}\label{fig:F23_D5_Shade_POIS}}
    \subfigure[Population diversity, \texttt{SHADE}, 5D]{\includegraphics[width=0.495\textwidth,trim=7mm 15mm 7mm 7mm,clip]{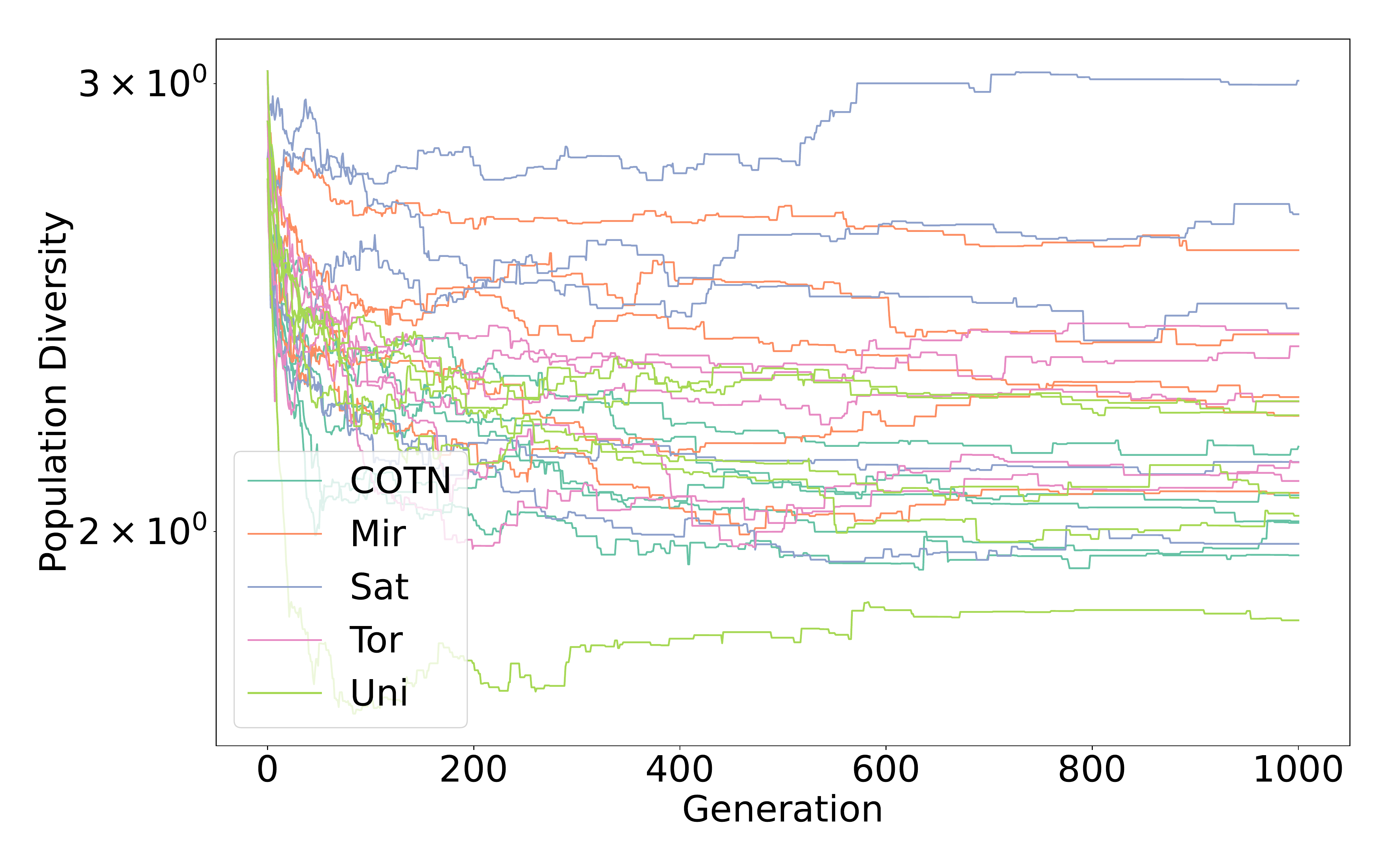}\label{fig:F23_D5_Shade_div}}
    \subfigure[Windowed POIS, \texttt{L-SHADE}, 5D ]{\includegraphics[width=0.495\textwidth,trim=7mm 15mm 7mm 7mm,clip]{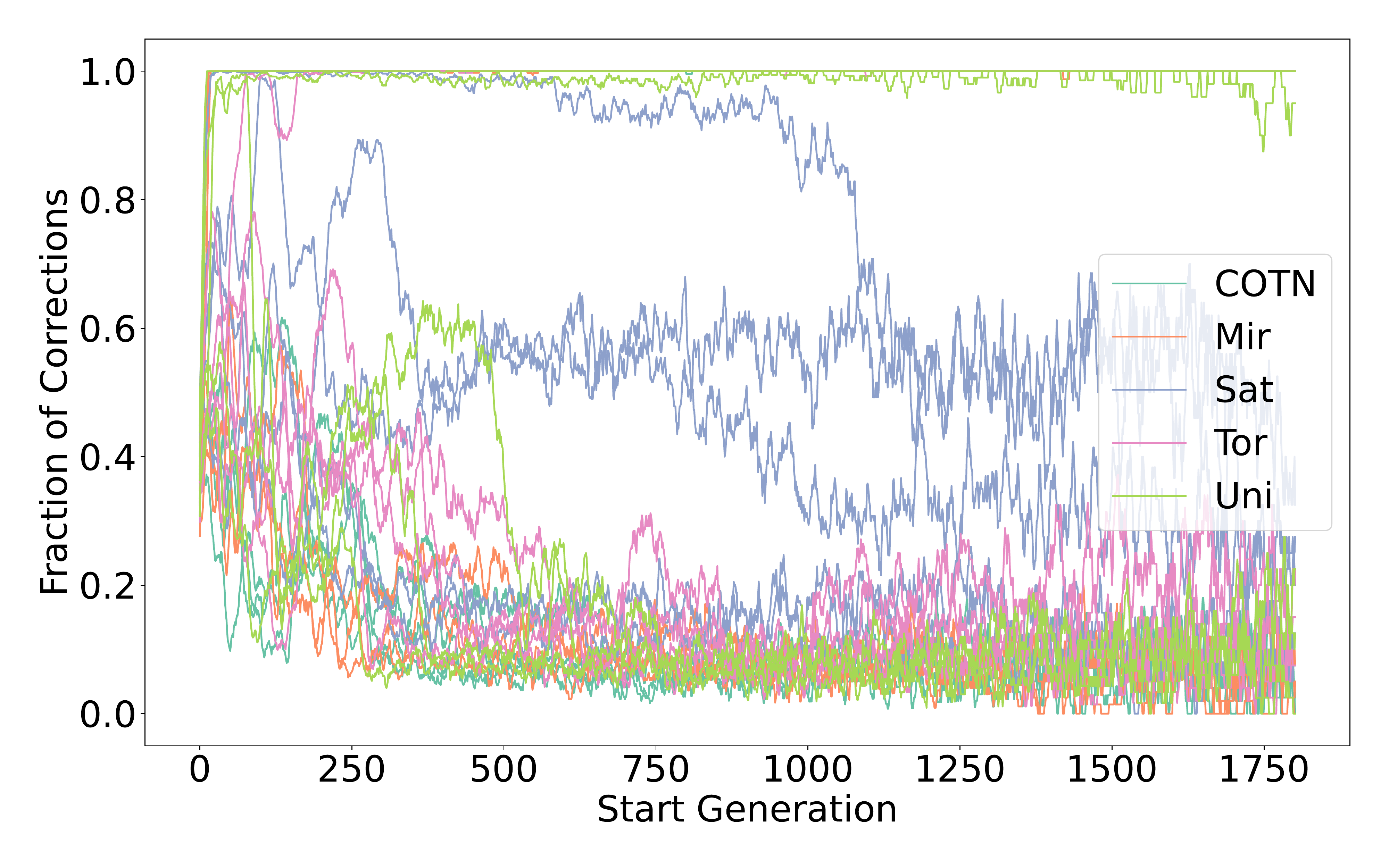}\label{fig:F23_D5_LShade_POIS}}
    \subfigure[Population diversity, \texttt{L-SHADE}, 5D]{\includegraphics[width=0.495\textwidth,trim=7mm 15mm 7mm 7mm,clip]{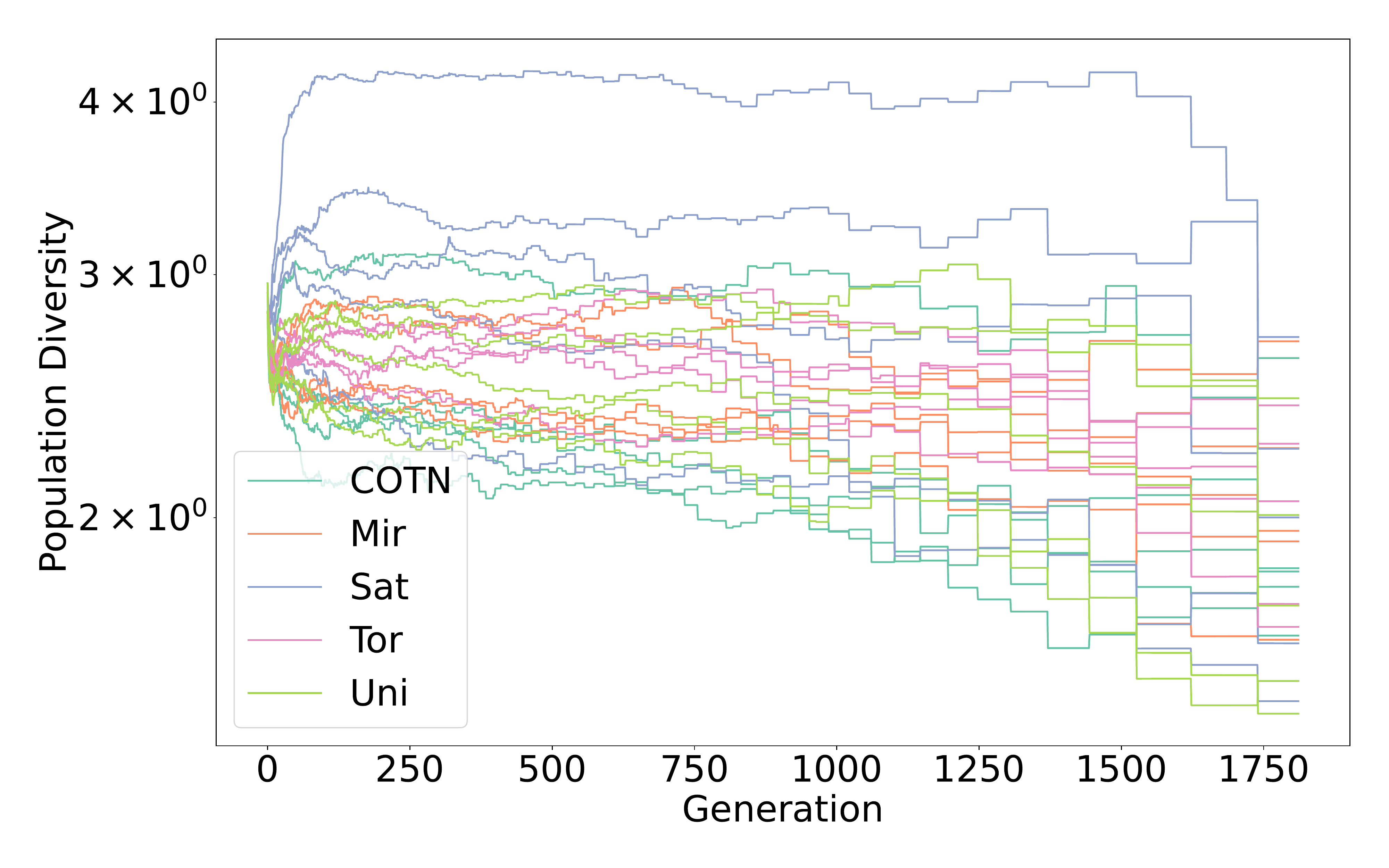}\label{fig:F23_D5_LShade_div}}
    \subfigure[Windowed POIS, \texttt{SHADE}, 30D ]{\includegraphics[width=0.495\textwidth,trim=7mm 15mm 7mm 7mm,clip]{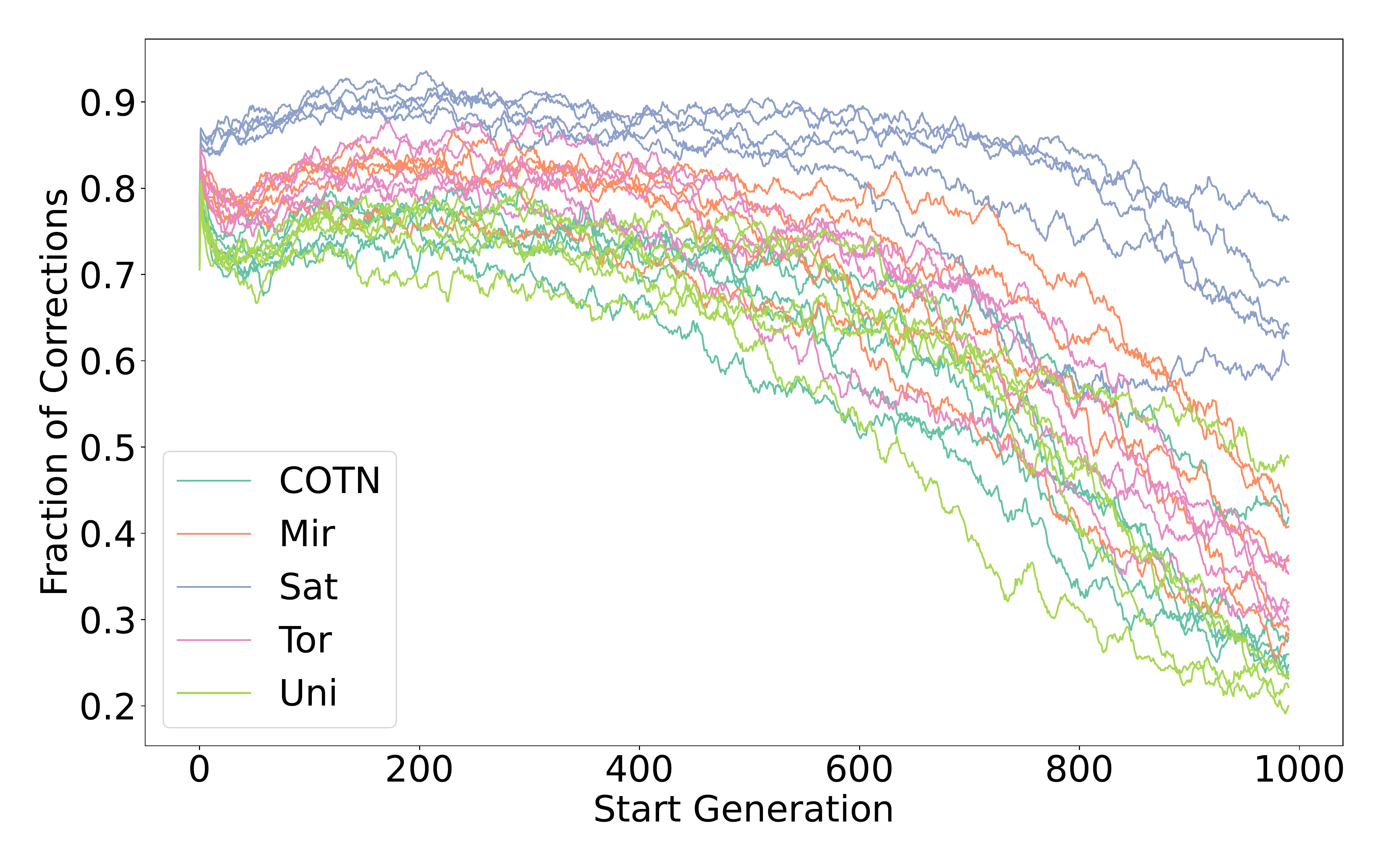}\label{fig:F23_D30_Shade_POIS}}
    \subfigure[Population diversity, \texttt{SHADE}, 30D]{\includegraphics[width=0.495\textwidth,trim=7mm 15mm 7mm 7mm,clip]{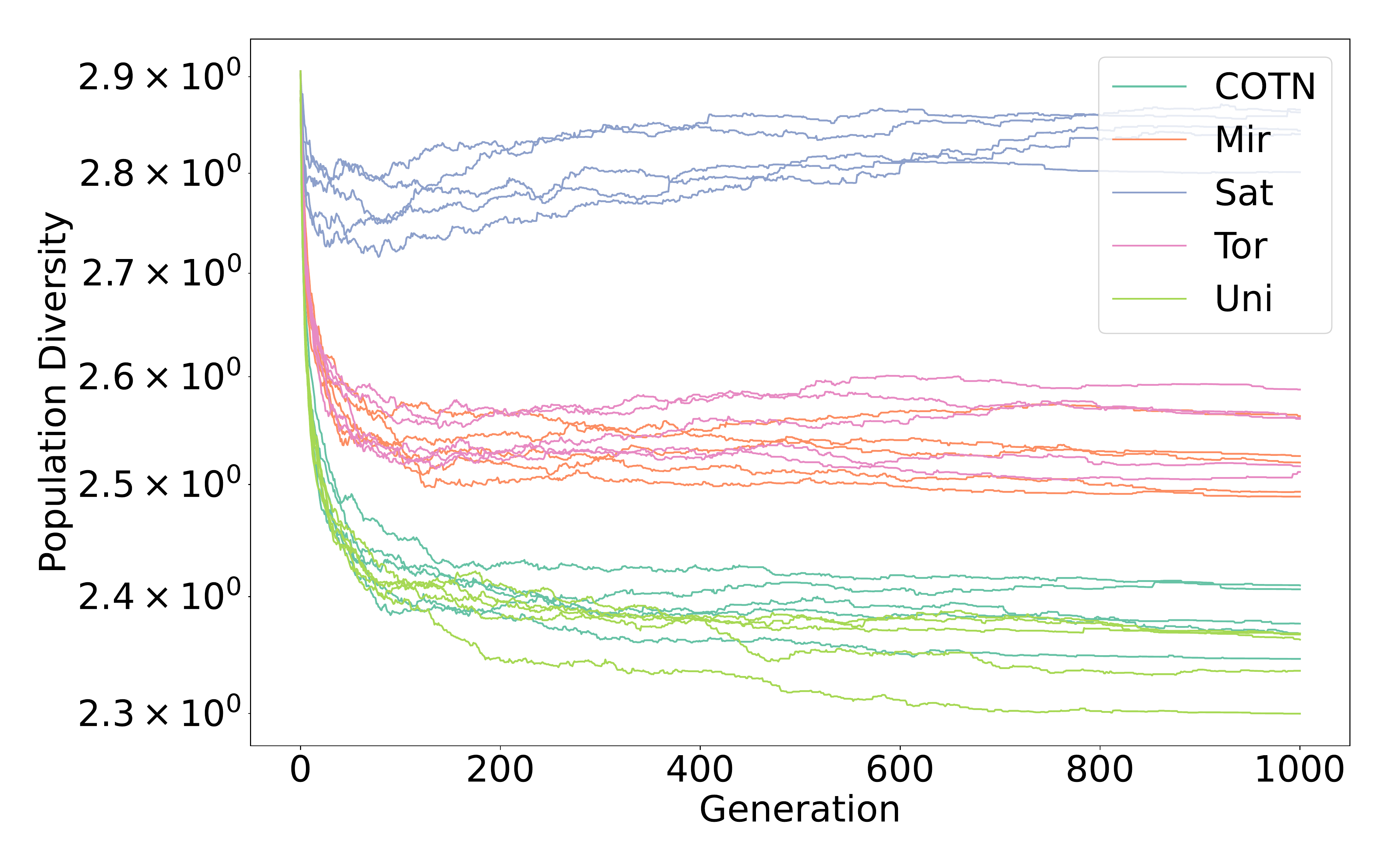}\label{fig:F23_D30_Shade_div}}
    \subfigure[Windowed POIS, \texttt{L-SHADE}, 30D ]{\includegraphics[width=0.495\textwidth,trim=7mm 15mm 7mm 7mm,clip]{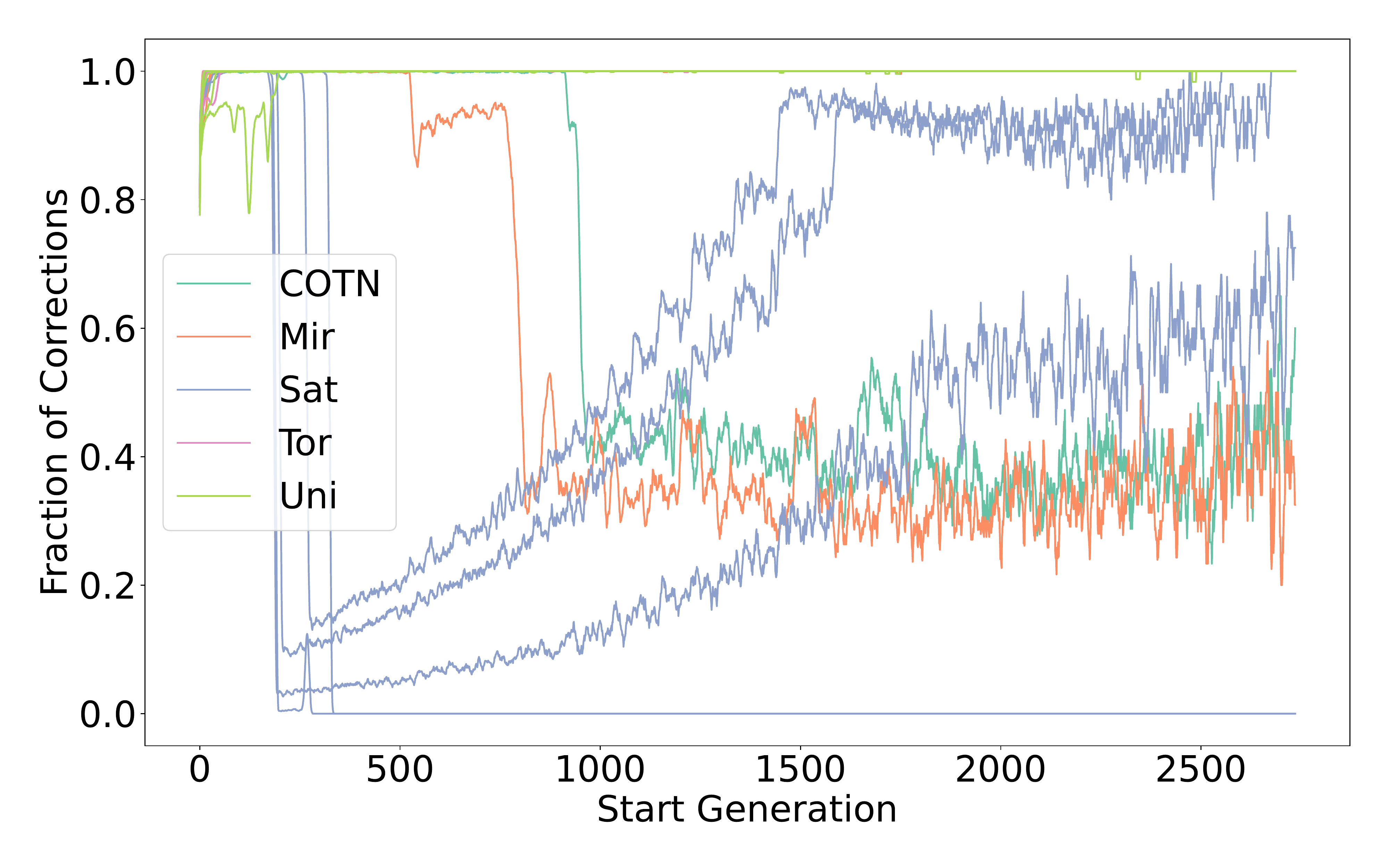}\label{fig:F23_D30_LShade_POIS}}
    \subfigure[Population diversity, \texttt{L-SHADE}, 30D]{\includegraphics[width=0.495\textwidth,trim=7mm 15mm 7mm 7mm,clip]{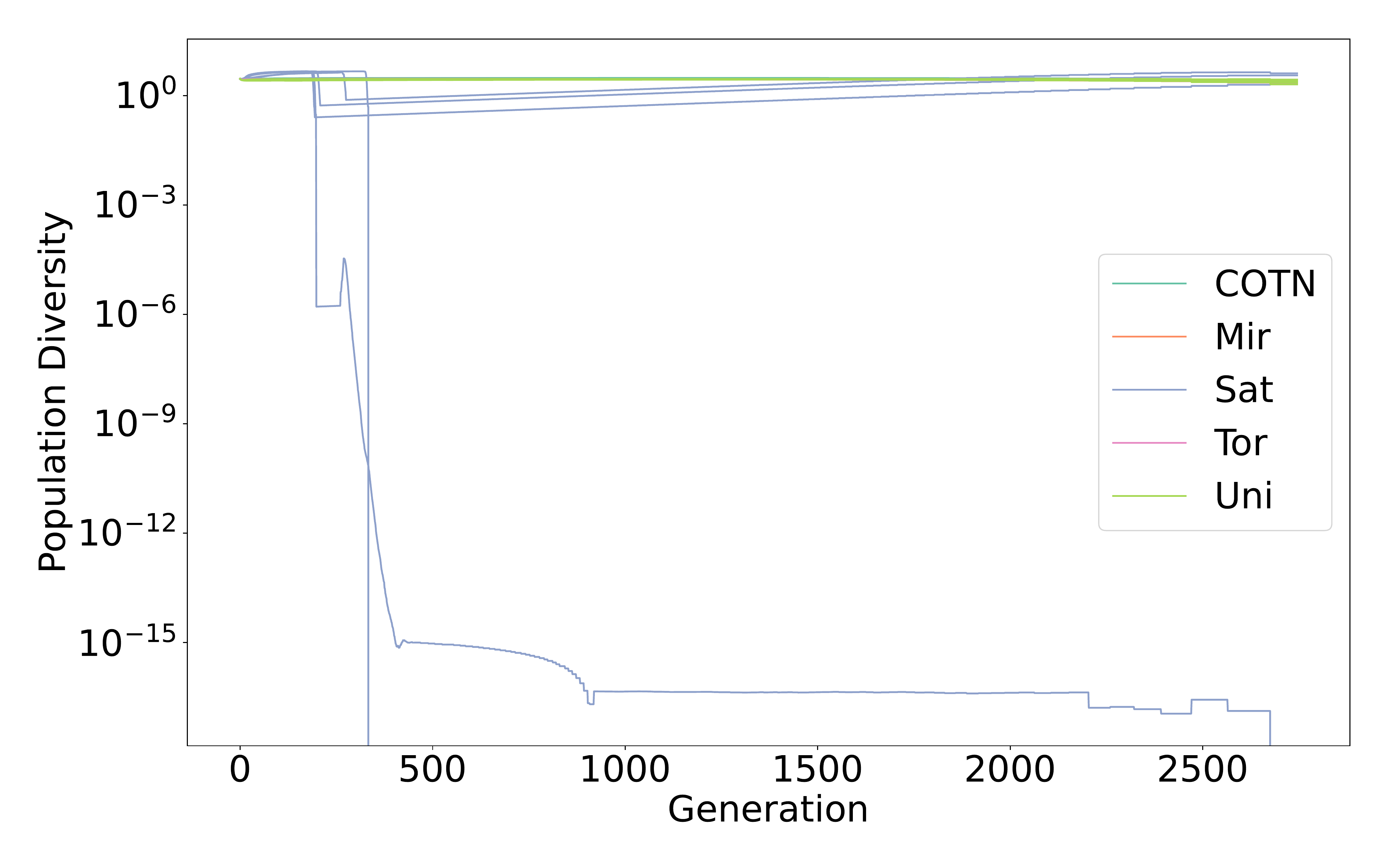}\label{fig:F23_D30_LShade_div}}
    \caption{Evolution of windowed POIS (left column) and population diversity (right column) over generations for different SDIS methods on 5 runs of $f_{23}$, instance 1. None of the runs reached the optimum.}\label{fig:F23_LShade_POIS_div}
\end{figure}

\begin{figure}
    \centering
    \subfigure[Windowed POIS, \texttt{SHADE}, 5D ]{\includegraphics[width=0.495\textwidth,trim=7mm 15mm 7mm 7mm,clip]{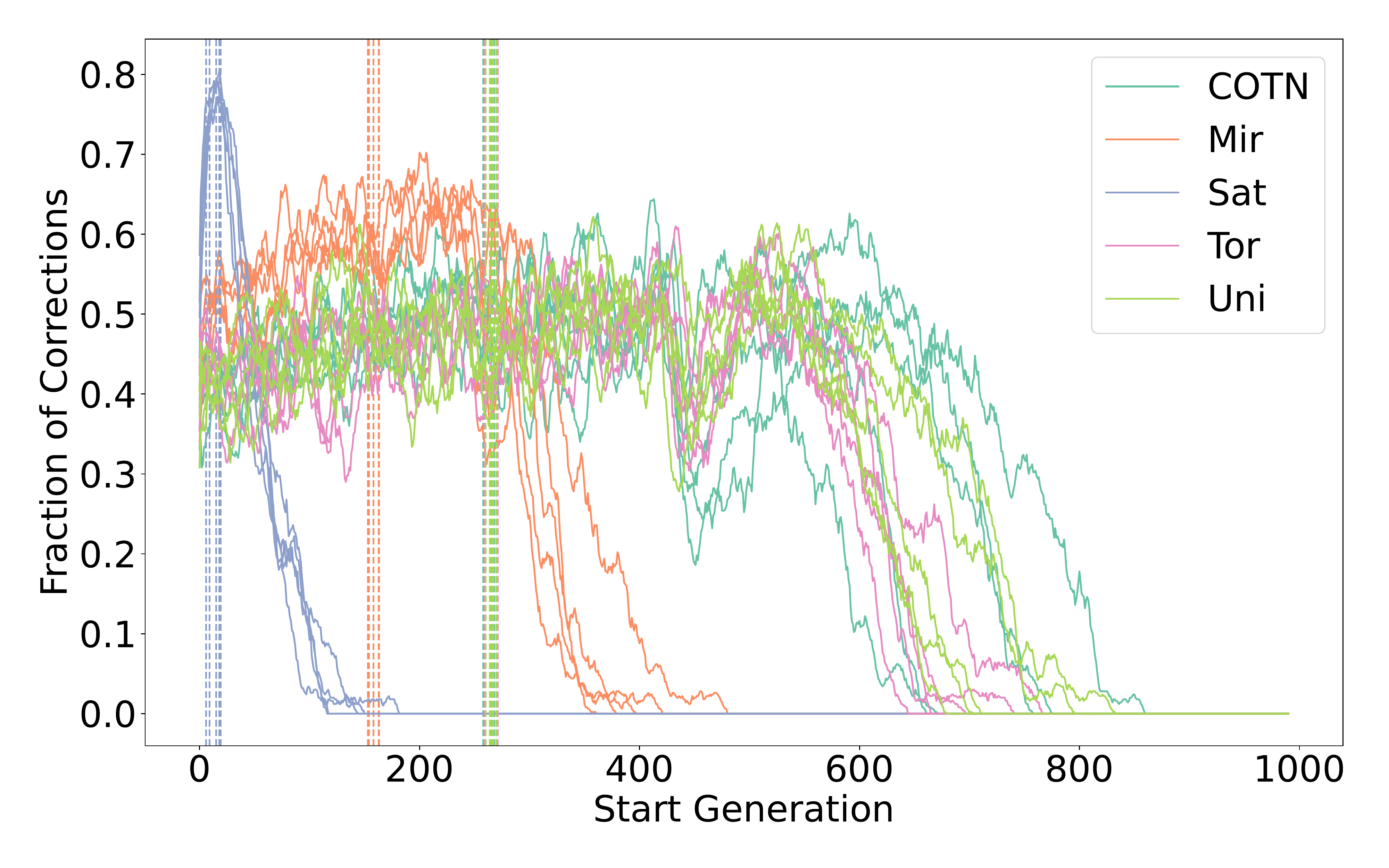}\label{fig:F5_5D_Shade_POIS}}
    \subfigure[Population diversity, \texttt{SHADE}, 5D]{\includegraphics[width=0.495\textwidth,trim=7mm 15mm 7mm 7mm,clip]{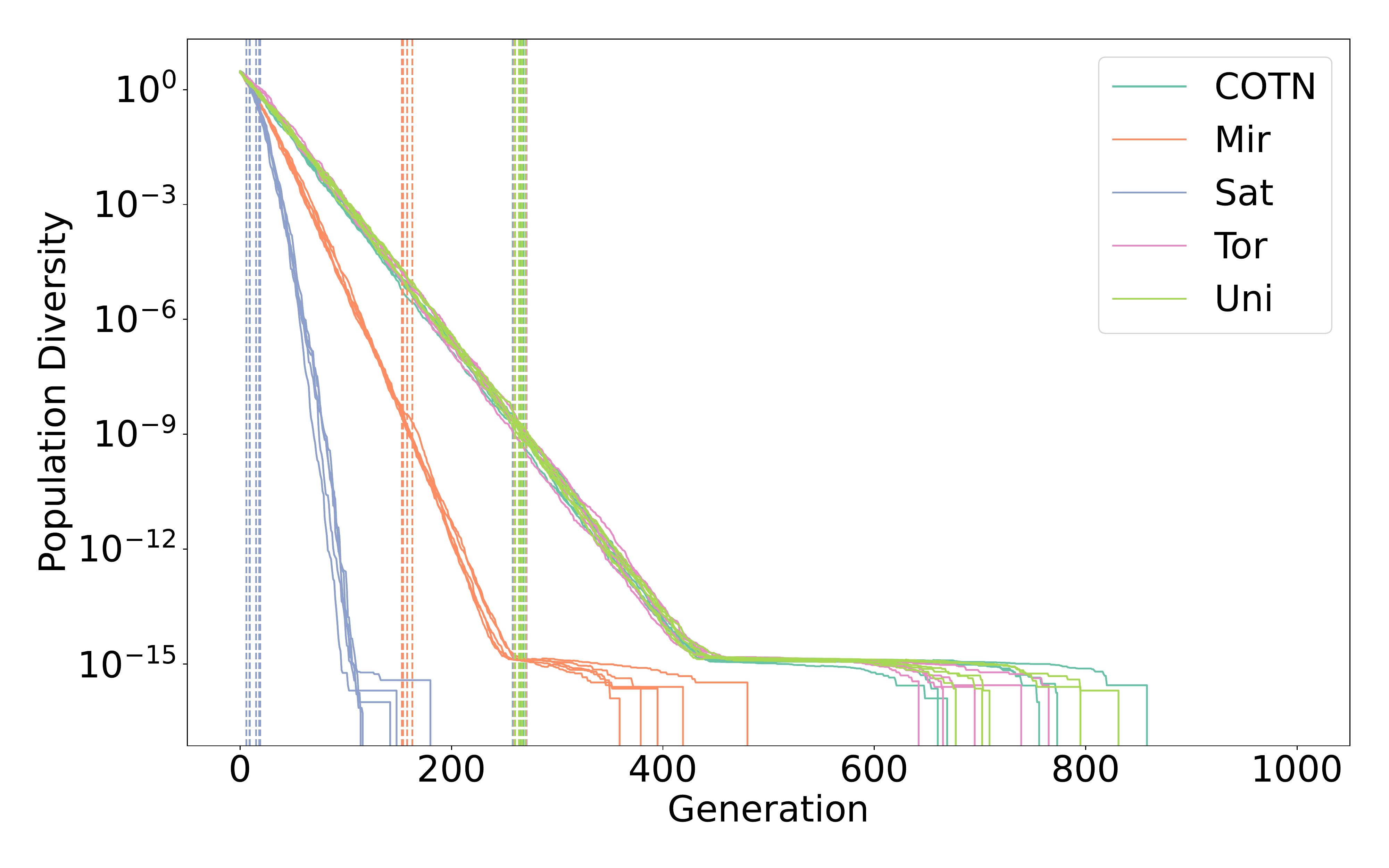}\label{fig:F5_5D_Shade_div}}
    \subfigure[Windowed POIS, \texttt{L-SHADE}, 5D ]{\includegraphics[width=0.495\textwidth,trim=7mm 15mm 7mm 7mm,clip]{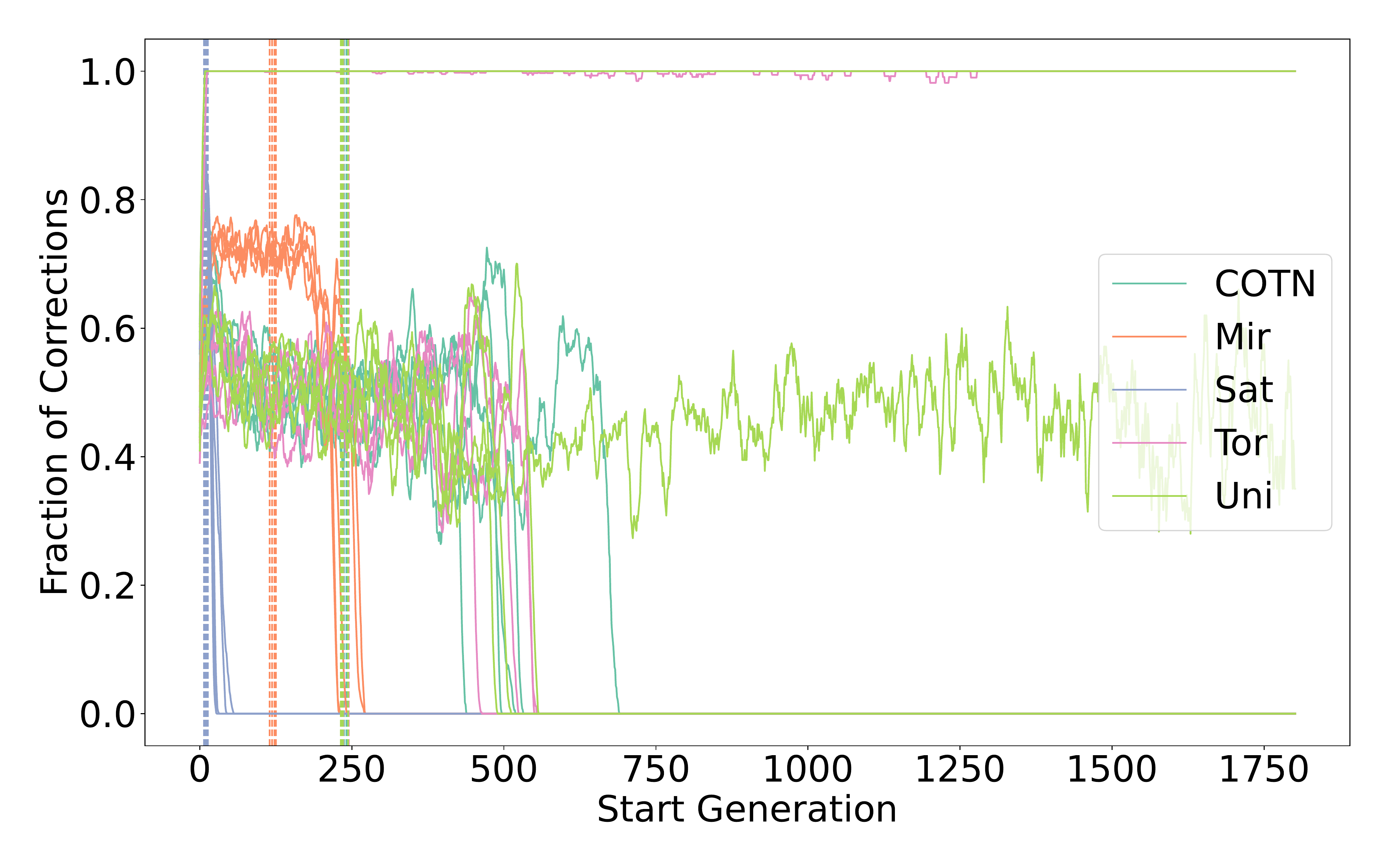}\label{fig:F5_5D_LShade_POIS}}
    \subfigure[Population diversity, \texttt{L-SHADE}, 5D]{\includegraphics[width=0.495\textwidth,trim=7mm 15mm 7mm 7mm,clip]{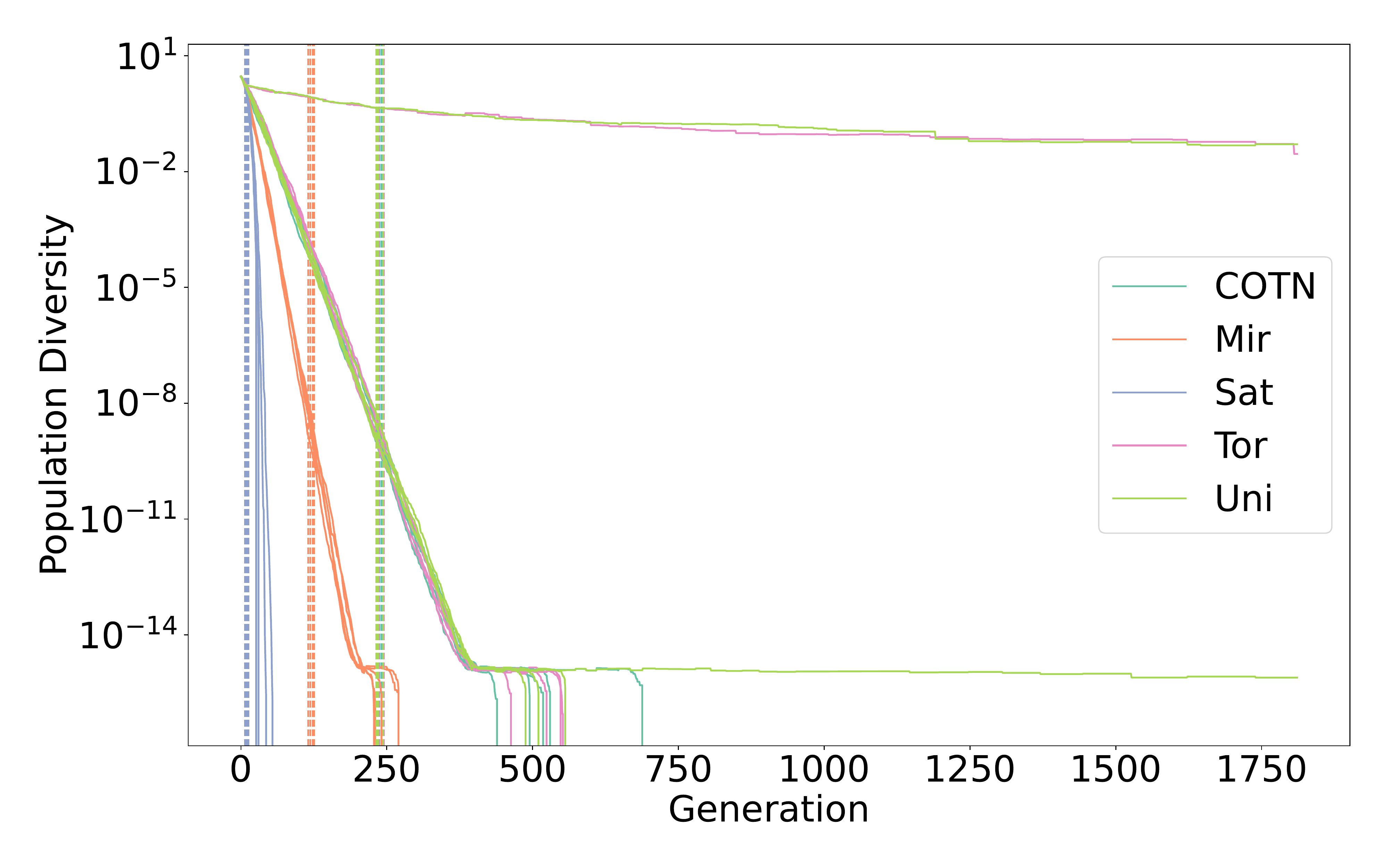}\label{fig:F5_5D_LShade_div}}
    \subfigure[Windowed POIS, \texttt{SHADE}, 30D]{\includegraphics[width=0.495\textwidth,trim=7mm 15mm 7mm 7mm,clip]{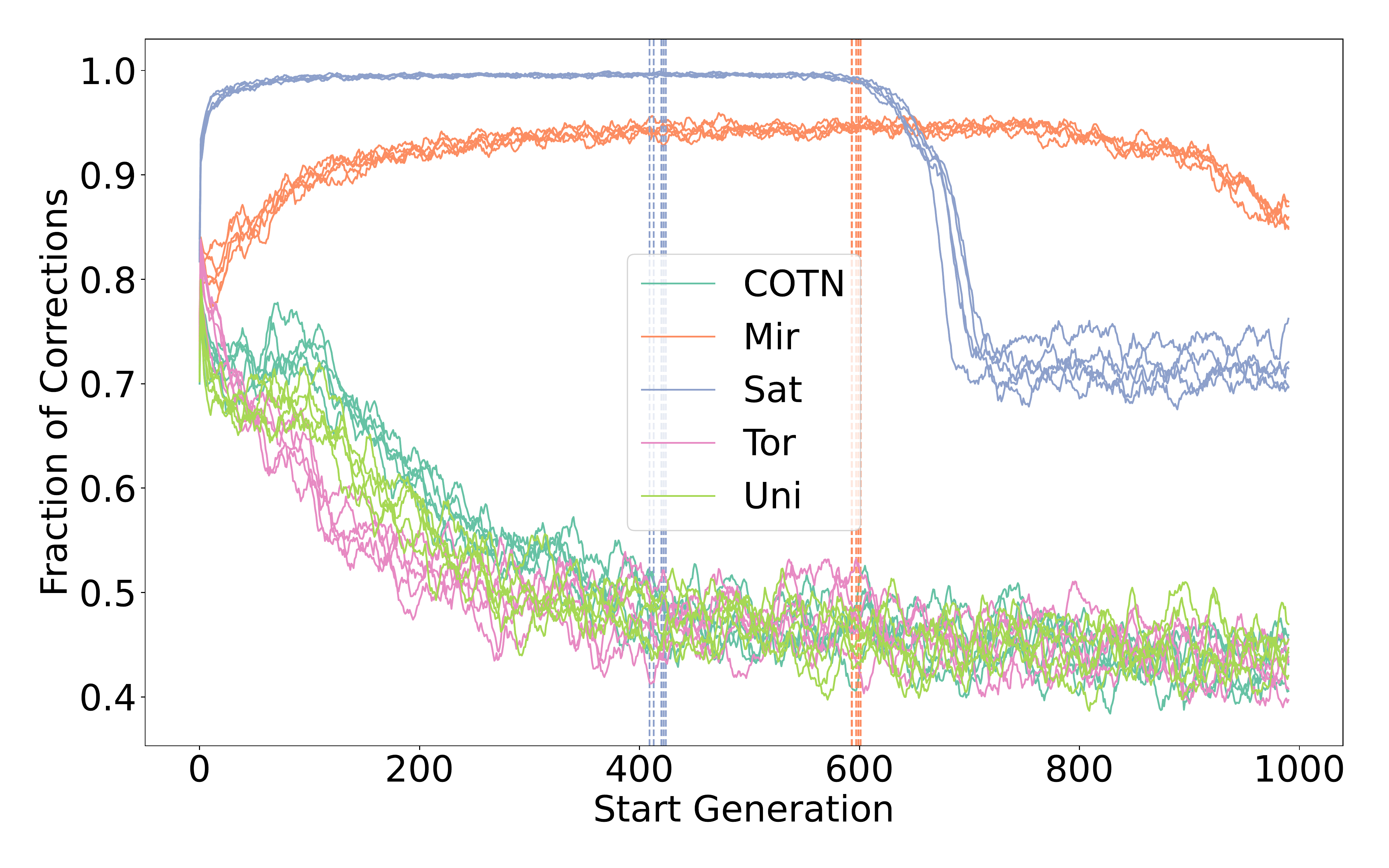}\label{fig:F5_30D_Shade_POIS}}
    \subfigure[Population diversity, \texttt{SHADE}, 30D]{\includegraphics[width=0.495\textwidth,trim=7mm 15mm 7mm 7mm,clip]{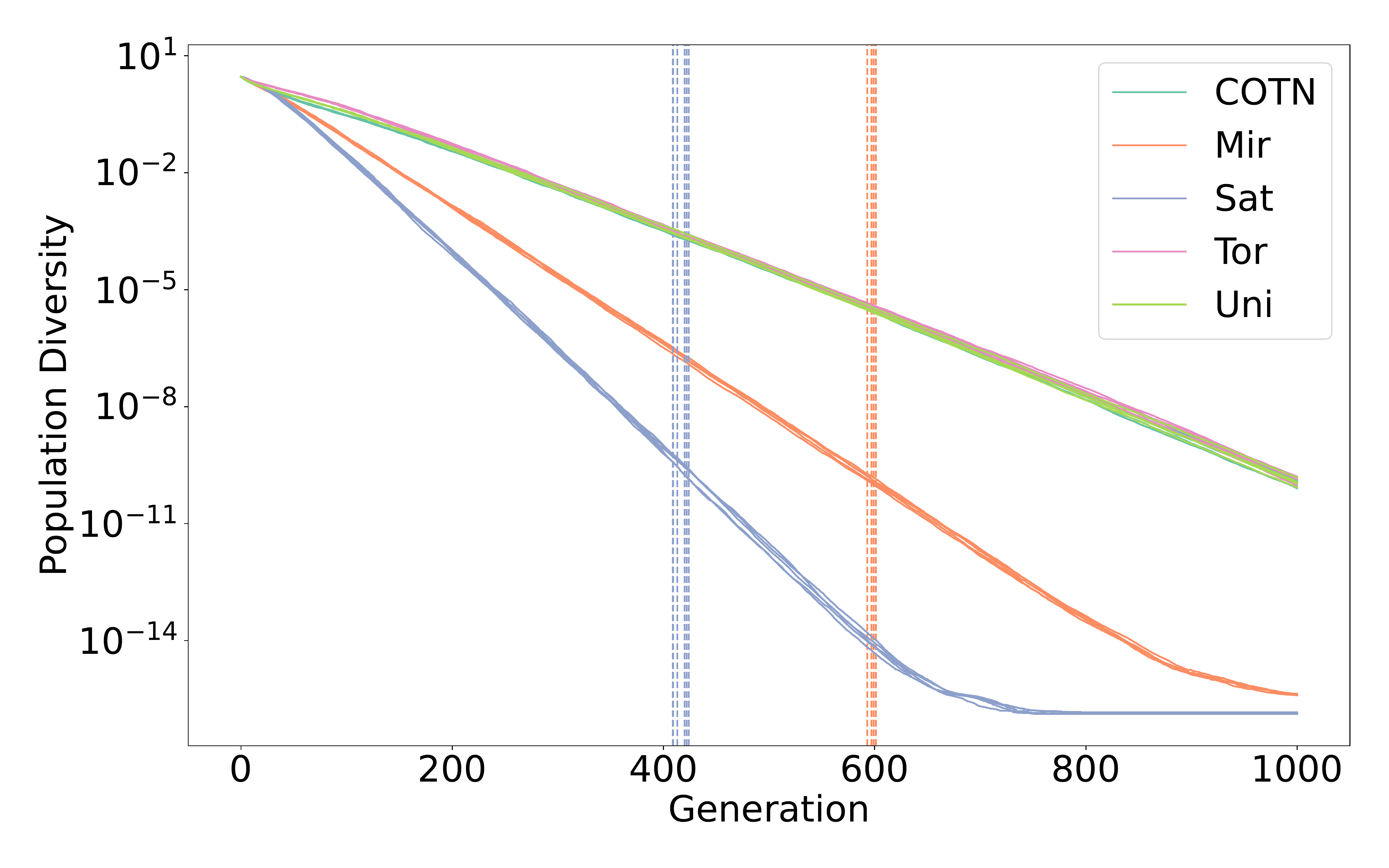}\label{fig:F5_30D_Shade_div}}
    \subfigure[Windowed POIS, \texttt{L-SHADE}, 30D]{\includegraphics[width=0.495\textwidth,trim=7mm 15mm 7mm 7mm,clip]{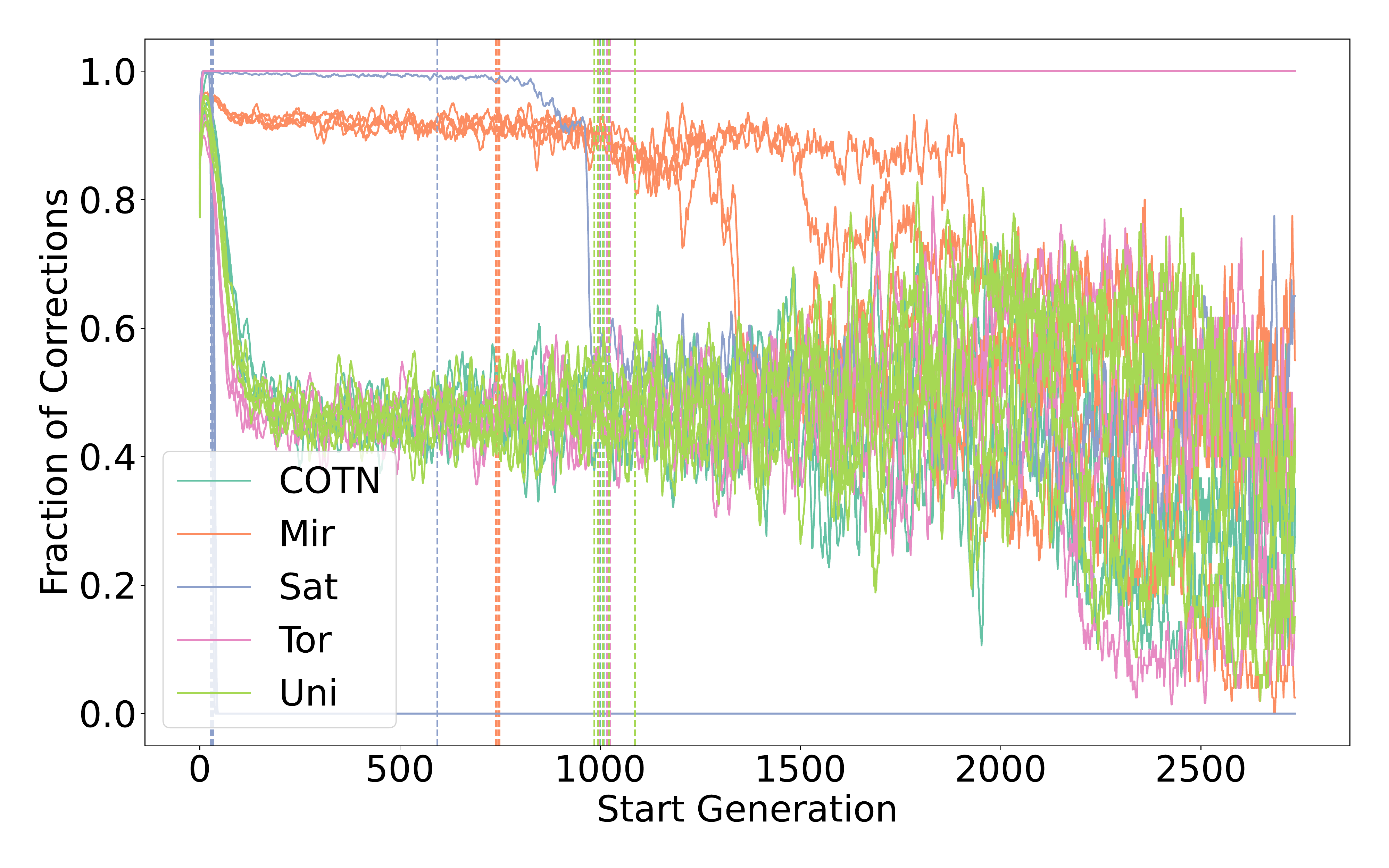}\label{fig:F5_30D_LShade_POIS}}
    \subfigure[Population diversity, \texttt{L-SHADE}, 30D]{\includegraphics[width=0.495\textwidth,trim=7mm 15mm 7mm 7mm,clip]{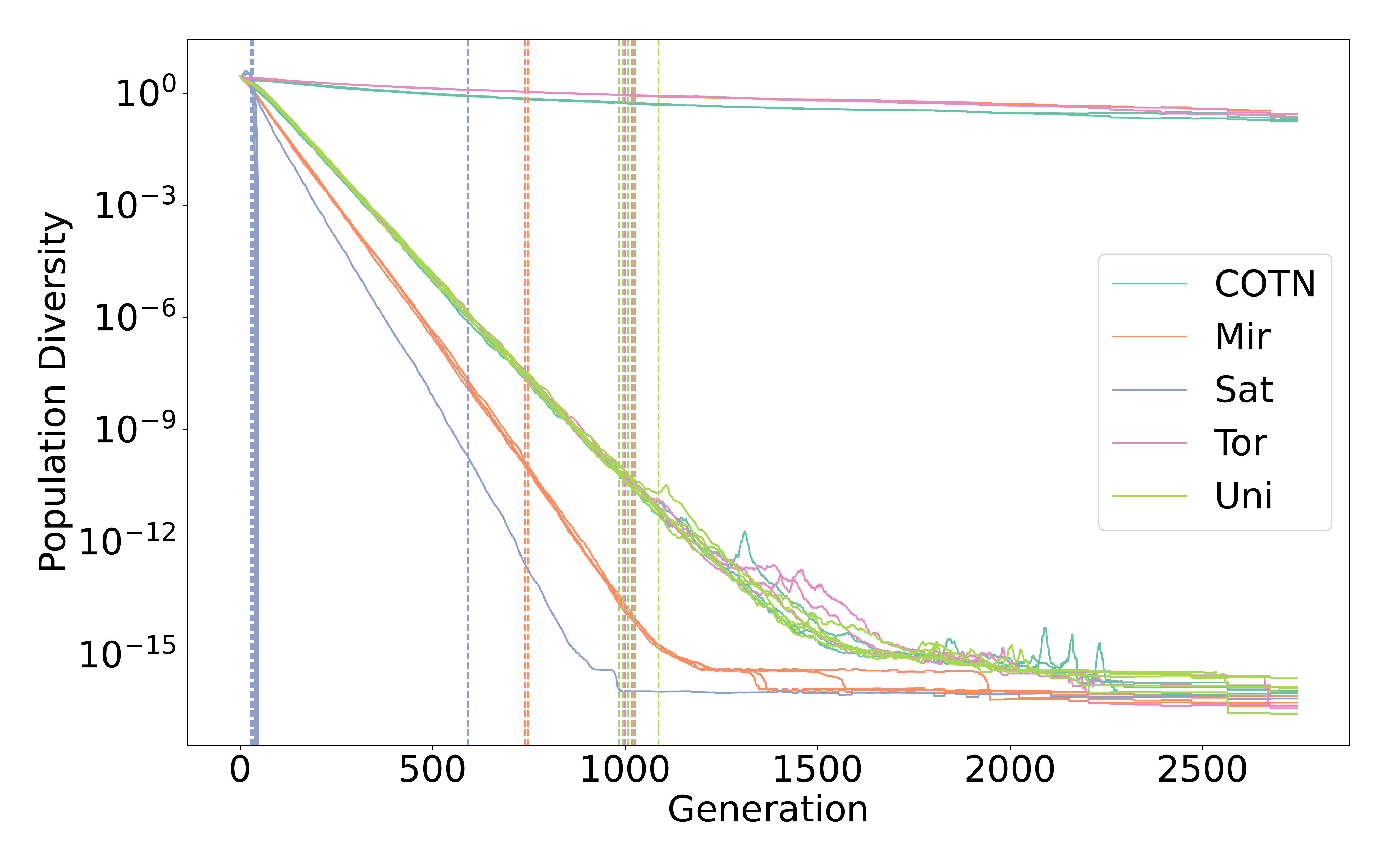}\label{fig:F5_30D_LShade_div}}
    \caption{Evolution of windowed POIS (left column) and population diversity (right column) over generations for different SDIS methods on 5 runs of $f_{5}$, instance 1. Coloured vertical lines show times when optimum was found at least once in a run.}\label{fig:F5_LShade_POIS_div}
\end{figure}

While final POIS discussed in Section~\ref{sect:pois} is very useful for giving a total number of infeasible solutions, we can get a more detailed overview of the number of infeasible solutions generated by taking a sliding-window approach: for each generation, we can calculate the fraction of infeasible solutions generated \textit{thus far}~\footnote{Equivalent to the fraction of solutions which have had a SDIS applied.} and visualise the moving average from $10$ generations. We can use a similar approach for population diversity, which allows us to check for possible correlations between on-going POIS and diversity. Figures~\ref{fig:F23_LShade_POIS_div} and~\ref{fig:F5_LShade_POIS_div} visualise these two kinds of plots for functions $f_5$ and $f_{23}$, respectively. 

In the case of the multi-modal $f_{23}$ function (Figure~\ref{fig:F23_LShade_POIS_div}) the behaviour of SHADE is in accordance with some of the theoretical insights concerning the amount of infeasible elements (largest amount in the case of \texttt{sat}) and population diversity evolution (largest diversity in the case of \texttt{sat}, smallest in the case of \texttt{uni} and intermediate in the case of \texttt{mir} and \texttt{tor}). Meanwhile, in particular for $f_5$, we can observe in Figure~\ref{fig:F5_LShade_POIS_div} a clear pattern between the different SDIS variants: \texttt{sat} initially has a much larger POIS, which would indeed be beneficial, as the optimum of this function is located on the bounds. The other SDISs don't have this direct benefit of \texttt{sat}, and need to generate a point on the boundary exactly, without the correction moving it away from the optimum. For the 5-dimensional version of the function, this is still achievable for most SDISs, with the lower disruptiveness of \texttt{mir} allowing it to find the optimum relatively easily, but for 30 dimensions the most disruptive SDIS have runs which fail to converge to a single solution, explaining their poor performance seen in Figure~\ref{fig:30D_overview}. 

We also notice that the differences between SHADE and L-SHADE are rather significant, especially on the higher dimensions -- detailed reasons for this require further investigation. 
\section{Conclusions}\label{sec:conclusion}

\begin{table}[!tb]
    \centering
    \caption{Ranking of SDIS based on theoretical (obtained where possible) and empirical results with respect to the amount of infeasible components, search direction disruptiveness measured using cosine (dis)similarity, population diversity measured using the component-wise variance and performance measured using ERT. }\label{tab:SDISdiscussion} 
    \begin{tabular}{cp{0.04\linewidth}p{0.04\linewidth}p{0.04\linewidth}p{0.04\linewidth}p{0.04\linewidth}p{0.04\linewidth}p{0.125\linewidth}}
    \toprule
     \multirow{3}{*}{\textbf{Aspect}} & \multicolumn{2}{c}{Amount of} & \multicolumn{2}{c}{\multirow{2}{*}{Amount of}} & \multicolumn{2}{c}{Increase of} & \multicolumn{1}{c}{\multirow{2}{*}{Fitness-based}}\\
   & \multicolumn{2}{c}{infeasible} & \multicolumn{2}{c}{\multirow{2}{*}{disruptiveness}} & \multicolumn{2}{c}{population} & \multicolumn{1}{c}{\multirow{2}{*}{performance}}\\
    & \multicolumn{2}{c}{components} & \multicolumn{2}{c}{} & \multicolumn{2}{c}{diversity} & \multicolumn{1}{c}{}\\
    \textbf{Sorting} & \multicolumn{2}{c}{smaller-larger}& \multicolumn{2}{c}{smaller-larger} & \multicolumn{2}{c}{larger-smaller} & \multicolumn{1}{c}{best-worst}\\
    \textbf{Algorithm} & \multicolumn{2}{c}{\texttt{DE/rand/1/*}} &  \multicolumn{2}{c}{\texttt{DE/rand/1/*}} &  \multicolumn{2}{c}{\texttt{DE/rand/1/*}} & \multicolumn{1}{c}{\texttt{(L-)SHADE}}\\
    \textbf{Type}& \multicolumn{1}{c}{theor.} & \multicolumn{1}{c}{empir.} & \multicolumn{1}{c}{theor.} &  \multicolumn{1}{c}{empir.} & \multicolumn{1}{c}{theor.} & \multicolumn{1}{c}{empir.} & \multicolumn{1}{c}{empir. } \\ 
    \textbf{Function(s)} & \multicolumn{1}{c}{flat} & \multicolumn{1}{c}{$f_0$} & \multicolumn{1}{c}{flat} &  \multicolumn{1}{c}{$f_0$} & \multicolumn{1}{c}{flat} & \multicolumn{1}{c}{$f_0$} & \multicolumn{1}{c}{BBOB} \\ 
    \textbf{Section} & \multicolumn{1}{c}{\ref{sect:ViolProb}} & \multicolumn{1}{c}{\ref{sect:POIS_f0}} & \multicolumn{1}{c}{\ref{sect:cosine}} & \multicolumn{1}{c}{\ref{sect:CSAnalysis}} & \multicolumn{1}{c}{\ref{sect:diversityTheor}} & \multicolumn{1}{c}{\ref{sect:diversity}} & \multicolumn{1}{c}{\ref{sect:BBOB_perf}}\\
    \midrule
    \textbf{1.} & \multicolumn{1}{c}{\texttt{uni}} & 
    \multicolumn{1}{c}{\texttt{COTN}} &
    \multicolumn{1}{c}{\texttt{sat}} &
    \multicolumn{1}{c}{\texttt{sat}} &
    \multicolumn{1}{c}{\texttt{sat}} &
    \multicolumn{1}{c}{\texttt{sat}} &
    \multicolumn{1}{c}{\texttt{sat}} \\
    
    \midrule
    \multirow{2}{*}{\textbf{2.}} & \multicolumn{1}{c}{\multirow{2}{*}{\texttt{sat}}} &
    \multicolumn{1}{c}{\texttt{mir}} &
    \multicolumn{1}{c}{\multirow{2}{*}{\texttt{mir}}} &
    \multicolumn{1}{c}{\texttt{COTN}} &
     \multicolumn{1}{c}{\texttt{mir}} & \multicolumn{1}{c}{\texttt{mir}} & \multicolumn{1}{c}{\texttt{COTN}} \\ 
    & & \multicolumn{1}{c}{\texttt{tor}} & \multicolumn{1}{c}{} & \multicolumn{1}{c}{\texttt{mir}} & \multicolumn{1}{c}{\texttt{tor}} & \multicolumn{1}{c}{\texttt{tor}} & \multicolumn{1}{c}{\texttt{mir}}\\ 
    \midrule
    \multirow{2}{*}{\textbf{3.}} & \multicolumn{1}{c}{\multirow{2}{*}{\texttt{}}} & \multicolumn{1}{c}{\multirow{2}{*}{\texttt{uni}}} & \multicolumn{1}{c}{\multirow{2}{*}{\texttt{tor}}} & \multicolumn{1}{c}{\multirow{2}{*}{\texttt{uni}}} & \multicolumn{1}{c}{\multirow{2}{*}{\texttt{uni}}} & \multicolumn{1}{c}{\texttt{COTN}} & \multicolumn{1}{c}{\multirow{2}{*}{\texttt{uni}}} \\ 
    & & & & & & \multicolumn{1}{c}{\texttt{uni}}\\ 
    \midrule
    \textbf{4.} & & \multicolumn{1}{c}{\texttt{sat}} & & \multicolumn{1}{c}{\texttt{tor}} & & & \multicolumn{1}{c}{\texttt{tor}} \\ 
    \bottomrule
    \end{tabular}
\end{table}

Results of the comparative analysis on SDIS presented in this paper are summarised in Table~\ref{tab:SDISdiscussion} where the investigated strategies are ranked based on what has been theoretically proved and/or experimentally observed. Despite the fact that the theoretical analysis is limited to subsets of strategies, consistency can be observed between theoretical and experimental results on how various SDISs can be grouped based on their impact on: (i) number of infeasible components; (ii) search direction; (iii) population diversity. The most significant agreement seem to be between the amount of disruptiveness and fitness based performance, suggesting that the strategies with a smaller impact on the search direction, i.e. higher cosine similarity between unconstrained and corrected search directions, have a better performance. When comparing with previously reported results, one can see that the more disruptive strategies, \texttt{uni} and \texttt{tor}, have been also identified in \citep{bib:BIEDRZYCKI2019} as having the highest influence on population distribution. Similarly, in \citep{Padhye2015} it is stated that \texttt{uni} leads to a loss of useful information carried by the current population which confirms the disruptive effect also observed in the current study. On the other hand, \texttt{sat} and \texttt{mir} identified in our study with a small impact on the search direction but a large one on diversity have been consistently amongst well performing strategies, as it is also illustrated in \citep{Kreischer2017} and \citep{Martinez2020}. However, performance advantage of less disruptive strategies on BBOB is not present on all functions equally.

Throughout this paper, we have shown that the strategy of dealing with infeasible solutions within Differential Evolution has a clear impact on the behaviour of the algorithm as a whole. While DE allows us to combine insights of the impact of SDIS from both a theoretical and empirical perspective, it is by no means the only algorithm where the way of handling box-constraints can have an impact on the overall search behaviour. This highlights an important issue in the field of evolutionary computation as a whole, as methods such as SDIS are often overlooked because they are considered unimportant compared to the novel algorithmic ideas discussed in the literature. When such omission is combined with other factors, such as inaccessibility of the source code used, this leads to a significant amount of ambiguity, even when other operators are defined clearly. As such, reproducing results requires specificity and ideally open source code for \textit{all} used operators, not only the core algorithmic components. 

In order to achieve a \textit{better standard for reproducibility}, SDIS should be considered as an operator to be specified in \textit{every} optimisation algorithm which deals with box-constraints, as the interaction of the algorithm with bounds can cause clear differences in behaviour, especially when for some of the high-dimension version of functions in the BBOB-set the state-of-the art DE variants use the boundary correction in a majority of their mutations. 
Finally, as such, we see potential benefits for the inclusion of SDIS in an \textit{automated algorithm configuration task}.

\section{Future work}\label{sec:future}

While we have investigated the impact of SDIS on several variants of differential evolution throughout this work, it remains formally unconfirmed that such findings translate to other optimisation algorithms. Since many of these algorithms are built to work with box-constraints by default, it stands to reason that the used SDIS would have an influence, and by studying this in more detail we could potentially widen our understanding of the interactions between SDIS, algorithm and objective function. 

In addition, we have started to lay a theoretical foundation for the study of SDIS by considering their disruptiveness and impact on diversity. By continuing to built upon these notions, we hope to gain a more detailed understanding on the impact of SDIS on search behaviour in general.

\appendix
\section{Proofs for statements in Sections~\ref{sect:theoretic} and ~\ref{sect:exponf0}}\label{sect:appendix_proofs}

\subsection{Amount of infeasible components. Violation probability for \texttt{sat} -- Section~\ref{sect:ViolProb}}

\begin{prop} \label{prop:ViolProb}
If $p_v(g)$ denotes the violation probability corresponding to generation $g$, then  under the assumptions that the population of elements which are within the bounds (i.e. in $(a_i,b_i)$) remains almost uniformly distributed and \texttt{sat} is applied for infeasible elements, the violation probability satisfies
\end{prop}

\begin{equation}p_v(g+1)=p_v(g)/2+(1-p_v(g))(p_v^2(g)F/4+(1-p_v^2(g))F/3)\end{equation}

\begin{proof}The analysis is conducted on one component, thus the component index is skipped. Let us consider that the population corresponding to generation $g$, $P(g)$, consists of three subpopulations: $P(g)=P_w(g)\cup P_{lb}(g) \cup P_{ub}(g)$, corresponding to mutants generated inside the bounding box ($P_w(g)$), placed on the lower bound ($P_{lb}(g)$) and placed on the upper bound ($P_{ub}(g)$). If $m$ denotes the number of elements in $P(g)$, the expected size of $P_w(g)$ is $m(1-p_v(g))$ and the expected size of $P_{lb}(g)$ and $P_{ub}(g)$ is $mp_v(g)/2$. To estimate the probability of generating infeasible mutants in generation $g+1$ we consider  two cases:

\begin{description}
\item{(a)} The base element, $x_{r_1}$, is on one of the bounds ($x_{r_1}\in P_{lb}(g) \cup P_{ub}(g)$). In this case the probability of generating an infeasible mutant, $x_{r_1}+F\cdot (x_{r_2}-x_{r_3})$, is $1/2$ as there are no incentives to sample more frequently positive or negative differences. On the other hand, the probability of selecting a base element from the bounds is $p_v(g)$ (under the assumption that the probability of generating through DE mutation elements on the bounds is negligible, i.e. they are generated only by applying the \texttt{sat} strategy). Thus in this case $p_v(g+1)=p_v(g)/2$.
\item{(b)} The base element $x_{r_1}$ is strictly between the bounds (event of probability $(1-p_v(g))$). In this case we can analyse three subcases: 
\begin{description}
\item{(i)} both $x_{r_2}$ and $x_{r_3}$ belong to the same bound (event of probability $p_v^2(g)/2$): in this case the mutant will be just $x_{r_1}$, thus feasible;
\item{(ii)} $x_{r_2}$ and $x_{r_3}$ belong to different bounds (event of probability $p_v^2(g)/2$): in this case the probability to generate an infeasible mutant is $F/2$;
\item{(iii)} $x_{r_2}$ and $x_{r_3}$ belong both to $P_w$ or at most one is on the bound (event of probability $(1-p_v^2(g))$): in this case, under the assumption that the population of elements belonging to $P_w$ is almost uniformly distributed, the probability of generating an infeasible mutant is $F/3$.
\end{description}
By combining the probabilities corresponding to these complementary events one obtains: 
$$p_v(g+1)=\frac{p_v(g)}{2}+(1-p_v(g))\left(p^2_v(g)\frac{F}{4}+(1-p_v^2(g))\frac{F}{3}\right)$$
\end{description}

\end{proof}

\begin{figure}[!tb]
    \centering
    \includegraphics[width=0.49\textwidth]{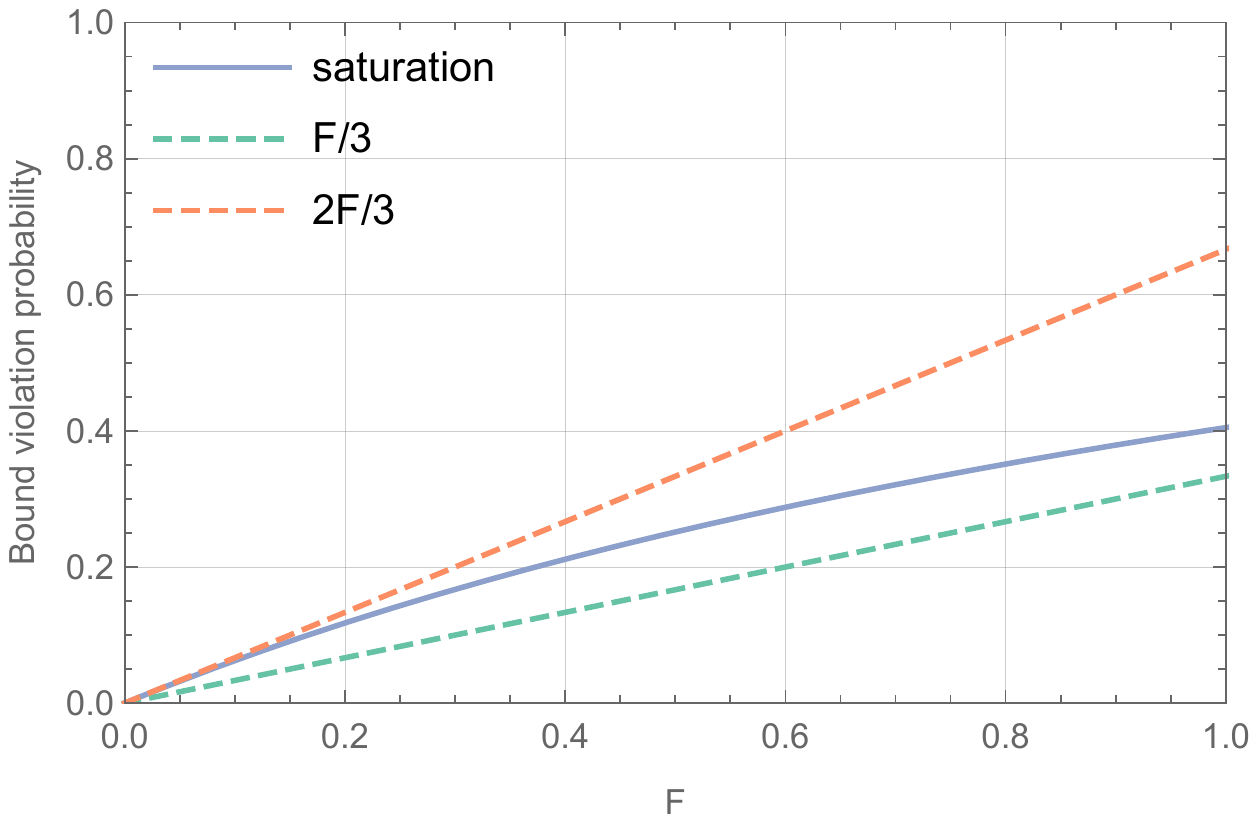}
    \includegraphics[width=0.49\textwidth]{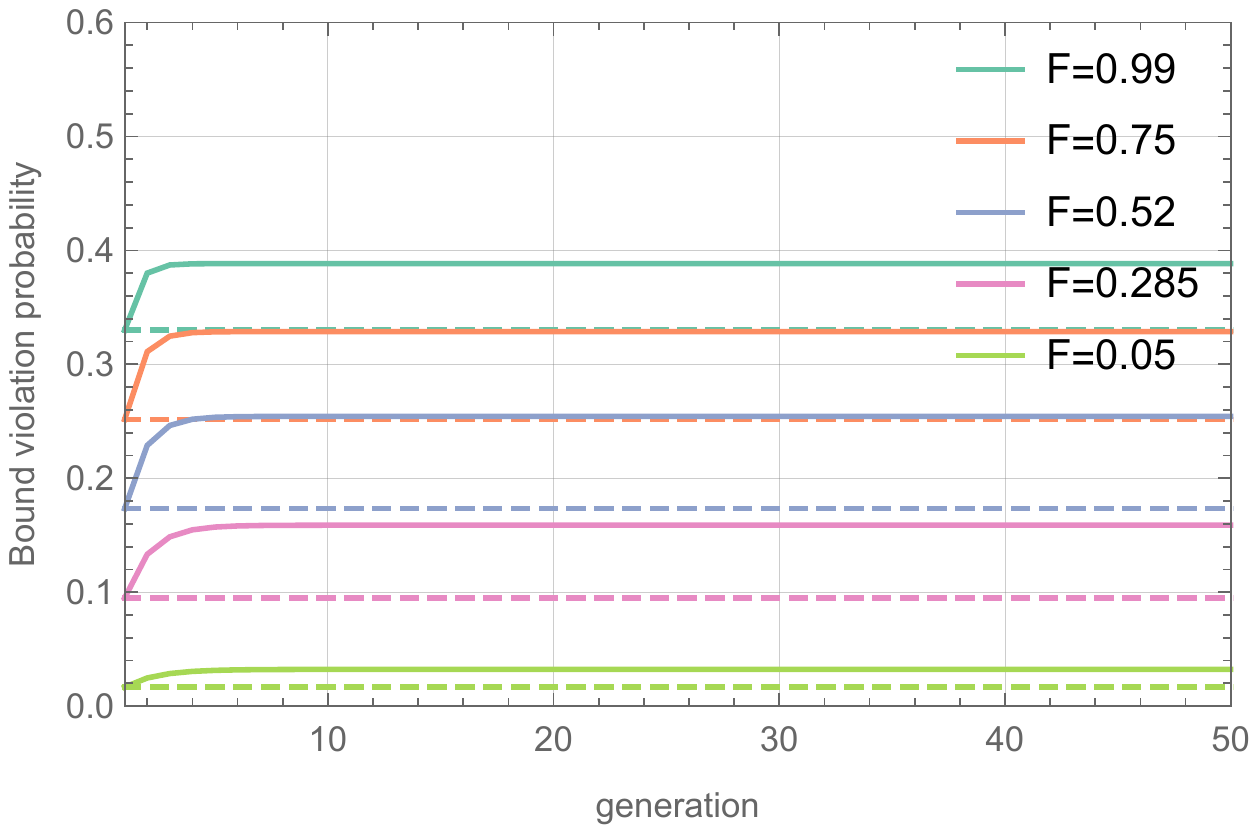}
    \caption{Probability of bounds violation in the case of \texttt{sat} (a) limit value (left); (b) evolution during generations (right).}
    \label{fig:ViolProbSaturate}
\end{figure}

\noindent{\it Remark.} It is easy to see that for $F\in [0,1]$ the sequence $p_v(g)$ converges to a value which is between $F/3$ and $2F/3$ (Fig.~\ref{fig:ViolProbSaturate}).

\subsection{(Dis)similarity between search directions: cosine similarity for search directions corresponding to \texttt{mir} and \texttt{tor} -- Section~\ref{sect:cosine}}

\begin{prop}  \label{prop:cosine}
If $0<F\leq 0.5$ and $x$ and $c_M(z)$ belong to the same quadrant, i.e. $(c_M(z_i)-(a_i+b_i)/2)( x_i-(a_i+b_i)/2)\geq 0$ for $i=\overline{1,n}$ then $\cos(d,d_M)\geq \cos(d,d_T)$.
\end{prop}

\begin{proof} Since $c_M(z_i)+c_T(z_i)=a_i+b_i$ it follows that $d^Td_M-d^Td_T=\sum_{i=1}^n(z_i-x_i)(2c_M(z_i)-(a_i+b_i))$. If $F\leq 0.5$ then an infeasible component satisfies either $a_i-(b_i-a_i)/2\leq z_i<a_i$ or $b_i<z_i\leq b_i+(b_i-a_i)/2$. 

Thus, when $a_i-(b_i-a_i)/2\leq z_i< a_i\leq x_i$ then $c_M(z_i)=2a_i-z_i\leq (a_i+b_i)/2$ meaning that $(z_i-x_i)(2c_M(z_i)-(a_i+b_i))\geq 0$.

In the other case, e.g. $x_i\leq b_i<z_i\leq b_i+(b_i-a_i)/2$ one have $c_M(z_i)=2b_i-z_i\geq (a_i+b_i)/2$ leading to $(z_i-x_i)(2c_M(z_i)-(a_i+b_i))\geq 0$. Thus, if $F\leq 0.5$ the scalar product between the DE search direction and the search direction corresponding to \texttt{mir} is larger than that corresponding to \texttt{tor} strategy, i.e. $d^Td_M\geq d^Td_T$. 

On the other hand, when comparing the Euclidean norms of $d_M$ and $d_T$ one obtains: 
$$\|d_T\|^2-\|d_M\|^2=\sum_{j=1}^n[(c_M(z_i)-x_i)^2-(a_i+b_i-c_M(z_i)-x_i)^2]=
$$
\begin{equation}
=\sum_{i=1}^n[(a_i+b_i-2c_M(z_i))(a_i+b_i-2x_i)]
\end{equation}
Thus, if $c_M(z_i)-(a_i+b_i)/2$ and $x_i-(a_i+b_i)/2$ have the same sign, meaning that $c_M(z)$ and $x$ belong to the same quadrant with respect to $(a_i+b_i)/2$, it follows that $\|d_T\|^2\geq \|d_M\|^2$. By combining this result with the property related to the scalar products it follows that $\cos(d,d_M)\geq \cos(d,d_T)$ always when $F\leq 0.5$ and the corrected and the target elements are in the same quadrant.
\end{proof}

\subsection{(Dis)similarity between search directions: cosine similarity analysis for one infeasible component -- Section~\ref{sect:cosine}}
 
\begin{prop} \label{prop:cosineOneComponent}
If there is only one infeasible component, e.g. $z_k>b_k$, and the norm of the search direction satisfies $\|d\|^2=\sum_{i=1}^n(z_i-x_i)^2\geq 2(z_k-x_k)(z_k-b_k)$ then the cosine similarity between the DE search direction and the direction induced by \texttt{sat} is larger than the cosine similarity between the DE search direction and the direction induced by any other SDIS which generate components inside $(a_k,b_k)$.
\end{prop}

\begin{proof} If $d_C$ denotes the search direction induced by a SDIS and $z_k$ is the infeasible component, then the cosine between the DE direction, $d$, and the modified one is:

$$
\cos(d,d_C)=\frac{\|d\|^2+(z_k-x_k)(c(z_k)-x_k)-(z_k-x_k)^2}{\|d\|\sqrt{\|d^2\|+(c(z_k)-x_k)^2-(z_k-x_k)^2}}
$$
\begin{equation} \label{eq:proofCos}
\qquad =\frac{\|d\|^2+(c(z_k)-z_k)(z_k-x_k)}{\|d\|\sqrt{\|d^2\|+(c(z_k)-x_k)^2-(z_k-x_k)^2}}
\end{equation}

To compare $\cos(d,d_C)$ for different SDISs, let us define the function:
\begin{equation} \label{eq:proofC}
C(D,\delta_c,\delta)=\frac{(D+\delta(\delta_c-\delta))^2}{D+\delta_c^2-\delta^2}
\end{equation}

From Eqs.~\ref{eq:proofCos} and~\ref{eq:proofC} it follows that $C(\|d\|^2,c(z_k)-x_k,z_k-x_k)/\|d\|^2=\cos^2(d,d_C)$. Thus to compare $\cos(d,d_S)$ and $\cos(d,d_C)$ it is enough to compare $C(\|d\|^2,c_S(z_k)-x_k,z_k-x_k)$ with $C(\|d\|^2,c(z_k)-x_k,z_k-x_k)$. Let us consider the case when $z_k>b_k$, thus $c_S(z_k)=b_k$. To find sufficient conditions ensuring that  $\cos(d,d_S)\geq \cos(d,d_C)$ one can solve the inequality 
\begin{equation} \label{eq:Ineq}
C(D,\delta_S,\delta)-C(D,\delta_C,\delta)\geq 0
\end{equation}
with respect to $\delta_C$ taking into account the fact that the following conditions are always satisfied: 
\begin{description}
\item{(i)} $D\geq \delta^2$, i.e. $\|d\|^2$ is larger or at least equal with the term corresponding to the infeasible component $\delta^2=(z_k-x_k)^2$;
\item{(ii)} $\delta_S>\delta_C$ (since $c(z_k)<b_k$ it follows that $\delta_S=b_k-x_k>c(z_k)-x_k=\delta_C$);
\item{(iii)} $\delta_S<\delta$ (if $z_k>b_k$ then $\delta_S=b_k-x_k<z_k-x_k=\delta$);
\item{(iv)} $\delta_S\geq 0$ (since $x_k\in [a_k,b_k]$ it follows that $\delta_S=b_k-x_k\geq 0$).
\end{description}
By using \texttt{Reduce} function from \texttt{Wolfram Mathematica 12} to solve the inequality~\ref{eq:Ineq} it follows that it is satisfied at least under the following conditions (depending on the position of the component $x_k$ of the target element):

\begin{description}
\item{(i)} if $a_k<x_k\leq b_k-(z_k-x_k)$ then $\cos^2(d,d_S)\geq \cos^2(d,d_C)$;
\item{(ii)} if $b_k-(z_k-x_k)< x_k\leq b_k$ and if $\|d\|^2\geq 2\delta (\delta-\delta_S)=2(z_k-x_k)(z_k-b_k)$ then $\cos^2(d,d_S)\geq \cos^2(d,d_C)$;
\end{description}

Since $\cos(d,d_S)\geq 0$ it follows that $\cos^2(d,d_S)\geq \cos^2(d,d_C)$ implies $\cos(d,d_S)\geq \cos(d,d_C)$. 
A similar result can be obtained when the lower bound is violated, i.e. $z_k<a_k$.
\end{proof}

\subsection{Influence of the SDIS on the population diversity -- Section~\ref{sect:diversityTheor}}

\begin{prop} \label{prop:diversity}
If the current population is almost uniformly distributed on $[0,1]$ the variance of the elements corrected by applying \texttt{mir} is $$var[c_M(Z)]=\frac{F^2}{10}-\frac{F}{4}+\frac{1}{4}.$$
\end{prop}

\begin{proof}
According to \citep{Ali2006}, the infeasible elements obtained by applying a DE/rand/1 (with $F\in [0.5,1]$) mutation on a uniformly distributed scalar population follows distributions given by:

\begin{equation}\label{eq:distrLB}
    f_{Z_{lb}}(z)=\frac{1}{F}\int_{0}^{z+F}\left(1-\frac{x-z}{F}\right)dx=\frac{(F+z)^2}{2F^2} \quad (-F\leq z<0)
\end{equation}
and
\begin{equation}\label{eq:distrUB}
    f_{Z_{ub}}(z)=\frac{1}{F}\int_{z-F}^{1}\left(1-\frac{z-x}{F}\right)dx=\frac{(1+F-z)^2}{2F^2} \quad (1< z\leq 1+F)
\end{equation}
The variance and the mean of these distributions are  $\hbox{var}[Z_{lb}]=\hbox{var}[Z_{ub}]=3F^2/80$ and $\mathbb{E}[Z_{lb}]=-F/4$, $\mathbb{E}[Z_{ub}]=1+F/4$, respectively.

Using $Z^{M}_{lb}=-Z_{lb}$ and $Z^{M}_{ub}=1-Z_{ub}$ to denote the random variables corresponding to the elements corrected by \texttt{mirror} it follows that $\mathbb{E}[Z^{M}_{lb}]=F/4$ and $\mathbb{E}[Z^{M}_{ub}]=1-F/4$. Since mirroring does not change the variance of the random variable it follows that $\hbox{var}[Z^{M}_{lb}]=\hbox{var}[Z^{M}_{ub}]=3F^2/80$.

The random variable, $c_M(Z)$, corresponding to the population of corrected elements can be interpreted as a mixture of the variables $Z^{M}_{lb}$ and $Z^{M}_{ub}$ with mixing weights $w_{lb}=w_{ub}=1/2$, as in the absence of a selection pressure there it should be no difference between the probabilities of violating the lower and the upper bound. The variance of the mixture satisfies Eq.~\ref{eq:varMixture}.

\begin{eqnarray}\label{eq:varMixture}
\hbox{var}[c_M(Z)] &= &w_{lb}\hbox{var}[Z^M_{lb}] +  w_{ub}\hbox{var}[Z^M_{ub}]+w_{lb}(\mathbb{E}[Z^M_{lb}])^2+w_{ub}(\mathbb{E}[Z^M_{ub}])^2- \nonumber \\
 & & \nonumber \\
 & & (w_{lb}\mathbb{E}[Z^M_{lb}]+w_{ub}\mathbb{E}[Z^M_{ub}])^2
\end{eqnarray}

By replacing in Eq.~\ref{eq:varMixture} the mixing weights $w_{lb}$ and $w_{ub}$ with $1/2$, the variance and the mean of $Z^M_{lb}$ and $Z^M_{lb}$ with the above mentioned values, one obtains that 
\begin{eqnarray}
\hbox{var}[c_M(Z)]&=&\frac{1}{2}\left(\frac{3F^2}{40}+\left(\frac{F}{4}\right)^2+\left(1-\frac{F}{4}\right)^2 \right)-\frac{1}{4}\left(\frac{F}{4}+1-\frac{F}{4}\right)^2 \nonumber\\
 & &\nonumber\\
 &=& \frac{3F^2}{40}+\frac{F^2}{16}-\frac{F}{4}+\frac{1}{4}\nonumber\\
 & &\nonumber\\
 &=& \frac{F^2}{10}-\frac{F}{4}+\frac{1}{4}
\end{eqnarray}

\end{proof}

\begin{figure}[ht!]
    \centering
    \includegraphics[width=0.75\textwidth]{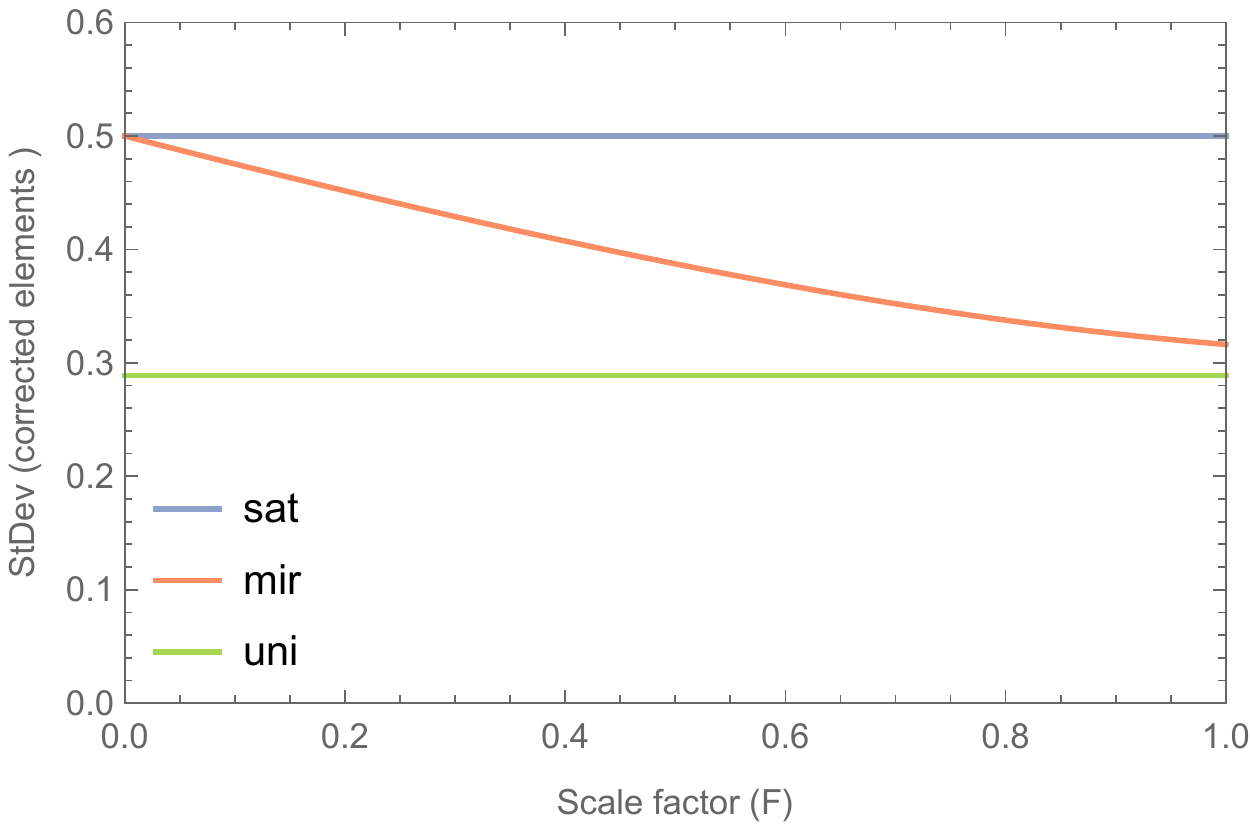}
    \caption{Standard deviation of the population of corrected elements (\texttt{sat}, \texttt{mir},\texttt{uni}) under the assumption that the current population is almost uniformly distributed on $[0,1]$.}
    \label{fig:VarianceCorrection}
\end{figure}

\begin{figure}[ht!]
    \centering
    \includegraphics[width=\textwidth]{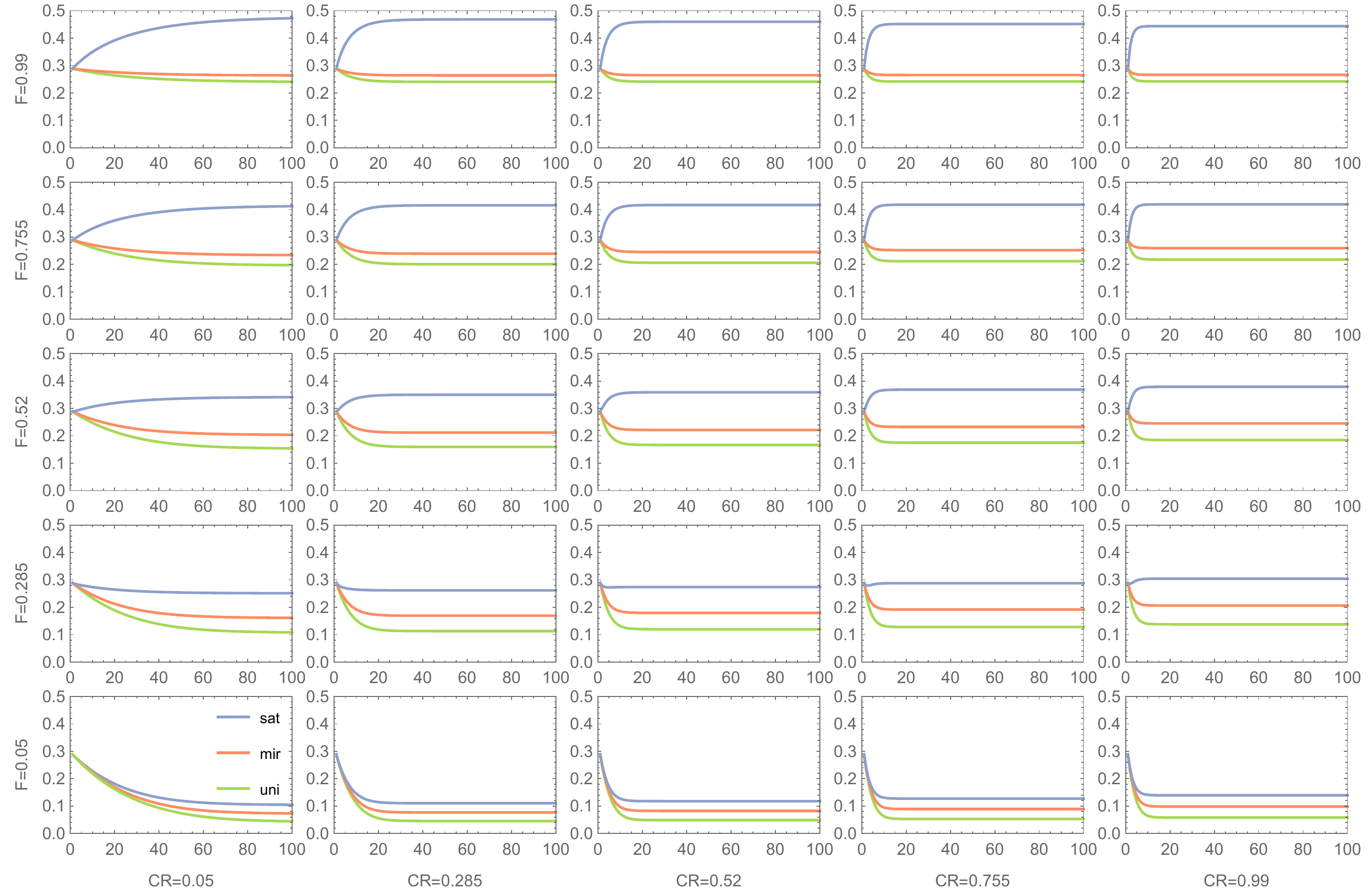}
    \caption{Influence of SDIS (\texttt{sat}, \texttt{mir}, \texttt{uni}) and of $C_r$ and $F$ on the evolution of the expected standard deviation of the population (\texttt{DE/rand/1/bin}, population size of $100$, $n=30$). 
    }
    \label{fig:TheoreticalDiversity}
\end{figure}

\subsection{Analysis of cosine similarity distribution -- Section~\ref{sect:CSAnalysis}}
\begin{prop}\label{prop:ECDF}
If $X$ and $Y$ are two independent random variables such that their cumulative distribution functions, $F_X$ and $F_Y$, satisfy $F_X(x)\geq F_Y(x)$ then the probability that $X$ is smaller than $Y$ is larger than $0.5$, i.e. it is more likely that $X$ is larger than $Y$ than the other case.
\end{prop}

\begin{proof}
The cumulative distribution function of $Z=X-Y$ satisfies:
\begin{equation}
F_Z(z)=\int_{-\infty}^{\infty}\int_{-\infty}^{y+z}f_{XY}(x,y)dxdy=\int_{-\infty}^{\infty}\int_{-\infty}^{y+z}f_X(x)f_Y(y)dxdy 
\end{equation}
Since $F_X(x)\geq F_Y(x)$ it follows that the probability density functions, $f_X$ and $f_Y$ satisfies the same property, i.e. $f_X(x)\geq f_Y(x)$. On the other hand, the probability that $X$ is smaller than $Y$ is $P(Z\leq 0)=F_Z(0)$ which satisfies:

\begin{eqnarray}
F_Z(0) &=&\int_{-\infty}^{\infty}\int_{-\infty}^{y}f_X(x)f_Y(y)dxdy \nonumber \\
       &\geq&\int_{-\infty}^{\infty}\int_{-\infty}^{y}f_Y(x)f_Y(y)dxdy \nonumber \\
       &=&\int_{-\infty}^{\infty}f_Y(y)\left(\int_{-\infty}^{y}f_Y(x)dx\right)dy \nonumber \\
       &=& \int_{-\infty}^{\infty}f_Y(y)F_Y(y)dy=\int_{-\infty}^{\infty}F^{\prime}_Y(y)F_Y(y)dy \nonumber\\
       &=& \frac{1}{2} \int_{-\infty}^{\infty} (F^2_Y(y))^{\prime}dy=\frac{1}{2}.
\end{eqnarray}

Thus $P(X\leq Y)\geq 0.5$. 

\end{proof}

\end{document}